%% file: ex_article.tex
\documentclass[onefignum,onetabnum]{siamonline250211}

\input{ex_shared}

\ifpdf
\hypersetup{
  pdftitle={Preconditioned Regularized Wasserstein Proximal Sampling},
  pdfauthor={H. Y. Tan, S. Osher, W. Li}
}
\fi

\begin{document}

\maketitle

\begin{abstract}
We consider sampling from a Gibbs distribution by evolving finitely many particles. We propose a preconditioned version of a recently proposed noise-free sampling method, governed by approximating the score function with the numerically tractable score of a regularized Wasserstein proximal operator. This is derived by a Cole--Hopf transformation on coupled anisotropic heat equations, yielding a kernel formulation for the preconditioned regularized Wasserstein proximal. The diffusion component of the proposed method is also interpreted as a modified self-attention block, as in transformer architectures. For quadratic potentials, we provide a discrete-time non-asymptotic convergence analysis and explicitly characterize the bias, which is dependent on regularization and independent of step-size.  Experiments demonstrate acceleration and particle-level stability on various log-concave and non-log-concave toy examples to Bayesian total-variation regularized image deconvolution, and competitive/better performance on non-convex Bayesian neural network training when utilizing variable preconditioning matrices. 
\end{abstract}



\section{Introduction}
We are interested in sampling from a Gibbs distribution 
\begin{equation}
    \pi(x) \propto \exp(-\beta V(x)),
\end{equation}
where $V:\R^d \rightarrow \R$ is some $\mathcal{C}^1$ potential, and $\beta>0$ is a constant. Such problems arise frequently in data science, such as uncertainty quantification in imaging \cite{laumont2022bayesian}, Bayesian statistical inference \cite{gelman1995bayesian}, inverse problems \cite{stuart2010inverse}, and computational physics \cite{krauth2006statistical}. Methods for sampling have also been connected to recent machine learning methods, such as simulating stochastic differential equations (SDEs) for diffusion models \cite{song2021scorebased}, unsupervised learning \cite{tan2024unsupervised}, and Bayesian neural networks \cite{mackay1995bayesian}. 

Current popular sampling methods to solve this are Markov chain Monte Carlo (MCMC) methods, which bypass the usual difficulty arising from the intractability of the normalizing constant $\int_{\R^d} \exp(-\beta V(x)) \dd{x}$. There is a wide variety of Monte Carlo methods each with different convergence rates and conditions on the potential $V$. Notable examples include the unadjusted Langevin algorithm and the Metropolis-adjusted Langevin algorithm \cite{rossky1978brownian,durmus2019high}, with recent algorithms considering proximal operators and Moreau--Yosida envelopes \cite{durmus2018efficient}, subgradients \cite{habring2024subgradient}, and annealed Moreau regularization \cite{habring2025diffusion}. The theoretical core of MCMC is to construct a Markov chain converging to some invariant distribution $\hat\pi$ that approximates the target distribution $\pi$. 

To construct such a Markov chain, a crucial theoretical component is exponential or geometric ergodicity of the chain. In typical Langevin methods, this is given by the Wiener diffusion term, which adds random noise at every step of the chain \cite{roberts1996exponential}. These chains can be interpreted as particular discretizations of SDEs, such as Euler--Maruyama in the case of the unadjusted Langevin algorithm. These SDEs correspond to Fokker--Planck equations \cite{risken1989fokker}, which are ODEs in density space, which in turn correspond to score-based particle evolutions based on Liouville equations \cite{kubo1963stochastic}. Here, the score function is defined by the gradient of the log density.

Deviating from MCMC methods, score-based methods crucially require knowledge of the score of the distribution at any given time-step. If the particles are evolved according to the Liouville equation 
\begin{equation}\label{eq:scoreODE}
    \frac{\dd X_t}{\dd t} = -\nabla V(X_t) - \beta^{-1} \nabla \log \rho(t,X_t),
\end{equation}
where $\rho(t,\cdot)$ denotes the density of $X_t$ at time $t$, then the density will again evolve through a Fokker--Planck equation and converge to the desired stationary distribution. The difficulty manifests as at each time step, only an empirical measure is known, from which the score has to be approximated. Various methods of score approximation include kernel density estimation \cite{carrillo2019blob,wand1994kernel}, adaptive kernels \cite{wang2022accelerated,van2003adaptive,botev2010kernel}, and approximation using neural networks \cite{bond2021deep, chen2018neural,nijkamp2022mcmc}. However, many such methods suffer from mode collapse and parameter/initialization sensitivity \cite{srivastava2017veegan,li2023reducing,gramacki2018nonparametric}.

Turning to density space, a classical discrete time approximation of the Fokker--Planck equation is the JKO method \cite{jordan1998variational}. This is given by the iteration of the Wasserstein proximal operator, defined as 
\begin{equation}
    \rho_{k+1} = \min_{\rho \in \gP_2}\int_{\R^d}V \rho + \beta^{-1} \rho \log \rho \dd{x} + \frac{1}{2h} \gW_2(\rho, \rho_k),
\end{equation}
where $\gW_2$ is the Wasserstein-2 distance, $h>0$ is some step-size, and the minimization is taken over Wasserstein-2 space $\gP_2$. This is a proximal iteration on the free energy using Wasserstein-2 as the distance. In the composite density case, splitting algorithms can also be employed \cite{salim2020wasserstein}. However, one of the main issues with using this approach for a general potential is that the Wasserstein proximal is itself difficult to compute. Some approaches have been proposed to combat this, such as a function objective utilizing an inner sampling loop \cite{fan2022variational} or coordinate-based methods \cite{yao2024wasserstein}. 

Motivated by the Wasserstein proximal iteration and a recent kernel formulation for an approximation of the Wasserstein proximal operator \cite{li2023kernel}, \cite{tan2024noise} proposed a noise-free sampling based on the regularized Wasserstein proximal operator (RWPO), called the backwards regularized Wasserstein proximal (BRWP) method. In this case, the score is approximated using the score of the RWPO, which is shown to correspond to a semi-implicit discretization of a Fokker--Planck equation. Similarly to the unadjusted Langevin algorithm (ULA), and unlike Langevin algorithms involving Metropolis steps (such as the Metropolis-adjusted Langevin algorithm MALA), BRWP can be shown to preserve the Gaussianity of a distribution at each iteration for quadratic target potentials. In particular, in continuous time, the mixing time dependency on dimension $d$ in the Gaussian case scales only as $\mathcal{O}(\log(d))$, whereas ULA scales as $\mathcal{O}(d^3)$ and MALA as $\mathcal{O}(d^2)$. For low dimensional distributions, the particles evolved through BRWP were found to be structured, with the outermost particles appearing to lie on level-set contours.

The BRWP method has also been extended with tensor-train techniques \cite{han2025tensor}, as well as to the non-smooth Bayesian LASSO problem by using $L^1$-splittings based on the Laplace approximation \cite{han2025splitting}. On the theoretical side, \cite{tan2024noise} originally demonstrate linear convergence in total variation for quadratic potentials in the continuous case. \cite{han2024convergence} further extend this to strongly log-concave distributions and smooth initial distributions, showing that the BRWP iterations with equal regularization parameter and time-step indeed form a first-order approximation for the Fokker--Planck equation. 

In this work, we aim to accelerate convergence of BRWP by utilizing preconditioners. Preconditioning is a common technique for accelerating gradient-based optimization, where gradient steps are multiplied by some (symmetric positive-definite) matrix, possibly varying across iterations and space. It is commonly useful when the objective function is ill-conditioned, allowing for larger step-sizes which translates to faster convergence, widely applicable including in inverse problems and deep learning \cite{beck2003mirror,kingma2014adam,fessler1999conjugate,tan2023data,ben2001ordered}. Such techniques can also be extended to sampling and Wasserstein space \cite{hsieh2018mirrored,jiang2021mirror,girolami2011riemann,li2016preconditioned,bonet2024mirror}. We consider approximating the following preconditioned version of \cref{eq:scoreODE}, which itself corresponds to a preconditioned Fokker--Planck equation, 
\begin{equation*}
    \frac{\dd X}{\dd{t}} = -M \nabla V(X)  - \beta^{-1} M \nabla \log \rho(t,X),
\end{equation*}
where $M \in \R^{d \times d}$ is a symmetric positive definite preconditioning matrix. The contributions of this work are as follows:
\begin{enumerate}
    \item In \Cref{sec:precond}, we derive the coupled PDE system that we wish to subsequently discretize, which define a preconditioned regularized Wasserstein proximal operator. Utilizing similar techniques to \cite{li2023kernel}, we demonstrate that the Laplacian regularizations in the Benamou--Brenier formulation of the Wasserstein proximal may be suitably replaced with second-order terms involving the preconditioning matrix. This is derived using a Cole--Hopf transform on a pair of forward-backward anisotropic heat equations. We provide an explicit kernel formula for the regularized Wasserstein proximal, and show that it maps $\gP_2$ to $\gP_2$ assuming mild coercivity conditions on the potential $V$.
    \item In \Cref{sec:PBRWP} we propose the preconditioned BRWP (PBRWP) method in \Cref{alg:PBRWP}. This is given by taking a backwards Euler discretization on the Fokker--Planck component of the PDE derived in the previous section. This yields a semi-implicit discretization of the Liouville equation. We provide expressions for the score approximation in the case of finitely many particles as an interacting particle system. \Cref{ssec:transformer} interprets the interactions as a modified self-attention block, which implicitly allows for efficient parallel computation using common GPU libraries.
    \item In the case of Gaussian initial and target measures (corresponding to the infinite-particle limit), \Cref{sec:Gaussian} provides a non-asymptotic analysis of the convergence rate in the discrete step-size regime. This characterizes the convergence rates in terms of the regularization parameter and the preconditioning matrix. We additionally provide various properties such as maximum allowable regularization parameters, Wasserstein contraction, a mean-variance contraction-diffusion inequality, and a maximum norm bound based on the number of particles.
    \item In \Cref{sec:experiments}, the proposed PBRWP method is compared on a variety of toy examples, a Bayesian imaging problem, and Bayesian neural network training. We additionally introduce some heuristic modifications to improve the performance of PBRWP in high dimension scenarios, corresponding to scaling laws in attention-based models. We observe in low dimensions that the proposed method converges faster and stably in KL divergence and maintains the structured particle phenomenon of BRWP. For a Bayesian image deconvolution task, we find rapid convergence to the MAP estimator, while maintaining inter-particle deviations. For the Bayesian neural network training task, we observe convergence to better trained networks compared to other score-based MCMC methods.
\end{enumerate}
The supplementary material contains proofs and additional experimental results, such as hyperparameter ablations and computational complexity.
\section{Background and Notation}
This section covers some basic definitions, including Wasserstein spaces, the Benamou--Brenier PDE formulation, and the non-preconditioned regularized Wasserstein proximal. 

Throughout, $M \in \R^{d \times d}$ is a symmetric positive definite matrix $M \in \mathrm{Sym}_{++}(\R^d)$. We define the scaled norms and inner product on $\R^d$ as
\begin{equation}
    \langle u,v\rangle_M = u^\top M^{-1}v,\quad \|\cdot\|_M^2 = \langle \cdot, \cdot\rangle_M.
\end{equation}
\begin{definition}[{\cite{santambrogio2015optimal}}]
    Let $\gP_2(\R^d)$ be the set of probability densities with finite second moment. For $\mu,\nu \in \gP_2(\R^d)$, the \emph{Wasserstein-2} distance $\gW_2(\mu, \nu)$ is
    \begin{equation}
        \gW_2(\mu,\nu) = \inf_{\pi \in \Gamma(\mu,\nu)} \int \|x-y\|^2 \dd\pi({x}, {y}),
    \end{equation}
    where the infimum is taken over couplings $\pi \in \Gamma(\mu,\nu)$, i.e. probability measures on $\R^d \times \R^d$ satisfying
    \begin{equation}
        \int_{\R^d} \pi(x,y) \dd{y} = \mu(x),\, \int_{\R^d} \pi(x,y) \dd{x} = \nu(y).
    \end{equation}

    Consider a probability density $\rho_0 \in \gP_2(\R^d)$ and $V \in \mathcal{C}^1(\R^d)$ be a lower bounded potential function. For a scalar $T>0$, the \emph{Wasserstein proximal} of $\rho_0$ is defined as 
    \begin{equation}\label{eq:wassproxdef}
        \wprox_{T, V}(\rho_0) \coloneqq \argmin_{q\in \mathcal{P}_2(\R^d)} \int_{\R^d} V(x)q(x) \dd{x} + \frac{\gW(\rho_0, q)^2}{2T}.
    \end{equation}
\end{definition}

The Wasserstein proximal can be equivalently reformulated as a coupled PDE system, sometimes known as the Benamou--Brenier or dynamical formulation \cite{benamou2000computational}. Using this, \cite{li2023kernel} defines the \emph{regularized Wasserstein proximal} by adding a diffusive term to the continuity equation of the Benamou--Brenier formulation. The resulting mean-field control problem is 
\begin{subequations}\label{eq:MFC}
    \begin{gather}
    \inf_{\rho,v,q} \int_0^T \int_{\R^d} \frac{1}{2}\|v(t,x)\|^2 \rho(t,x) \dd{x} \dd{t} + \int_{\R^d} V(x)q(x) \dd{x},\\
    \partial_t \rho(t,x) + \nabla\cdot (\rho(t,x) v(t,x)) = \beta^{-1} \Delta \rho(t,x), \quad \rho(0,x) = \rho_0(x),\, \rho(T,x) = q(x).
\end{gather}
\end{subequations}

From this problem, the regularized Wasserstein proximal operator (RWPO) of a density $\rho_0$, written $\wprox_{T,V} \rho_0$, is defined as the terminal time solution $\rho_T$ of the mean-field control problem \cref{eq:MFC}. Taking the regularization parameter $\beta \rightarrow \infty$ recovers the Wasserstein proximal \cref{eq:wassproxdef} through the Benamou--Brenier formulation of the Wasserstein distance. The RWPO admits the following coupled PDE formulation, consisting of a forward-time Fokker--Planck equation in $\rho$, and a backward-time viscous Hamilton--Jacobi equation in a dual variable $\Phi$, 
\begin{subequations}\label{eqs:regPDE}
    \begin{numcases}{}
      \partial_t \rho(t,x) + \nabla_x \cdot\left(\rho(t,x) \nabla_x \Phi(t,x)\right) = \beta^{-1} \Delta_x \rho(t,x), \label{eq:regPDE_a} \\
      \partial_t \Phi(t,x) + \frac{1}{2} \|\nabla_x \Phi(t,x)\|^2 = -\beta^{-1} \Delta_x \Phi(t,x), \label{eq:regPDE_b}\\
      \rho(0,x) = \rho_0(x),\quad \Phi(T,x) = -V(x).\label{eq:regPDE_c}
    \end{numcases}
\end{subequations}
Consider the Cole--Hopf transformation
\begin{equation*}
    \begin{cases}
        \eta(t,x) = e^{\beta\Phi(t,x)/2}, \\
        \hat\eta(t,x) = \rho(t,x)e^{-\beta\Phi(t,x)/2}
    \end{cases} \Leftrightarrow
    \begin{cases}
        \Phi(t,x) = 2\beta^{-1} \log \eta(t,x), \\
        \rho(t,x) = \eta(t,x)\hat\eta(t,x).
    \end{cases}
\end{equation*}
This transforms the coupled PDEs \cref{eqs:regPDE} to the coupled forward-backward heat equations
\begin{equation}\label{eqs:coupledHeat}
    \begin{cases}
        \partial_t \hat\eta(t,x) = \beta^{-1} \Delta \hat\eta(t,x), \\
        \partial_t \eta(t,x) = -\beta^{-1} \Delta\eta(t,x), \\
        \eta(0,x) \hat\eta(0,x) = \rho_0(x),\quad \eta(T,x) = e^{\beta\Phi(T,x)/2} = e^{-\beta V(x)/2},
    \end{cases}
\end{equation}
which may be solved using the heat kernel \cite{li2023kernel}
\begin{gather*}
    G_t(x,y) = \frac{1}{(4\pi \beta^{-1} t)^{d/2}} e^{-\beta\frac{\|x-y\|^2}{4 t}}, \quad
    \eta(t,x) = G_{T-t} * \eta_T = G_{T-t} * e^{-\beta V/2}.
\end{gather*}
Translating back, the terminal solution $t=T$ to the coupled PDEs \cref{eqs:regPDE} is given by
\begin{gather}
    \rho(T,x) = \int_{\R^d} K(x,y) \rho_0(y) \dd{y}, \label{eq:rhoT}\\
    K(x,y) = \frac{\exp(-\frac{\beta}{2} (V(x) + \frac{\|x-y\|^2}{2T}))}{\int_{\R^d} \exp(-\frac{\beta}{2} (V(z) + \frac{\|z-y\|^2}{2T})) \dd{z}}.\label{eq:kernelDef}
\end{gather}
\section{Preconditioned Regularized Wasserstein Proximal}\label{sec:precond}
We have seen that the regularized Wasserstein proximal is a particular Laplacian regularization of the Benamou--Brenier formulation, corresponding to coupled heat equations and computable using a kernel formula \cite{li2023kernel}. In this section, we define a geometry-aware elliptic regularization to the Benamou--Brenier formulation, corresponding to coupled anisotropic heat equations under a Cole--Hopf transformation. This manifests as computation with a scaled heat kernel, allowing for more generality in the proximal operator.

From here onwards, let $M \in \R^{d \times d}$ be a symmetric positive definite matrix $M \in \mathrm{Sym}_{++}(\R^d)$. Define the anisotropic heat kernel
\begin{equation}\label{eq:aniKernel}
    G_{t,M}(x,y) \coloneqq \frac{1}{(4\pi \beta^{-1} t)^{d/2} |M|^{1/2}} e^{-\beta\frac{(x-y)^\top M^{-1}(x-y)}{4t}}.
\end{equation}
This is a scaled version of the standard heat kernel, and it is also a Green's function for an anisotropic heat equation.
\begin{proposition}
    The anisotropic kernel $G_{t,M}$ is a Green's function for the following PDE:
    \begin{equation*}
    \begin{cases}
        \partial_t u - \beta^{-1} \nabla \cdot(M \nabla u) = 0,\\
        u(0,x) = \delta(y).
    \end{cases}
    \end{equation*}
\end{proposition} 
\begin{proof}
   This is from a direct computation. The boundary condition arises from a change of variables with the standard heat kernel. The derivation can be found in \ref{app:greens}.
\end{proof}

We now define the preconditioned regularized Wasserstein proximal operator of a density as the terminal time solution to a particular set of PDEs. Then, we show that it is equivalently a convolution involving another inhomogeneous kernel, with the kernel itself defined as a convolution with the anisotropic heat kernel.
\begin{definition}\label{def:PRWPO}
    For a symmetric positive definite matrix $M$ and regularization parameter $T>0$, define the \emph{preconditioned regularized Wasserstein proximal operator} (PRWPO) $\wprox^M_{T, V}: \rho_0 \mapsto \rho_T$ to be the terminal density of the following PDE system
    \begin{equation}\label{eqs:PBRWPPDE}
        \begin{cases}
            \partial_t \rho(t,x) + \nabla \cdot(\rho(t,x) \nabla \Phi(t, M^{-1}x)) = \beta^{-1}\nabla \cdot(M \nabla \rho)(t,x),  \\
             \partial_t \Phi(t, M^{-1}x) + \frac{1}{2}\|\nabla \Phi(t, M^{-1}x)\|_M^2= -\beta^{-1} \Tr(M^{-1}(\nabla^2 \Phi)(t, M^{-1}x)), \\
             \rho(0,x) = \rho_0(x),\quad \Phi(T,M^{-1}x) = -V(x).
        \end{cases}
    \end{equation}
    Equivalently, the density of $\wprox^M_{T, V} \rho_0$ is given by 
    \begin{align}
        \rho_T(x) = \rho(T,x) = \int_{\R^d} K(x,y,\beta, M, T,V) \rho(0,y)\dd{y}, \label{eq:WProxKernelDef}
    \end{align}
    where the normalized kernel is given by
    \begin{equation*}
        K(x,y,\beta,M,T,V) \coloneqq \frac{e^{-\frac{\beta}{2}(V(x) + \frac{\|x-y\|^2_M}{2T})}}{\int_{\R^d} e^{-\frac{\beta}{2} (V(z) + \frac{\|z-y\|^2_M}{2T})} \dd{z}}.
    \end{equation*}
\end{definition}
Compared to \cref{eq:kernelDef}, the preconditioned kernel replaces the Euclidean norm $\|x-y\|^2$ with the scaled norm $\|x-y\|_M^2$, which is by construction. We show that the terminal density and the kernel formula are equivalent in the following section.

\begin{remark}
    Similarly to \cref{eq:MFC}, the PRWPO can be defined as the solution of a mean field control problem. In particular, it corresponds to 
    \begin{equation}
        \inf_{\rho, u, q} \int_0^T \int_{\R^d} \frac{1}{2}\|u(t,x)\|^2_{M^{-1}} \rho(t,x) \dd{x} \dd{t} + \int_{\R^d} V(x) q(x) \dd{x},
    \end{equation}
subject to the continuity equation and boundary conditions    
\begin{equation}
        \partial_t \rho(t,x) + \nabla \cdot (\rho Mu) = \beta^{-1} \nabla\cdot(M \nabla \rho),\quad \rho(0,x) = \rho_0(x),\, \rho(T,x) = q.
    \end{equation}
    The solution density $\rho$ of the above problem satisfies the PDE system \cref{eqs:PBRWPPDE}, see Appendix \ref{appsec:mfc}. Taking $M=I$ recovers the non-preconditioned mean-field control problem \cref{eq:MFC}.
\end{remark}

\subsection{Proof of equivalence}
Like in \cite{li2023kernel}, the kernel formula \cref{eq:WProxKernelDef} arises from a Cole--Hopf transformation to the underlying coupled PDE system. We consider the following coupled forward-backward anisotropic heat equations, which admit the kernel $G_{t,M}$,
\begin{equation} \label{eq:aniHeat}
    \begin{cases}
        \partial_t \hat\eta(t,x) = \beta^{-1} \nabla \cdot(M\nabla\hat\eta(t,x)), \\
        \partial_t \eta(t,x) = -\beta^{-1} \nabla \cdot(M\nabla\eta(t,x)), \\
        \eta(0,x) \hat\eta(0,x) = \rho_0(x),\quad \eta(T,x) = e^{\beta\Phi(T,x)/2} = e^{-\beta V(x)/2}.
    \end{cases}
\end{equation}
We now transform this back into a set of coupled PDEs with the following Cole--Hopf transform,
\begin{equation}\label{eq:colehopf}
    \begin{cases}
        \eta(t,x) = e^{\beta\Phi(t,M^{-1}x)/2}, \\
        \hat\eta(t,x) = \rho(t,x)e^{-\beta\Phi(t,M^{-1}x)/2}
    \end{cases} \Leftrightarrow
    \begin{cases}
        \Phi(t,x) = 2\beta^{-1} \log \eta(t,Mx), \\
        \rho(t,x) = \eta(t,x)\hat\eta(t,x).
    \end{cases}
\end{equation}
The boundary conditions transform to yield the desired $\rho(0,x) = \rho_0(x),\, \Phi(T,M^{-1}x) = -V(x)$. We first compute the second (backwards) heat equation, dealing solely in the dual variable $\Phi$:
\begin{align*}
    \partial_t \eta(t,x) &=\frac{\beta}{2}\partial_t \Phi(t, M^{-1}x) e^{\beta\Phi(t,M^{-1}x)/2},\\
    \nabla \eta(t,x) &= \frac{\beta}{2}M^{-1} \nabla \Phi(t, M^{-1}x)e^{\beta\Phi(t,M^{-1}x)/2}.
\end{align*}
The RHS of the second heat equation yields
\begin{align*}
    & \quad  \nabla \cdot(M\nabla \eta(t,x)) \\
    &= \nabla \cdot\left(\frac{\beta}{2}\nabla \Phi(t, M^{-1}x)e^{\beta\Phi(t,M^{-1}x)/2}\right) \\
    &=  \left(\frac{\beta}{2}\Tr(M^{-1}\nabla^2\Phi(t,M^{-1}x)) + \frac{\beta^2}{4}\nabla\Phi(t, M^{-1}x)^\top M^{-1} \nabla\Phi(t, M^{-1}x)\right)e^{\beta\Phi(t,M^{-1}x)/2}.
\end{align*}
Therefore the backwards heat equation evolving from time $t=T$ to $t=0$
\begin{equation*}
    \partial_t \eta(t,x) = -\beta^{-1} \nabla\cdot(M\nabla\eta(t,x))
\end{equation*}
becomes a PDE in the dual variable
\begin{equation}\label{eq:PhiHJ}
    \partial_t \Phi(t, M^{-1}x) + \frac{1}{2}\|\nabla \Phi(t, M^{-1}x)\|_M^2= -\beta^{-1} \Tr(M^{-1}(\nabla^2 \Phi)(t, M^{-1}x)).
\end{equation}
This is the desired Hamilton--Jacobi equation in \cref{eqs:PBRWPPDE}. Now computing the transformation for the forward heat equation $\hat\eta$:
\begin{align*}
    \partial_t \hat\eta(t,x) &= \left(\partial_t \rho - \frac{\beta}{2}\rho\partial_t \Phi(t, M^{-1}x)\right) e^{-\beta\Phi(t, M^{-1}x)/2},\\
    \nabla \hat\eta(t,x) &= \left(\nabla \rho - \frac{\beta}{2}\rho M^{-1} \nabla \Phi(t, M^{-1}x)\right)e^{-\beta\Phi(t,M^{-1}x)/2}.
\end{align*}
The RHS can be computed as
\begin{align*}
    \nabla \cdot(M\nabla \hat\eta(t,x)) &= \nabla\cdot(M\nabla \rho) e^{-\beta\Phi(t,M^{-1}x)/2} - \beta\nabla \rho^\top \nabla\Phi(t,M^{-1}x) e^{-\beta\Phi(t,M^{-1}x)/2} \\ &\qquad- \frac{\beta}{2}\rho\Tr(M^{-1}\nabla^2 \Phi(t, M^{-1}x))e^{-\beta\Phi(t,M^{-1}x)/2} \\ &\qquad+ \frac{\beta^2}{4} \rho\nabla\Phi(t, M^{-1}x)^\top M^{-1} \nabla \Phi(t, M^{-1}x) e^{-\beta\Phi(t,M^{-1}x)/2}.
\end{align*}
Applying the transformation to the forward heat equation in $\hat\eta$ yields
\begin{align*}
     \partial_t \rho - \frac{\beta}{2}\rho\partial_t \Phi(t, M^{-1}x)&= \beta^{-1} \nabla\cdot (M\nabla\rho) -  \langle \nabla\rho, \nabla \Phi(t, M^{-1}x)\rangle \\
     &\qquad - \frac{1}{2}\rho\Tr(M^{-1}\nabla^2 \Phi(t, M^{-1}x)) + \frac{\beta}{4}\rho\|\nabla \Phi(t,M^{-1}x)\|_M. 
\end{align*}
Simplifying with the Hamilton--Jacobi equation for $\Phi$ \cref{eq:PhiHJ}, we obtain
\begin{gather*}
    \partial_t \rho + \langle \nabla \rho, \nabla \Phi(t, M^{-1}x)\rangle + \rho \Tr(M^{-1}\nabla^2 \Phi(t, M^{-1}x)) = \beta^{-1} \nabla \cdot(M \nabla \rho) \\
    \Rightarrow \partial_t \rho + \nabla \cdot(\rho(t,x) \nabla \Phi(t, M^{-1}x)) = \beta^{-1} \nabla \cdot(M \nabla \rho).
\end{gather*}
This is the desired (modified) Fokker--Planck equation in \cref{eqs:PBRWPPDE}. We have demonstrated that the Cole--Hopf transformation \cref{eq:colehopf} transforms the coupled anisotropic heat equations \cref{eq:aniHeat} into a diffusive coupled PDE system: a forward-time Fokker--Planck equation, and a backward-time Hamilton--Jacobi equation. Going backwards, we may infer that the coupled system admits a kernel solution corresponding to the anisotropic heat equation, and showing equivalence in \Cref{def:PRWPO}. We summarize this in the following proposition.
\begin{proposition}
    The coupled PDE system
    \begin{equation}\label{eq:PDEs}
    \begin{cases}
        \partial_t \rho(t,x) + \nabla \cdot(\rho(t,x) \nabla \Phi(t, M^{-1}x)) = \beta^{-1} \nabla \cdot(M \nabla \rho)(t,x) , \\
         \partial_t \Phi(t, M^{-1}x) + \frac{1}{2}\|\nabla \Phi(t, M^{-1}x)\|_M^2= -\beta^{-1}  \Tr(M^{-1}(\nabla^2 \Phi)(t, M^{-1}x)) ,\\
         \rho(0,x) = \rho_0(x),\quad \Phi(T,M^{-1}x) = -V(x),
    \end{cases}
    \end{equation}
    admits the following kernel formulation, where $G_{t,M}(x,y)$ is given in \cref{eq:aniKernel}:
    \begin{equation}
        \begin{cases}
            \rho(t,x) = \left(G_{T-t, M} * e^{-\beta V/2}\right)(x)\cdot \left(G_{t,M} * \frac{\rho_0}{(G_{T, M} * e^{-\beta V/2})}\right)(x), \\
            \Phi(t,x) = 2\beta^{-1} \log \left(G_{T-t, M} * e^{-\beta V/2}\right)(Mx).
        \end{cases}
    \end{equation}
\end{proposition}

\begin{proof}
    The PDEs under the change of variables \cref{eq:colehopf} becomes the decoupled anisotropic heat equations \cref{eq:aniHeat}. Using the Green's function, the solutions are given by 
    \begin{equation*}
        \eta(t,x) = G_{T-t, M} * \eta_T = (G_{T-t, M} * e^{-\beta V/2})(x).
    \end{equation*}
    From the initial boundary conditions, 
    \begin{equation*}
        \hat\eta(0,x) = \frac{\rho_0(x)}{(G_{T, M} * e^{-\beta V/2})(x)}.
    \end{equation*}
    From the Green's function again,
    \begin{equation*}
        \hat\eta(t,x) = \left(G_{t,M} * \frac{\rho_0}{(G_{T, M} * e^{-\beta V/2})}\right)(x).
    \end{equation*}
    Since $\rho = \eta \hat\eta$ from \cref{eq:colehopf},
    \begin{equation}
        \rho(t,x) = \left(G_{T-t, M} * e^{-\beta V/2}\right)(x)\cdot \left(G_{t,M} * \frac{\rho_0}{(G_{T, M} * e^{-\beta V/2})}\right)(x).
    \end{equation}
\end{proof}
    In particular, plugging in $t=T$, we have the following kernel formula for the PRWPO which is defined as the terminal solution.
\begin{corollary}\label{cor:KernelFormula}
    The preconditioned regularized Wasserstein proximal admits the kernel formula:
    \begin{align}
        \wprox^M_{T, V}\rho_0(x) &= e^{-\beta V(x)/2} \int_{\R^d} \frac{G_{T,M}(y,x) \rho_0(y)}{(G_{T,M} * e^{-\beta V/2})(y)}\dd y \notag\\
        &= \int_{\R^d} K(x,y,\beta, M, T,V) \rho(0,y)\dd{y} ,\label{eq:kernelRho}
    \end{align}
    where
    \begin{equation}\label{eq:kernelFormPrecond}
        K(x,y,\beta,M,T,V) \coloneqq \frac{e^{-\frac{\beta}{2}(V(x) + \frac{\|x-y\|^2_M}{2T})}}{\int_{\R^d} e^{-\frac{\beta}{2} (V(z) + \frac{\|z-y\|^2_M}{2T})} \dd{z}}.
    \end{equation}
 \end{corollary}
For notation purposes, we shorten the kernel \cref{eq:kernelFormPrecond} to $K_M(x,y)$. We may interpret $K_M$ as a Markov kernel: for any $y$, $\int_x K_M(x,y) \dd{x}=1$, and moreover $K_M$ is measurable. This is useful to determine how each particle affects the score, given by applying $K_M$ to a Dirac mass.

Since $K_M$ is a stochastic kernel, we have that the PRWPO of a probability density is another probability density, as noted in \cite{li2023kernel}. Under some mild regularity conditions on $V$, the PRWPO maps $\gP_2$ into $\gP_2$.
\begin{proposition}\label{prop:MomentEstimate}
    Let $\rho_0 \in \gP_2(\R^d)$. Assume that $V \in \gC^1$ satisfies the following regularity conditions:
    \begin{enumerate}
        \item $V$ is lower bounded,
        \item (Coercivity) $x \cdot \nabla V(x)$ is lower bounded\footnote{This assumption can be relaxed to $x^\top \hat M \nabla V(x) \ge -C$ for some $C \ge 0$, where $\hat M$ is a matrix commuting with $M$.}, i.e. $\exists C\ge 0$ such that 
        \begin{equation*}
            x \cdot \nabla V(x) \ge -C,\quad \forall x \in \R^d.
        \end{equation*}
    \end{enumerate}
    Then the following second moment estimate holds for any $T>0$:
    \begin{equation*}
        \int \|x\|_M^2 \wprox_{T,V}^M \rho_0(x) \dd{x} \le 4T(d\beta^{-1}+C/2) + \mathbb{E}_{\rho_0}[\|x\|^2_M]
    \end{equation*}
    In particular, the PRWPO of $\rho_0$ also has finite second moment.
\end{proposition}
\begin{proof}
    By integrating the kernel formula by parts. See Appendix \ref{appssec:moment}.
\end{proof}
The second assumption on $V$ can be thought of as a tail condition. For example, if $V(x) = \frac{1}{2} x^\top \Sigma^{-1} x + g(x)$ for some positive definite $\Sigma$ and $\gC^1_b$ perturbation function $g$ with bounded gradient, then the coercivity condition will hold. This condition can be further relaxed into an estimate of the form
\begin{equation}
    x \cdot \nabla V(x) \ge -C_1 - C_2 \|x\|^2_M,
\end{equation}
where $C_1\ge 0$ and $C_2 < T^{-1}$, yielding a similar affine growth in $T$. We note that no growth assumption on $V$ is required for the PRWPO to be defined. For example, taking $V \equiv 0$ turns the PRWPO into a heat diffusion on $\rho_0$.

\section{Preconditioned Backwards Regularized Wasserstein Proximal Method}\label{sec:PBRWP}
Equipped with a preconditioned version of the regularized Wasserstein proximal operator, we can proceed like in \cite{tan2024noise} and devise a sampling algorithm based on the Fokker--Planck component of the coupled PDEs. The Fokker--Planck component can be written as 
\begin{equation}\label{eq:MRegWassProx}
\partial_t \rho(t,x) + \nabla \cdot\left(\rho(t,x)\nabla \Phi(t, M^{-1}x)- \beta^{-1} M \rho \nabla (\log \rho)(t,x)\right)=0.
\end{equation}
Using Liouville's equation with the Hamiltonian $\nabla \Phi(t, M^{-1}x)- \beta^{-1} M \nabla (\log \rho)(x)$, the corresponding particle formulation with density satisfying \cref{eq:MRegWassProx} is 
\begin{equation}
    \frac{\dd X}{\dd t} = \nabla \Phi(t, M^{-1}X) - \beta^{-1} M\nabla  \log \rho(t, X).
\end{equation}

This can be seen to approximate the original Liouville equation since $\nabla \Phi(t, M^{-1}x) \approx -M \nabla V(x)$. Performing a backwards Euler discretization on this ODE, we define the following \emph{preconditioned backward regularized Wasserstein proximal} (PBRWP) method:
\begin{equation}\label{eq:PrecondBRWP}
    X_{k+1} = X_k + \eta \left(-M \nabla V(X_k) - \beta^{-1} M \nabla \log \wprox_{T, V}^M\rho_{k}(X_k)\right),
\end{equation}
where $\wprox_{T, V}^M \rho_{k}$ is the PRWPO of the distribution of $X_k$ as given in \cref{eq:PDEs}. As in \cite{tan2024noise}, this can be evaluated using the kernel formula \cref{eq:kernelFormPrecond}. 

Note that the formal limit as $T \rightarrow 0$ and $\eta \rightarrow 0$ is the score-based ODE with potential $V$,
\begin{equation}\label{eq:particleExact}
    \frac{\dd X}{\dd{t}} = -M \nabla V(X)  - \beta^{-1} M \nabla \log \rho(t,X),
\end{equation}
which has a stationary solution $\pi \propto \exp(-\beta V)$. 

Based on \cref{eq:PrecondBRWP} and the kernel formulation \cref{eq:kernelFormPrecond}, we proceed as in \cite{tan2024noise} and consider the case where the input $\rho_k$ to the regularized Fokker--Planck equation is given by an empirical measure. This arises when updating a collection of points that together approximate the underlying distribution.

Suppose that $\rho_{k}$ is a uniform mixture of Dirac masses at points $\{\rvx_1^{(k)},...,\rvx_N^{(k)}\}$. From the kernel formulation \cref{eq:kernelFormPrecond}, we have the following expressions to compute the score (dropping the $k$ superscripts):
\begin{subequations}\label{eqs:ClosedForms}
\begin{equation}\label{eq:rhoTClosedForm}
\begin{split}
    \wprox^M_{T, V}\rho_{k}(\rvx_i) &= \frac{1}{N}\sum_{j=1}^N K_M(\rvx_i, \rvx_j) 
    = \frac{1}{N}\sum_{j=1}^N \frac{\exp \left[-\frac{\beta}{2}\left(V(\rvx_i) + \frac{\|\rvx_i - \rvx_j\|_M^2}{2T}\right)\right]}{\gZ(\rvx_j)}, \\
\end{split}
\end{equation}
\begin{equation}\label{eq:gradrhoTClosedForm}
\begin{split}
    \nabla \wprox^M_{T, V}\rho_{k}(\rvx_i) &= \frac{1}{N}\sum_{j=1}^N \frac{\left(-\frac{\beta}{2}\left(\nabla V(\rvx_i) + M^{-1}\frac{\rvx_i - \rvx_j}{T}\right)\right)\exp \left[-\frac{\beta}{2}\left(V(\rvx_i) + \frac{\|\rvx_i - \rvx_j\|_M^2}{2T}\right)\right]}{\gZ(\rvx_j)}, \\
\end{split}
\end{equation}
\begin{equation}\label{eq:normalizingConstant}
    \gZ(\rvx_j) \coloneqq \int_{\R^d} e^{-\frac{\beta}{2} (V(z) + \frac{\|z-\rvx_j\|^2_M}{2T})} \dd{z}.
\end{equation}
\end{subequations}

The latter normalization constant can be approximated using a Monte Carlo approximation about each particle, or using Laplace's method \cite{tibshirani2025laplace}. Turning to the other terms, we may observe the score $\nabla \log \rho = \frac{\nabla \rho}{\rho}$ as a particular weighted sum. Recall the definition of the softmax function: for a vector $v \in \R^d$, the softmax is the vector given by 
\begin{equation*}
    \softmax(v) = \left(\frac{\exp(v_i)}{\sum_{j=1}^d \exp(v_j)}\right)_{i=1,...,d},
\end{equation*}
satisfying $\sum_j \softmax(v)_j=1$. We may rewrite the score in terms of the softmax function, which makes it clear that it acts as an inter-particle diffusive term:
\begin{align*}
    \nabla \log \wprox^M_{T, V}\rho_{k}(\rvx_i) &= \frac{\nabla \wprox^M_{T, V} \rho_{k}}{\wprox^M_{T, V} \rho_{k}}(\rvx_i)\\
    &= \frac{\sum_{j=1}^N \frac{\left(-\frac{\beta}{2}\left(\nabla V(\rvx_i) + M^{-1}\frac{\rvx_i - \rvx_j}{T}\right)\right)\exp \left[-\frac{\beta}{2}\left(V(\rvx_i) + \frac{\|\rvx_i - \rvx_j\|_M^2}{2T}\right)\right]}{\gZ(\rvx_j)}}{\sum_{j=1}^N \frac{\exp \left[-\frac{\beta}{2}\left(V(\rvx_i) + \frac{\|\rvx_i - \rvx_j\|_M^2}{2T}\right)\right]}{\gZ(\rvx_j)}} \\
    &= -\frac{\beta \nabla V(\rvx_i)}{2}- \frac{\beta}{2 T}M^{-1} \rvx_i + \frac{\beta M^{-1}}{2 T} \sum_{j=1}^N \softmax(U_{i,\cdot})_j \rvx_j,
\end{align*}
where we define the interaction kernel $U_{i,j} = U_{i,j}(\rvx_1,...,\rvx_N)$ as the matrix
\begin{equation}
    U_{i,j} = -\frac{\beta}{2}\left(V(\rvx_i) + \frac{\|\rvx_i-\rvx_j\|^2_M}{2T}\right)- \log \gZ(\rvx_j).
\end{equation}
From \cref{eq:PrecondBRWP} we can write the particle-wise PBRWP iterations
\begin{align*}
    \rvx_i^{(k+1)} &= \rvx_i^{(k)} + \eta \left(-M \nabla V(\rvx_i^{(k)}) - \beta^{-1} M \nabla \log \wprox^M_{T, V}\rho_{k}(\rvx_i^{(k)})\right)\\
    &= \rvx_i^{(k)} - \frac{\eta}{2} M \nabla V(\rvx_i^{(k)}) + \frac{\eta}{2T}\left(\rvx_i^{(k)} - \sum_{j=1}^N \softmax(U_{i,\cdot}^{(k)})_j \rvx_j^{(k)}\right) \\
    &= \rvx_i^{(k)} - \frac{\eta}{2} M \nabla V(\rvx_i^{(k)}) + \frac{\eta}{2T}\left( \sum_{j=1}^N \softmax(U_{i,\cdot}^{(k)})_j (\rvx_i^{(k)}-\rvx_j^{(k)})\right).
\end{align*}

We note that since $U_{i,j}$ is only used in the softmax function, it is possible to add a constant for each column without changing the dynamics. In particular, we define a simpler matrix $W_{i,j}$ such that
\begin{align*}
    \softmax(U_{i,\cdot})_j = \softmax(W_{i,\cdot})_j,
\end{align*}
defined by adding constants depending only on the row $i$:
\begin{align}
    W_{i,j} &\coloneqq U_{i,j} + \frac{\beta V(\rvx_i)}{2} \notag \\
    &= -\beta \frac{\|\rvx_i-\rvx_j\|_M^2}{4T} - \log \gZ(\rvx_j). \label{eq:interactionSimple}
\end{align}
This has the advantage of avoiding one computation of $V$, which is already implicitly used in the normalizing constants $\gZ$. To further simplify, let $\tX$ be the matrix of particles
\begin{equation*}
    \tX = \begin{bmatrix}
        \rvx_1 & ... & \rvx_N
    \end{bmatrix} \in \R^{d \times N}.
\end{equation*}

The above iteration can thus be written as 
\begin{equation}\label{eq:tensorForm}
    \tX^{(k+1)} = \tX^{(k)} - \frac{\eta}{2} M \nabla V(\tX^{(k)}) + \frac{\eta}{2T}\left(\tX^{(k)} - \tX^{(k)} \softmax(W^{(k)})^\top\right),
\end{equation}
where $\nabla V(\tX) = \begin{bmatrix}
        \nabla V(\rvx_1) & ... & \nabla V(\rvx_N)
    \end{bmatrix}$, $W$ is defined in \cref{eq:interactionSimple}, and $\softmax(W)_{i,j} = \softmax(W_{i,\cdot})_j$ is a row-stochastic matrix (and therefore, $\softmax(W)^\top$ is a column-stochastic matrix). Computing this interaction matrix requires the typical $\mathcal{O}(N^2)$ operations for interacting particle methods. This formulation is parallelizable using common GPU libraries such as PyTorch or JAX, assuming that the normalizing constant can be easily computed. Summarizing, the PBRWP update is given in \Cref{alg:PBRWP}, which details the steps used in practice.

\SetKwComment{Comment}{/* }{ */}

\begin{algorithm2e}
\caption{PBRWP: Preconditioned Backwards Regularized Wasserstein Proximal}\label{alg:PBRWP}
\KwData{Initial points $\rvx_1^{(1)},...,\rvx_N^{(1)} \in \R^d$, potential $V: \R^d \rightarrow \R$, preconditioner $M\in \mathrm{Sym}_{++}(\R^d)$, regularization parameter $T>0$, diffusion $\beta>0$, step-size $\eta>0$, iteration count $K$.}
\KwResult{$\tX^{(K)} = \begin{bmatrix}
    \rvx_1^{(K)} &...& \rvx_N^{(K)}
\end{bmatrix}$ sampling from $\exp(-\beta V)$}
Tensorize $\tX^{(1)} = \begin{bmatrix}
    \rvx_1^{(1)} & ... & \rvx_N^{(1)} 
\end{bmatrix} \in \R^{d \times N}$\;
\For{$k=1,...,K$}
{
Approximate normalizing constants $\gZ(\rvx_i^{(k)}),\, i=1,...,N$ using Monte Carlo/Laplace method\;
Compute interaction matrix $W_{i,j} = -\beta \frac{\|\rvx_i-\rvx_j\|_M^2}{4 T} - \log \gZ(\rvx_j)$\;
Compute row-wise softmax interaction matrix $\softmax(W)_{i,j} = \softmax(W_{i,\cdot})_j$\;

Evolve $\tX^{(k+1)} = \tX^{(k)} - \frac{\eta}{2} M \nabla V(\tX^{(k)}) + \frac{\eta}{2T}\left(\tX^{(k)} - \tX^{(k)} \softmax(W^{(k)})^\top\right)$ \;
}
\end{algorithm2e}

\begin{remark}
    The corresponding limit of \cref{eq:PrecondBRWP} as $\eta \rightarrow 0$ satisfies the modified Fokker--Planck equation
\begin{equation}\label{eq:ModFP}
    \partial_t \rho = \beta^{-1}\nabla \cdot\left(\rho M \nabla \log \frac{\wprox^M_{T, V}(\rho)}{\exp(-\beta V)}\right).
\end{equation}
A stationary solution is $\hat\pi$ satisfying $\wprox^M_{T, V}(\hat\pi) \propto \exp(-\beta V)$. This is a semi-implicit discretization of \cref{eq:MRegWassProx}, where $\rho\nabla \log\rho(t,x)$ is replaced with $\rho \nabla \log \wprox^M_{T,V} (\rho)$, and the dual variable $\Phi$ is also replaced with the terminal condition $\Phi(T,M^{-1}x) = -V(x)$. 
\end{remark}
\subsection{Transformer reformulation}\label{ssec:transformer}
We may connect \cref{eq:tensorForm} to self-attention mechanisms, which are commonly employed in transformers \cite{vaswani2017attention}, and have recently attracted interpretations as particle evolutions \cite{sander2022sinkformers,han2025splitting}. Recall that an attention block takes the following form: for queries $Q \in \R^{d\times N}$, keys $K \in \R^{d \times N}$ and values $V \in \R^{d_V \times N}$ of the prescribed dimensions\footnote{The matrices are transposed compared to the literature in order to be consistent with the convention in this work. The usual form is $\mathrm{Attn}(Q;K,V) = \softmax\left(\frac{QK^\top}{\sqrt{d}} \right) V$, where the softmax is taken row-wise as in this work. Both forms are equivalent.}, 
\begin{equation}\label{eq:attn}
    \mathrm{Attn}(Q;K,V) = V\softmax\left(\frac{Q^\top K}{\sqrt{d}} \right)^\top \in \R^{d_V \times N}.
\end{equation}
We may connect this with the softmax term in \cref{eq:tensorForm}. The argument of the softmax \cref{eq:interactionSimple} is  
\begin{align*}
    W_{i,j} &= -\beta \frac{\|\rvx_i-\rvx_j\|_M^2}{4 T} - \log \gZ(\rvx_j)\\
    &= \beta \frac{\rvx_i^\top M^{-1} \rvx_j}{2 T} - \log \gZ(\rvx_j) - \beta\frac{\|\rvx_i\|^2_M + \|\rvx_j\|^2_M}{4T}.
\end{align*}
Discarding the $\sqrt{d}$ scaling\footnote{This will be reintroduced in \Cref{ssec:highDimExs} as a heuristic modification in high dimensions.} of \cref{eq:attn} and the constant-in-$j$ term $\|\rvx_i\|_M^2$, the diffusion term $\tX^{(K)} \softmax(W^{(k)})^\top$ for $N$ particles takes the form of masked attention
\begin{equation}
    V\softmax\left(Q^\top K - \mathbf{1} \rvz^\top \right)^\top,
\end{equation}
where $Q^\top K = \frac{\beta}{2 T}\tX^\top M^{-1} \tX$, $\rvz\in \R^N$ is the vector of normalizing constants $\rvz_j = \log \gZ(\rvx_j) + \beta\frac{\|\rvx_j\|_M^2}{4 T}$, and $V = \tX$. 

For an input $H \in \R^{d \times N}$, the self-attention framework imposes the particular forms $V=W_V H,\, K=W_K H,\, V=W_V H$, where $W_V, W_K\in \R^{d\times d},\, W_V \in \R^{d_V \times d}$ are some admissible matrices. By taking $H=\tX$, the preconditioning arises as $W_V^\top W_K = M^{-1}$, and $W_V = I$, making the diffusion in PBRWP interpretable as a modified self-attention.

\begin{remark}
    In the other direction, modeling transformers as interacting particle systems, \cite{sander2022sinkformers} consider self-attention transformers with symmetric interaction and an additional condition $W_Q^\top W_K = W_K^\top W_Q = -W_V$, demonstrating that transformers can be interpreted as a particular Wasserstein gradient flow, converging to a nonlinear PDE. \cite{geshkovski2025mathematical} consider transformers acting on the sphere $\mathbb{S}^{d-1}$ in the special case $W_Q = W_K = I$, presenting various conditions on the dimension, number of particles, and temperature such that particles will cluster. \cite{castin2025unified} demonstrate well-posedness of the induced PDE for various forms of self-attention, as well as a continuous-time analysis of the Gaussian-preserving properties. These works consider the continuous limit of the transformers to obtain these links.
\end{remark}

\section{Analysis in Gaussian distributions}\label{sec:Gaussian}
In the particular case where $V(x) = \frac{x^\top \Sigma^{-1}x}{2}$, i.e. the target distribution is $\gN(0, \beta^{-1}\Sigma)$, the PRWPO and PBRWP iterations can be computed in closed form. Suppose $M$ is now such that $cM \preceq \Sigma \preceq CM$, where $0<c\le C < \infty$, and $T < c$. For Gaussian distributions, the PRWPO admits a closed form as follows.

\begin{proposition}\label{prop:PRWPOGaussian}
    For a Gaussian density $\rho_k = \gN(\mu_k, \Sigma_k)$ and the quadratic potential $V(x) = x^\top \Sigma^{-1}x/2$, the PRWPO satisfies
    \begin{equation}
        \wprox_{T,V}^M (\rho_k) = \gN(\tilde{\mu}_k,\tilde\Sigma_k),
    \end{equation}
    where $\tilde\mu_k, \tilde\Sigma_k$ are given by 
    \begin{gather}
        \tilde\mu_{k} = (I + T M\Sigma^{-1})^{-1} \mu_k, \\
        \tilde\Sigma_k = 2\beta^{-1} T\left(T\Sigma^{-1} + M^{-1}\right)^{-1} + \left(T\Sigma^{-1} + M^{-1}\right)^{-1} M^{-1} \Sigma_k M^{-1}\left(T\Sigma^{-1} + M^{-1}\right)^{-1}.
    \end{gather}
    Moreover, the kernel $K_M(x,y)$ in \cref{eq:kernelFormPrecond} takes the form of a Gaussian density 
    \begin{equation}\label{eq:kernelMarkovGaussian}
    K_M(\cdot, y) \sim \gN\left(\left(\frac{\Sigma^{-1}}{2} + \frac{M^{-1}}{2 T}\right)^{-1}\frac{M^{-1}}{2 T}y, \left(\frac{\beta\Sigma^{-1}}{2\beta } + \frac{\beta M^{-1}}{2 T}\right)^{-1}\right).
    \end{equation}
\end{proposition}
\begin{proof}[Proof Sketch]
    The kernel formula can be computed in closed form for Gaussians, and the product of two Gaussian densities is again proportional to a Gaussian density. The full proof is given in \ref{app:wproxGaussianCF}.
\end{proof}
Since the PRWPO maps Gaussian distributions to Gaussian distributions, the score function is linear, and the PBRWP iterations also preserve Gaussianity.
\begin{corollary}\label{cor:GaussianUpdate}
    For a Gaussian density $\rho_k = \gN(0, \Sigma_k)$, the PBRWP iteration satisfies the update $\rho_{k+1} = \gN(0, \Sigma_{k+1})$, where
    \begin{equation}
    \Sigma_{k+1} = (I - \eta M\Sigma^{-1} + \eta \beta^{-1} M\tilde\Sigma^{-1}_{k}) \Sigma_k (I - \eta M\Sigma^{-1} + \eta \beta^{-1} M\tilde\Sigma^{-1}_{k})^\top.
\end{equation}
\end{corollary}
\begin{proof}
    Follows from using \Cref{prop:PRWPOGaussian} in the PBRWP update \cref{eq:PrecondBRWP}.
\end{proof}

Now that we have shown that the PBRWP iteration preserves Gaussianity, we wish to consider the convergence for discrete time-steps. In the case where the initial and target distributions are both Gaussian, we can show that the distributions generated by PBRWP with discrete step-sizes converge to a Gaussian. Moreover, the terminal distributions are independent of step-size, only depending on $T$, the terminal Gaussian and the preconditioner. The following result extends the result of \cite{tan2024noise} from continuous time to discrete time, as well as adding generality by working in the preconditioned setting. We consider the zero-mean initialization case for brevity, as the covariances do not depend on the mean; moreover, it can be shown that the means of the iterates converge linearly to the target mean.
\begin{theorem}\label{thm:discreteTime}
    Suppose $V(x) = x^\top \Sigma^{-1} x/2$, and fix a $\beta >0$. For the target distribution $\pi = \gN(0, \beta^{-1} \Sigma)$ and positive definite preconditioner $M$, let $c, C>0$ be such that $cM \preceq \Sigma \preceq CM$. Let $T \in(0,c)$ so that the inverse PRWPO is well-defined. Then there exists a stationary distribution to the discrete-time PBRWP iterations, and it is Gaussian $\hat\pi = \gN(0, \hat\Sigma)$, satisfying $\wprox_{T,V}^M \hat\pi = \pi$. Moreover, it is unique within distributions with sufficiently decaying Fourier transforms.

    Suppose further that the initial Gaussian distribution $\rho_0 = \gN(0, \Sigma_0)$ has its covariance commuting with $\Sigma$. Let $\rho_k$ be the iterations generated by PBRWP, and let $\tilde\rho_{k} = \wprox_{T,V}^M(\rho_{k})$. Further let $\lambda$ be such that $\sqrt{M}^{-1} \Sigma_k \sqrt{M}^{-1}$ are uniformly bounded below by some $\lambda$ (which holds as $\Sigma_k \rightarrow \Sigma$). Let the step-size $\eta$ be bounded by 
    \begin{equation*}
        \eta\beta^{-1} \le \min((\max_i |\lambda_i(\tilde \Xi_k^{-1} - \tilde\Xi_\infty^{-1})|)^{-1}/2, 3\lambda_{\min}(\tilde\Xi_k)/32)
    \end{equation*}
    Then the exact discrete-time PBRWP method has the following decay in KL divergence:
    \begin{multline}
        \KL(\tilde\rho_{k+1} \| \pi) - \KL(\tilde\rho_{k} \| \pi) \\ \le -\frac{\eta}{2C[\beta+2T(1+TC^{-1})^{-1}(1+Tc^{-1})^2\lambda^{-1}]}\KL(\tilde\rho_{k}\| \pi).
    \end{multline}
\end{theorem}
\begin{proof}[Sketch Proof]
    Note that the covariance commuting condition can be easily satisfied by taking $\Sigma_0$ to be a multiple of the identity. From the condition, there exists a (Gaussian) distribution $\gN(0, \hat\Sigma)$ such that $\wprox_{T,V}^M(\gN(0, \hat\Sigma)) = \gN(0, \beta^{-1} \Sigma)$. Uniqueness comes from the closed-form kernel \cref{eq:kernelMarkovGaussian} and the invertibility of Gaussian convolution on the range of $\wprox^M_{T,V}$. Moreover, the PBRWP iterations generate Gaussian distributions, with closed-form covariance updates. The KL divergence between two Gaussians has a closed form in terms of the means and covariances, from which the decay comes using a Taylor expansion. The full proof is given in \ref{app:discreteTime}.
\end{proof}

\subsection{Additional properties}
We briefly list some properties of the PRWPO and PBRWP, and refer to Appendix \ref{appsec:finpartres} for more detailed statements and proofs.

\begin{enumerate}
    \item \textit{Conditions for PRWPO invertibility.} For a density $\rho$, there exists a unique distribution $\rho_0$ such that $\rho = \wprox_{T,V}^M \rho_0$ if and only if $\rho(x) e^{\beta V(x)/2}$ has Fourier transform decaying at least as fast than $e^{-T\beta^{-1} \xi^\top M \xi}$. In the Gaussian case $V = x^\top \Sigma^{-1} x/2$, this is equivalent to $\Sigma \succeq TM$.
    \item \textit{The regularized Wasserstein proximal is a $\gW_2$ contraction.} For quadratic $V$, let $\zeta \coloneqq(\frac{C^{-1}}{2} + \frac{1}{2 T})^{-1} (\frac{1}{2T})<1$, where $C$ is as in \Cref{thm:discreteTime}. For any $\mu,\nu$ in Wasserstein-2 space,
    \begin{equation}
        \gW_2(\wprox^M_{T, V}\mu, \wprox^M_{T, V}\nu) \le \zeta \gW_2(\mu, \nu).
    \end{equation}
    \item \textit{Mean-variance tradeoff.} For strongly convex $V$ and any $\beta>0$, the particles move together towards the minimizer, at a rate balanced between the inter-particle variance and the distances to the minimizer. In particular, let $\hat x$ be the minimizer of $V$, and let $\hat{\tX} = \begin{bmatrix}
            \hat{x} & ... & \hat{x}
    \end{bmatrix} = \hat{x} \mathbf{1}_N^\top \in \R^{d \times N}$. In particular, suppose $V$ is $\mu$-relatively strongly convex with respect to $M$. For any $\tX \in \R^{d \times N}$, let $\Delta$ be the descent direction \begin{equation*}
        \Delta = -\frac{1}{2} M \nabla V(\tX) + \frac{1}{2T}(\tX - \tX \softmax(W)^\top),
    \end{equation*}
    such that the PBRWP iteration (\ref{eq:tensorForm}) becomes $\tX^{(k+1)} = \tX^{(k)} - \eta \Delta$. The following inequality holds:
        
    \begin{multline*}
         \langle \Delta, M^{-1}(\tX - \hat{\tX}) \rangle_{\text{Frob}} \\\le \|\tX - \hat{\tX}\|_{2,M} \left(-\frac{\mu}{2} \|\tX - \hat{\tX}\|_{2,M} + \frac{(1+\sqrt{N})}{2T} \|\tX - \frac{1}{N}\tX \mathbf{1}_N \mathbf{1}_N^\top\|_{2,M}\right).
    \end{multline*}
    
    \item \textit{Diffusion is bounded based on the number of particles.} For quadratic $V = x^\top \Sigma^{-1} x/2$, suppose $M = \Sigma$ (and $T<1$). Suppose that $\rvx_1$ is an exterior point of the convex hull of $\{\rvx_i\}$, and further that $\delta>0$ is such that for all $j \ne 1$, 
    \begin{equation}
    \rvx_1^\top M^{-1} \rvx_j \le \rvx_1^\top M^{-1} \rvx_1 - \delta \|\rvx_1\|_M.
    \end{equation}
    Then, if $\rvx_1$ has sufficiently large norm, in particular assuming,
    \begin{equation}
        \delta \|\rvx_1\|_M \ge 2\beta^{-1} T\log (\frac{2 (N-1)}{T}),
    \end{equation}
    then for small step-size, the updated point satisfies $\|\rvx_1^{(k+1)}\|_M < \|\rvx_1^{(k)}\|_M$. In other words, the norm of the outermost particle scales at most as $\mathcal{O}(\log N)$.
\end{enumerate}

\section{Experiments}\label{sec:experiments}
In this section, we present various two-dimensional toy examples to demonstrate the effects of adding a preconditioning matrix on the sampling behavior, followed by high-dimensional Bayesian total-variation regularized image deconvolution, and Bayesian neural network training. We compare against the MCMC-based baselines ULA, MALA and the mirror Langevin algorithm (MLA) \cite{hsieh2018mirrored,jiang2021mirror} and the kernel-based method SVGD \cite{liu2016stein} for the low-dimensional examples. For the Bayesian neural network experiment, we compare with various kernel-based methods to be detailed later. Throughout, we fix $\beta = 1$ without loss of generality, see Appendix \ref{appsec:finpartres}.

\subsection{Gaussian low particle regime}
We first consider the evolution of (P)BRWP for the two-dimensional standard Gaussian $\gN(0, I_2)$, with the identity preconditioner $M=I$. In this case, for a fixed $T \in (0,1)$, \Cref{thm:discreteTime} gives that the RWPO (with parameter $T$) of the iterated distributions should approach $\gN(0, I_2)$. Moreover, the RWPO for this quadratic distribution is given by the closed-form Gaussian kernel \cref{eq:kernelMarkovGaussian}. In particular, each point gets mapped into a Gaussian with covariance $(\frac{1}{2} + \frac{1}{2 T})^{-1}I$.

\Cref{fig:2dGaussianLowPartEvo} demonstrates the density evolution of the RWPO of the (P)BRWP iterations, evaluated with 5 particles, regularization $T=0.2$ and step-size $\eta=0.1$. We observe as expected from \Cref{thm:discreteTime} that the RWPO converges to that of the standard Gaussian, even in the finite particle setting. \Cref{fig:2dGaussianParticleAblation} demonstrates the effect of varying the number of particles from 3 to 6, evaluated at convergence at iteration 100 and with the same parameters above. In particular, we observe that the contours of the RWPO are initially triangular, then square and pentagonal, gradually becoming smoother and better approximating the standard Gaussian. We observe as well in the 6-particle case that while the contour is pentagonal, the density is more rounded at the origin, indicating that there is one particle there surrounded by 5 other particles.

\begin{figure}
    \centering
    \setlength{\tabcolsep}{1pt}
    \renewcommand{\arraystretch}{0}
    \noindent\makebox[\textwidth]{
    \begin{tblr}{
        colspec={ccccc},
        }
     Iterations& 1 & 10 & 20 & 100 \\
    & \adjincludegraphics[height=2.2cm,trim={2.8cm 2.1cm 4cm 2.5cm},clip,valign=c]{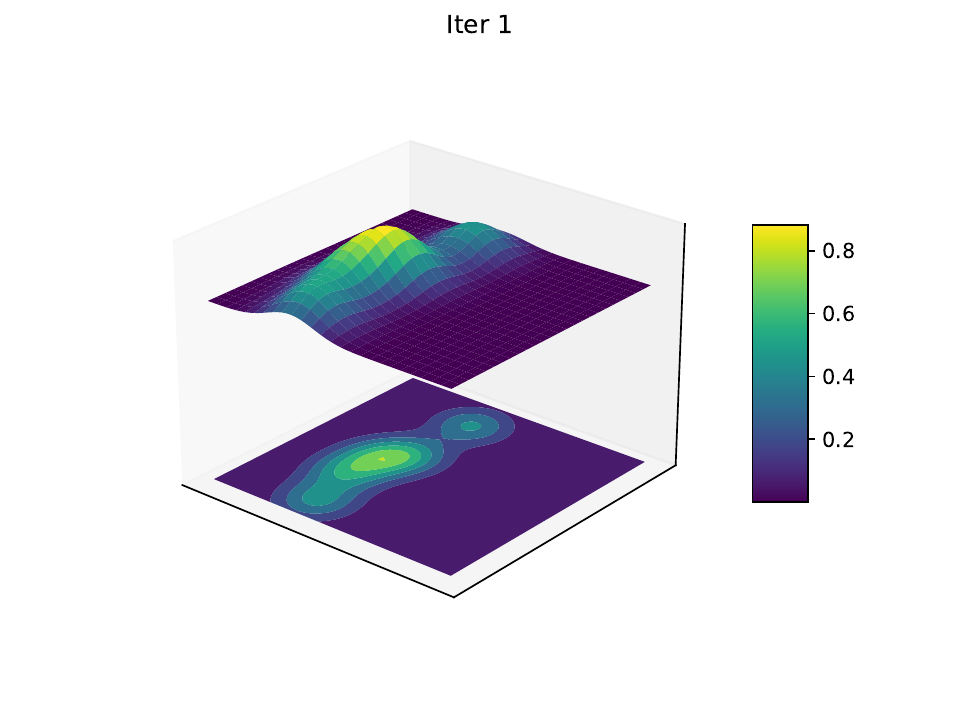}& \adjincludegraphics[height=2.2cm,trim={2.8cm 2.1cm 4cm 2.5cm},clip,valign=c]{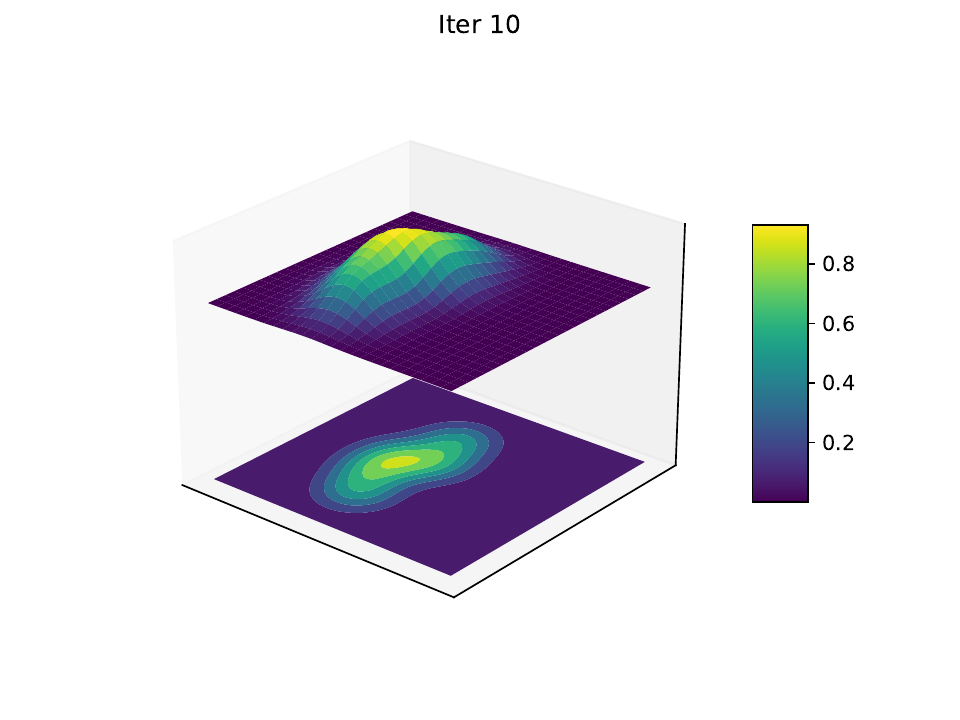}& \adjincludegraphics[height=2.2cm,trim={2.8cm 2.1cm 4cm 2.5cm},clip,valign=c]{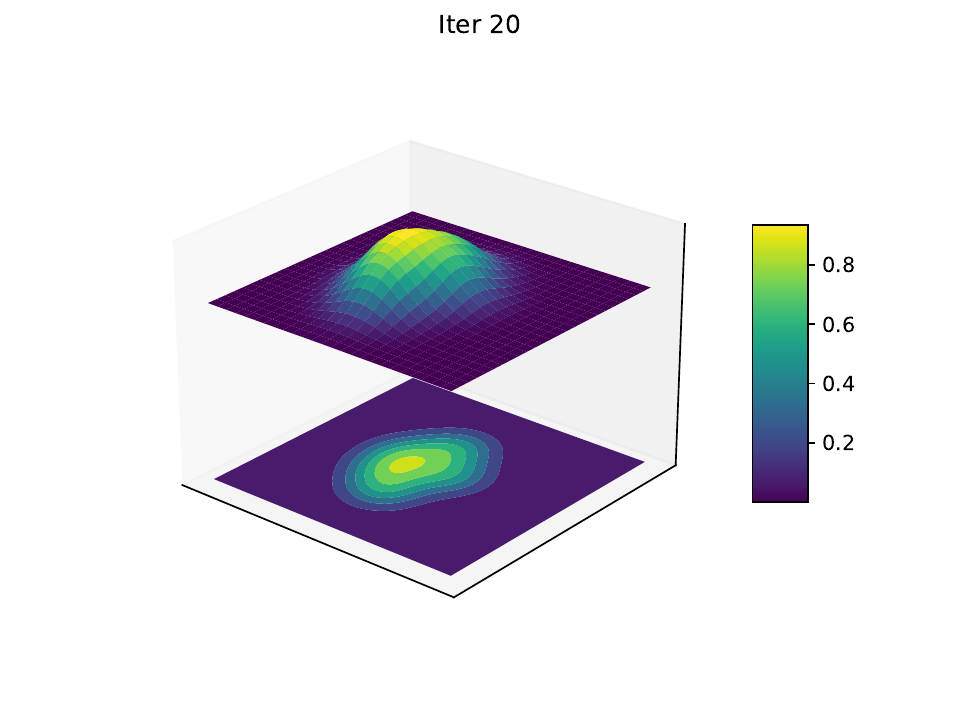}& \adjincludegraphics[height=2.2cm,trim={2.8cm 2.1cm 4cm 2.5cm},clip,valign=c]{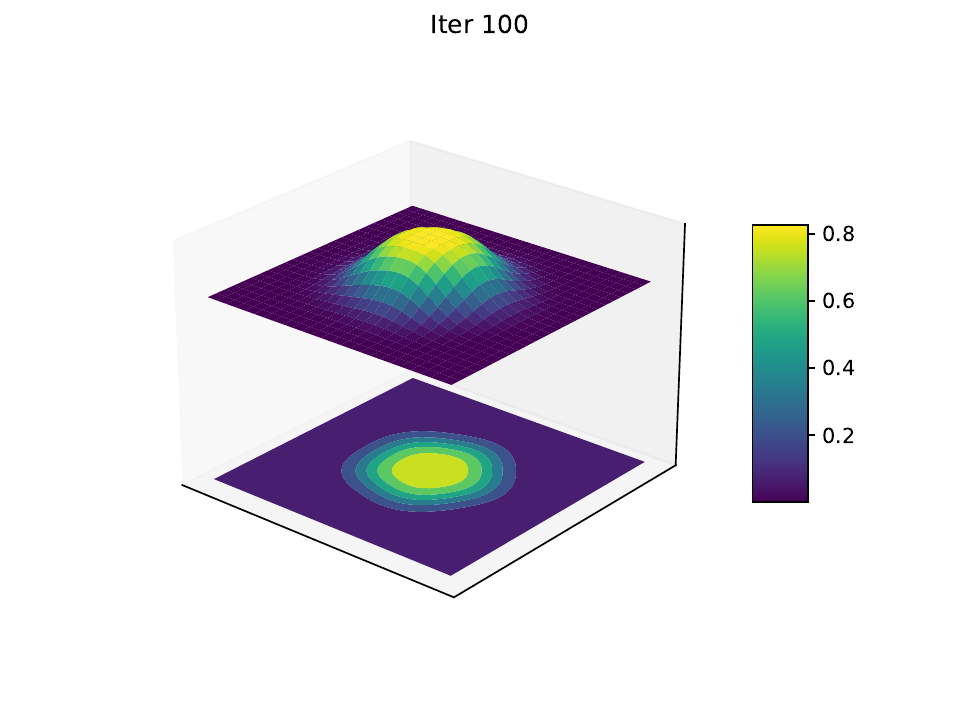}
    \end{tblr}}
    \caption{Evolution of the preconditioned regularized Wasserstein proximal $\wprox^I_{T=0.2,I}$ for the 2-dimensional standard Gaussian, done with 5 particles and step-size of $\eta=0.1$. The bandwidth is automatically determined by the regularization. As suggested by theory, the PRWPO of the particles approaches the standard Gaussian.}
    \label{fig:2dGaussianLowPartEvo}
\end{figure}
\begin{figure}
    \centering
    \setlength{\tabcolsep}{1pt}
    \renewcommand{\arraystretch}{0}
    \noindent\makebox[\textwidth]{
    \begin{tblr}{
        colspec={ccccc},
        }
     Particles& 3 & 4 & 5 & 6 \\
    & \adjincludegraphics[height=2.2cm,trim={2.8cm 2.1cm 4cm 2.5cm},clip,valign=c]{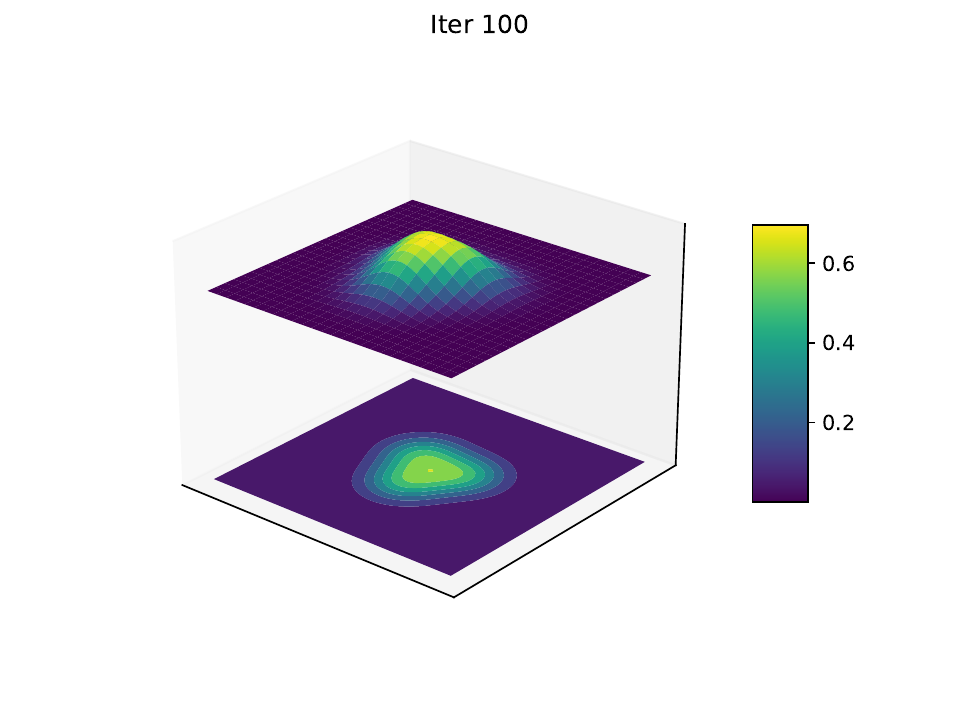}& \adjincludegraphics[height=2.2cm,trim={2.8cm 2.1cm 4cm 2.5cm},clip,valign=c]{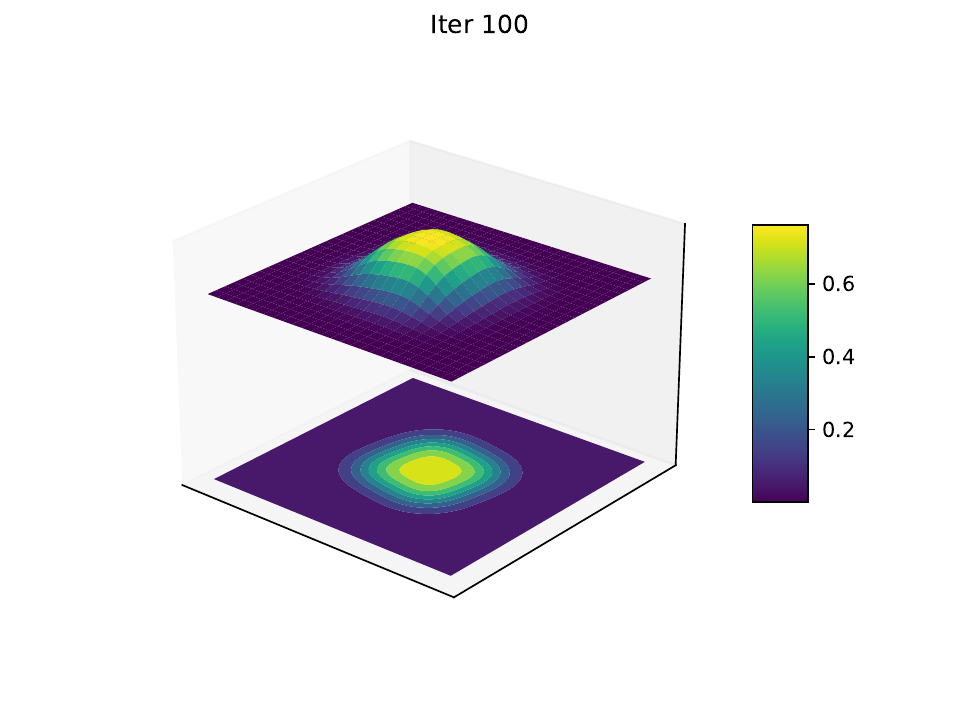}& \adjincludegraphics[height=2.2cm,trim={2.8cm 2.1cm 4cm 2.5cm},clip,valign=c]{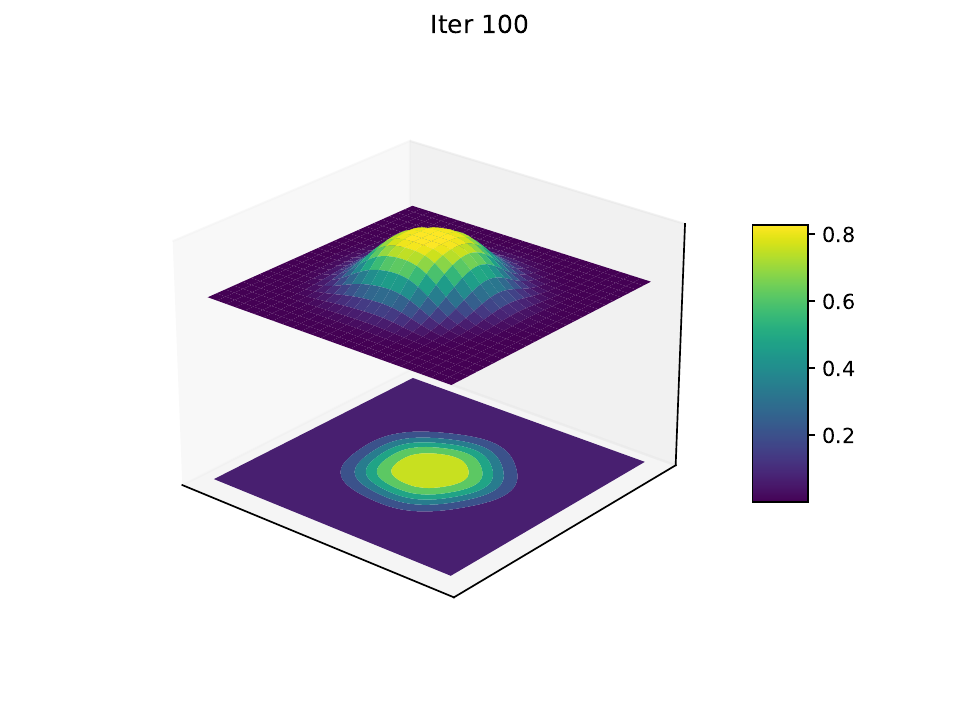}& \adjincludegraphics[height=2.2cm,trim={2.8cm 2.1cm 4cm 2.5cm},clip,valign=c]{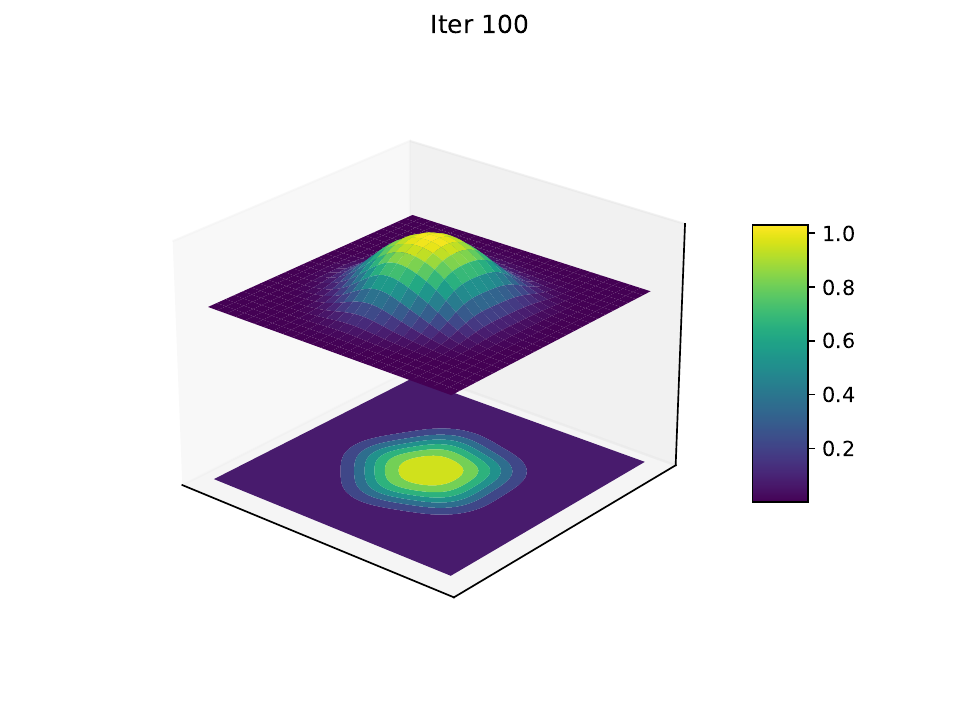}
    \end{tblr}}
    \caption{Densities of the PRWPO $\wprox^I_{0.2I}$ for the 2-dimensional standard Gaussian at iteration 100, done with $n\in \{3,4,5,6\}$ particles and a step-size of $\eta=0.1$. We observe that the density of the Wasserstein proximal gradually becomes more spherical and Gaussian-like.}
    \label{fig:2dGaussianParticleAblation}
\end{figure}

\subsection{Bimodal distribution}
As a non-convex potential, we consider the two-moons bimodal distribution as given in \cite{wang2022accelerated,tan2024noise},
\begin{equation*}
    p(x) \propto \exp(-2(\|x\|-3)^2) \left[\exp(-2(x_1-3)^2) + \exp(-2(x_1+3)^2)\right].
\end{equation*}
This is generated by the potential $V$ with corresponding gradient $\nabla V$ as follows:
\begin{subequations}
\begin{align}
        V(x) &= 2 (\|x\|-3)^2 - 2\log \left[\exp(-2(x_1-3)^2) + \exp(-2(x_1+3)^2)\right],\\
        \nabla V(x) &= 4  \frac{(\|x\|-3)x}{\|x\|} + \frac{4(x_1-3) \exp(-2(x_1-3)^2) +4(x_1+3) \exp(-2(x_1+3)^2)}{\exp(-2(x_1-3)^2) + \exp(-2(x_1+3)^2)} \mathrm{e}_1,
\end{align}
\end{subequations}
where $\mathrm{e}_1$ is the first standard basis vector. We consider a low-variance initialization and a high-variance initialization. For the low-variance initialization, we consider the same setup as \cite{tan2024noise}, where the particles are initially distributed as $\gN(0, I)$; for the high-variance setup, the initial distribution is $\gN(0, 6I)$. The preconditioning matrix is taken to be $M=\diag([1,4])$.

We consider the KL divergence as approximated using a Gaussian KDE, with a bandwidth of $0.1$, numerically integrated over $[-5,5]^2$ with a grid size of $0.01$. \Cref{fig:KL_banana} demonstrates that in this case, adding preconditioning has minimal effect on the convergence rate of BRWP and PBRWP, and indeed leads to a slightly higher bias. The bias can be seen qualitatively in \Cref{fig:banana_particles}, which demonstrates that the particles retain the same structured behavior at convergence, but takes a different path at small iterations. In this case, SVGD converges to a slightly better particle collection at a slower per-iteration rate. \footnote{While \cite{liu2016stein} propose an adaptive bandwidth choice based on the median of pairwise distances, this causes instability in our non-log-concave experiments. We therefore only consider a fixed bandwidth, with parameters chosen via grid search.}
\begin{figure}
    \centering
    \includegraphics[width=0.9\linewidth]{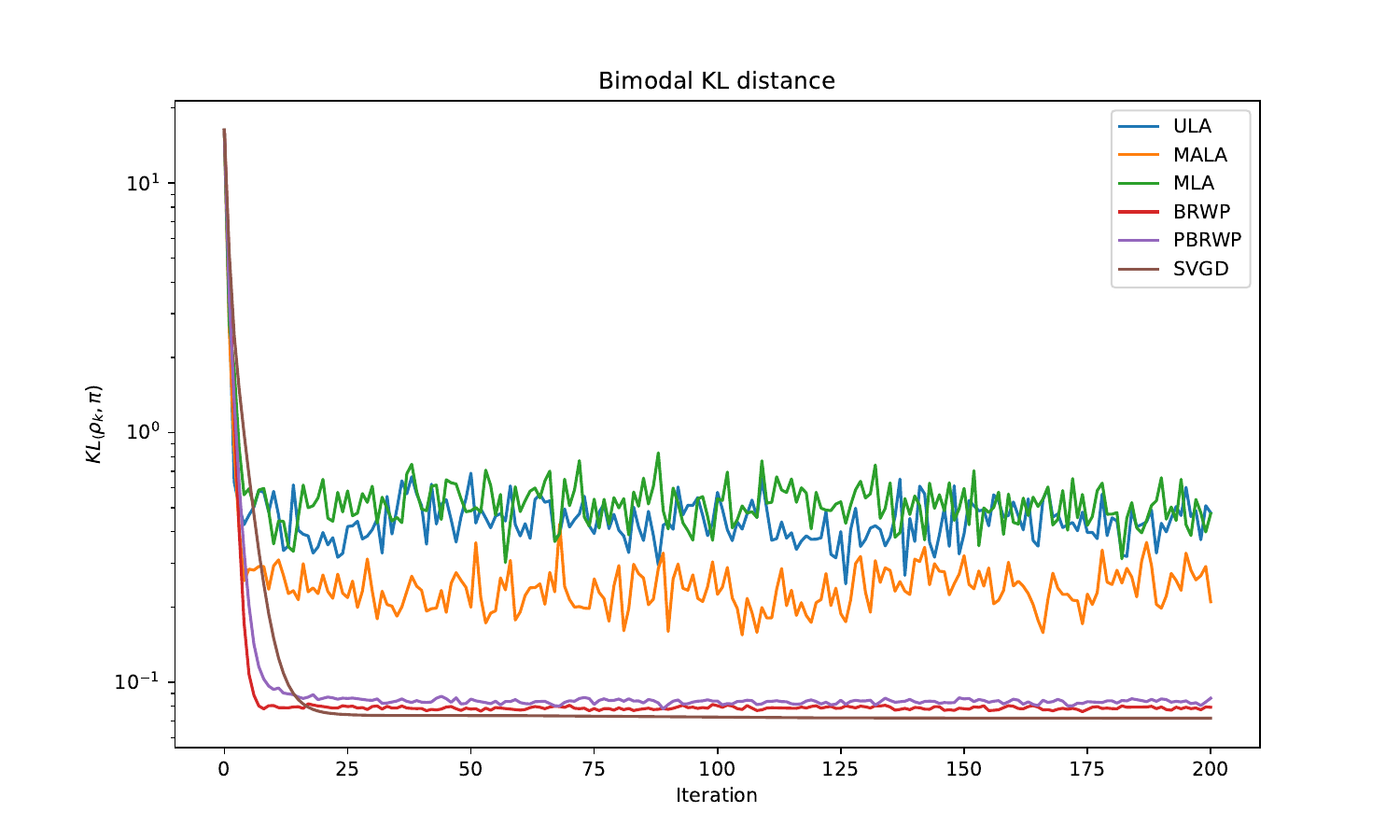}
    \caption{Evolution of the KL divergence between baselines and BRWP-based methods for the bimodal distribution. Applied with 100 particles with initial distribution $\gN(0, I)$, fixed step-size $\eta =0.1$ and regularization parameter $T=0.05$. We observe that the BRWP-based methods converge more smoothly even in later iterations, and the particles better approximate the target distribution.}
    \label{fig:KL_banana}
\end{figure}
\begin{figure}
    \centering
    \renewcommand{\arraystretch}{0}
    \noindent\makebox[\textwidth]{
    \begin{tblr}{
        colspec={cccccc},
        }
    ULA & MALA & MLA & SVGD & BRWP & PBRWP\\
     \adjincludegraphics[width=0.17\textwidth,trim={3.2cm 2.6cm 3.2cm 2.6cm},clip,valign=c]{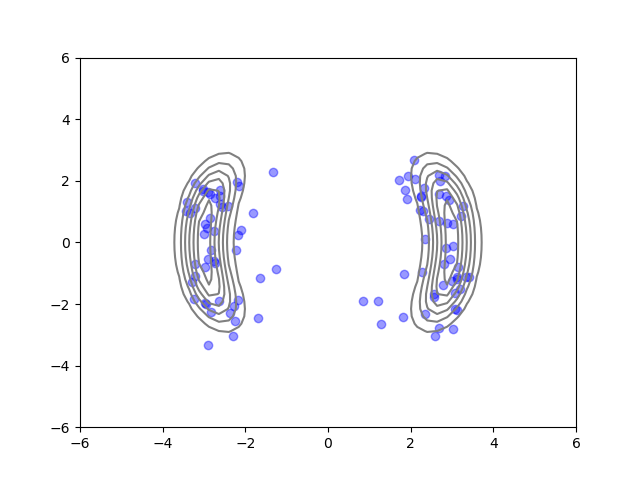}& \adjincludegraphics[width=0.17\textwidth,trim={3.2cm 2.6cm 3.2cm 2.6cm},clip,valign=c]{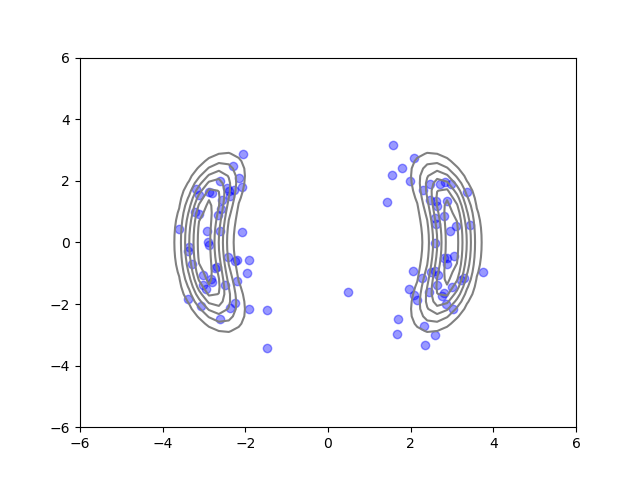}& \adjincludegraphics[width=0.17\textwidth,trim={3.2cm 2.6cm 3.2cm 2.6cm},clip,valign=c]{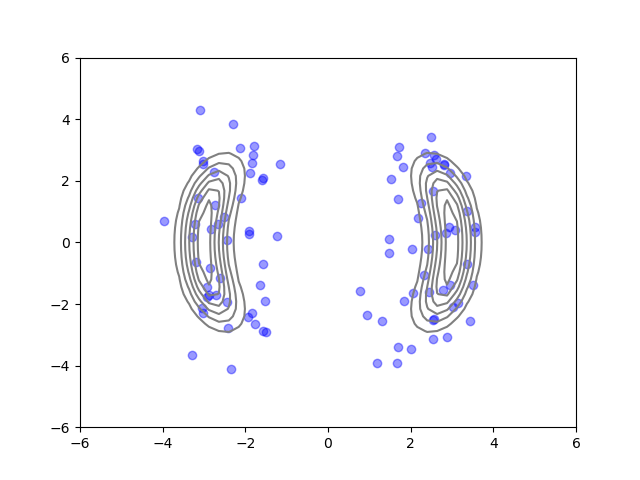}& 
     \adjincludegraphics[width=0.17\textwidth,trim={3.2cm 2.6cm 3.2cm 2.6cm},clip,valign=c]{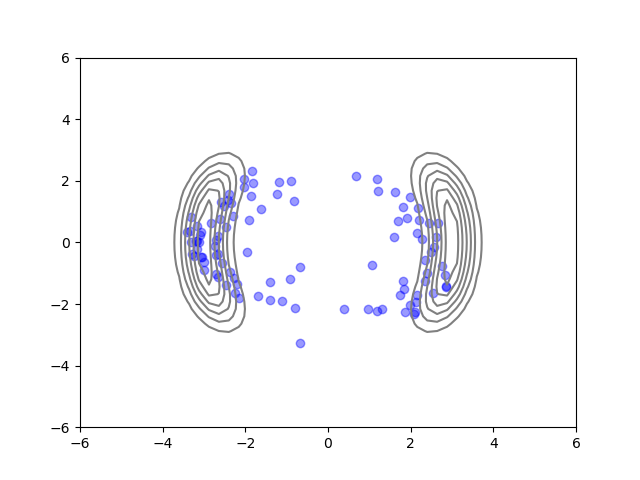}&
     \adjincludegraphics[width=0.17\textwidth,trim={3.2cm 2.6cm 3.2cm 2.6cm},clip,valign=c]{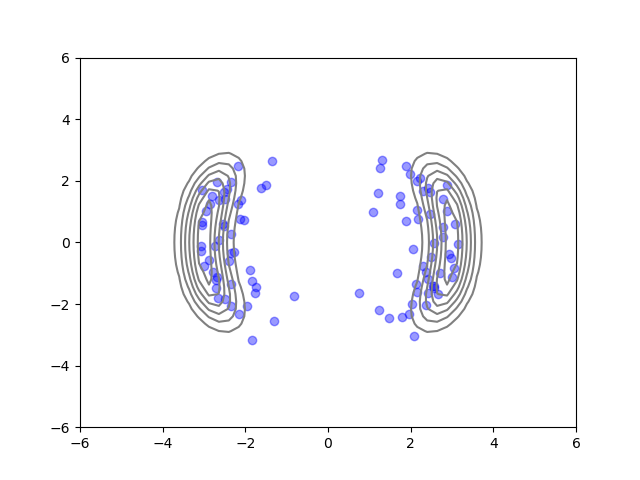}& \adjincludegraphics[width=0.17\textwidth,trim={3.2cm 2.6cm 3.2cm 2.6cm},clip,valign=c]{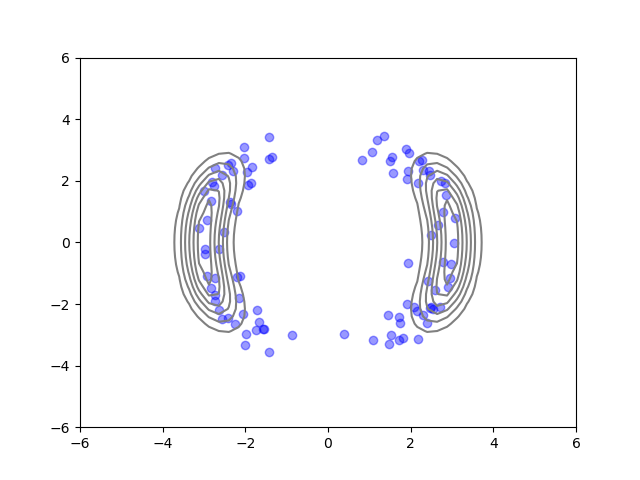}\\
     \adjincludegraphics[width=0.17\textwidth,trim={3.2cm 2.6cm 3.2cm 2.6cm},clip,valign=c]{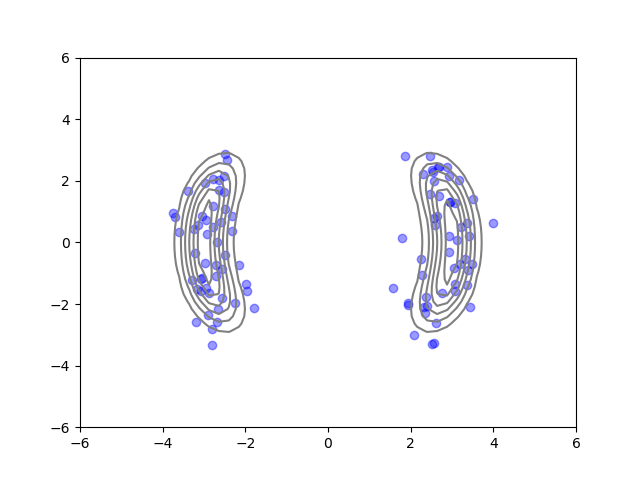}& \adjincludegraphics[width=0.17\textwidth,trim={3.2cm 2.6cm 3.2cm 2.6cm},clip,valign=c]{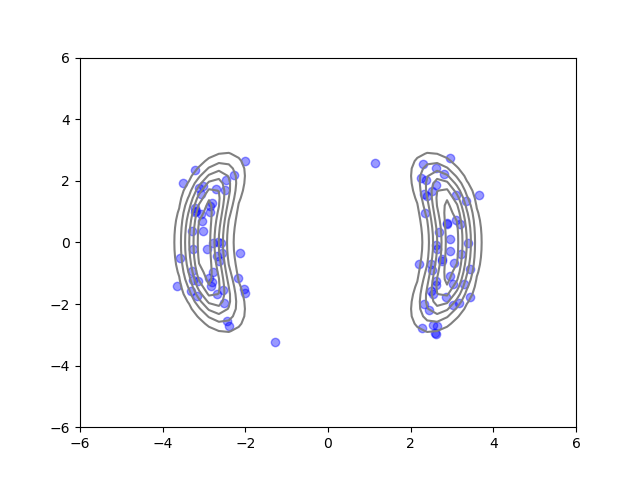}& \adjincludegraphics[width=0.17\textwidth,trim={3.2cm 2.6cm 3.2cm 2.6cm},clip,valign=c]{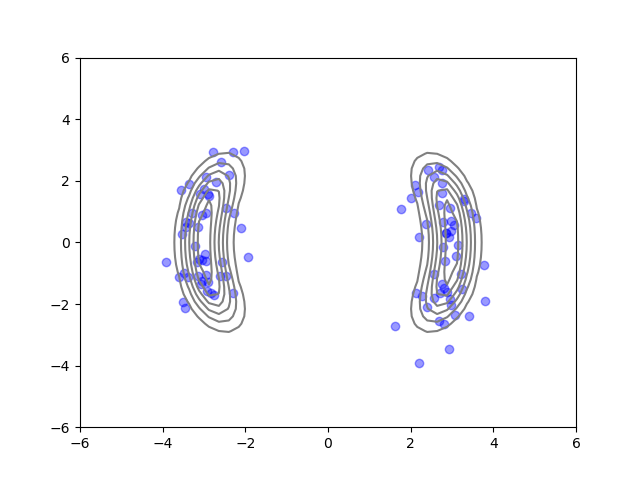}&
     \adjincludegraphics[width=0.17\textwidth,trim={3.2cm 2.6cm 3.2cm 2.6cm},clip,valign=c]{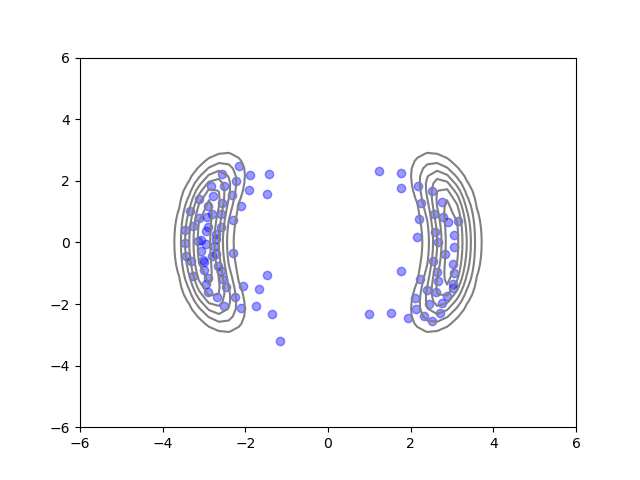}&
     \adjincludegraphics[width=0.17\textwidth,trim={3.2cm 2.6cm 3.2cm 2.6cm},clip,valign=c]{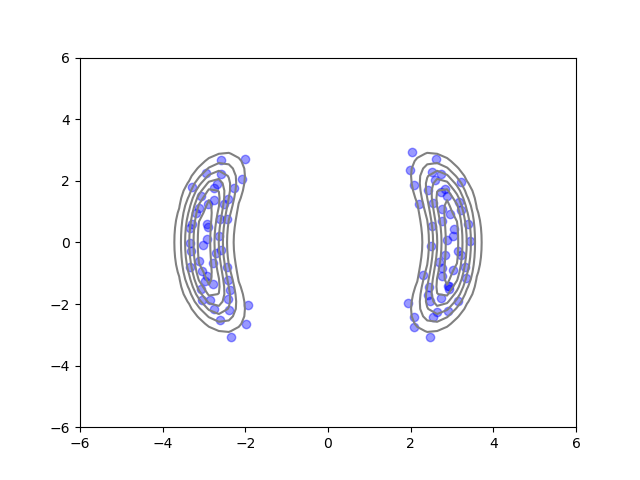}& \adjincludegraphics[width=0.17\textwidth,trim={3.2cm 2.6cm 3.2cm 2.6cm},clip,valign=c]{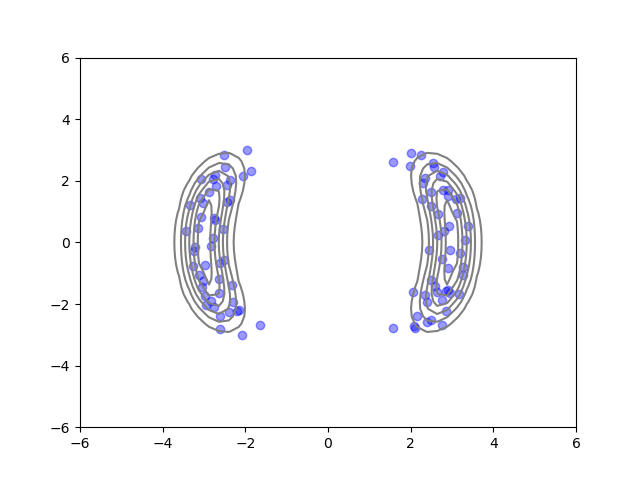}\\
     \adjincludegraphics[width=0.17\textwidth,trim={3.2cm 2.6cm 3.2cm 2.6cm},clip,valign=c]{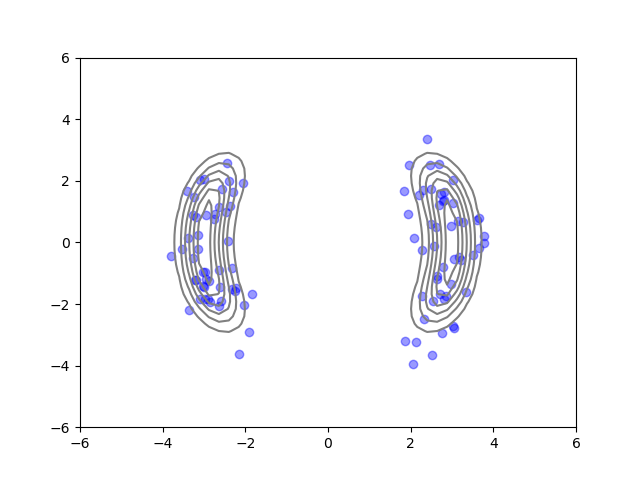}& \adjincludegraphics[width=0.17\textwidth,trim={3.2cm 2.6cm 3.2cm 2.6cm},clip,valign=c]{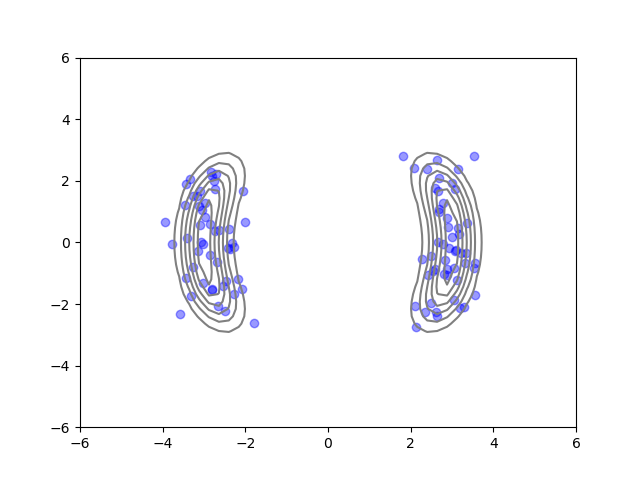}& \adjincludegraphics[width=0.17\textwidth,trim={3.2cm 2.6cm 3.2cm 2.6cm},clip,valign=c]{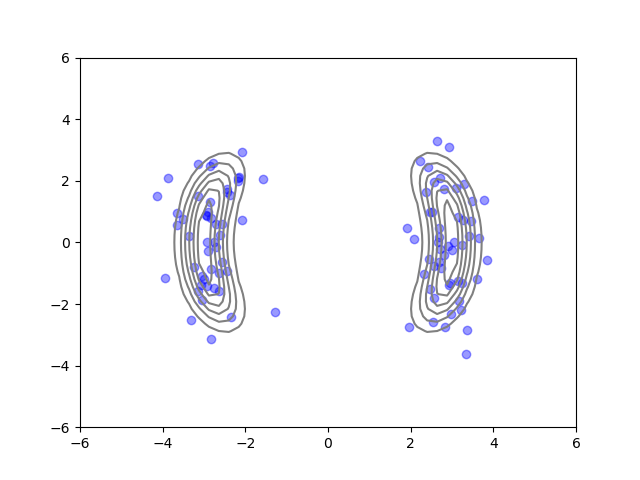}&
     \adjincludegraphics[width=0.17\textwidth,trim={3.2cm 2.6cm 3.2cm 2.6cm},clip,valign=c]{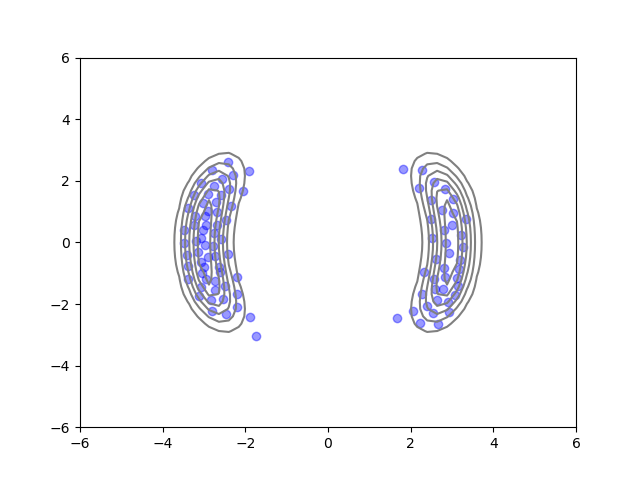}&
     \adjincludegraphics[width=0.17\textwidth,trim={3.2cm 2.6cm 3.2cm 2.6cm},clip,valign=c]{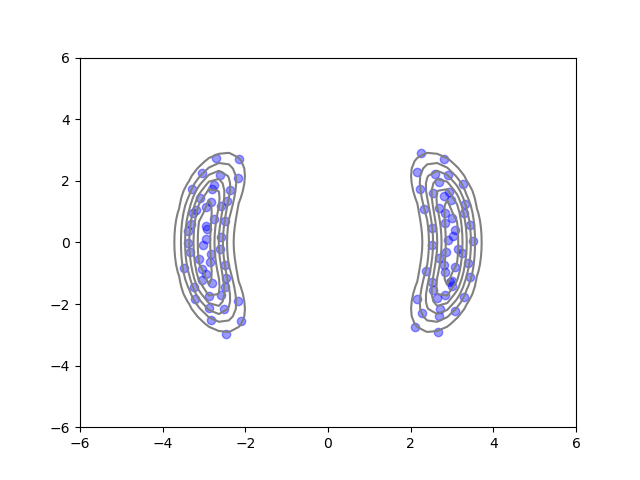}& \adjincludegraphics[width=0.17\textwidth,trim={3.2cm 2.6cm 3.2cm 2.6cm},clip,valign=c]{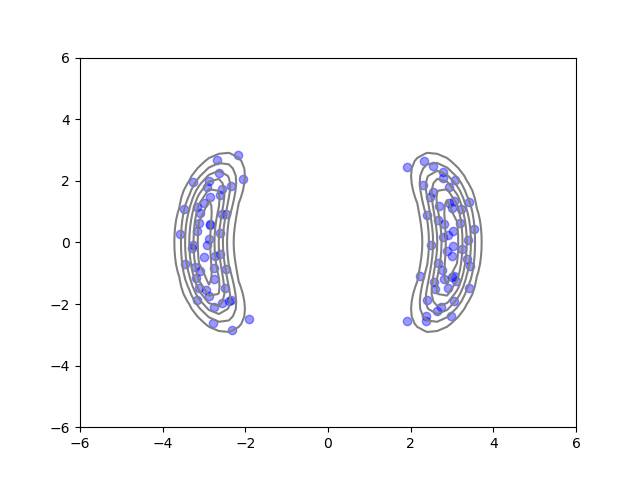}\\
    \end{tblr}}
    \caption{Evolution of the various methods for the bimodal distribution at different iterations, with initial distribution $\gN(0, I)$, which is contained between the moons. Evaluated with 100 particles, at iterations 2, 5 and 10 in the top, middle, and bottom rows respectively. Observe the stable behavior of SVGD, BRWP and PBRWP, as opposed to the randomness of the Langevin methods.}
    \label{fig:banana_particles}
\end{figure}

We observe a different compression phenomenon with the high-variance initialization, when the samples are initially distributed around the modes instead of between the modes. \Cref{fig:KL_banana_bigvar} demonstrates that in this case, the preconditioning has a significant acceleration effect, where the KL divergence decreases faster than BRWP. Moreover, \Cref{fig:banana_bigvar} demonstrates that PBRWP particles are less noisy than BRWP particles at low iterations. 

\begin{figure}
    \centering
    \includegraphics[width=0.9\linewidth]{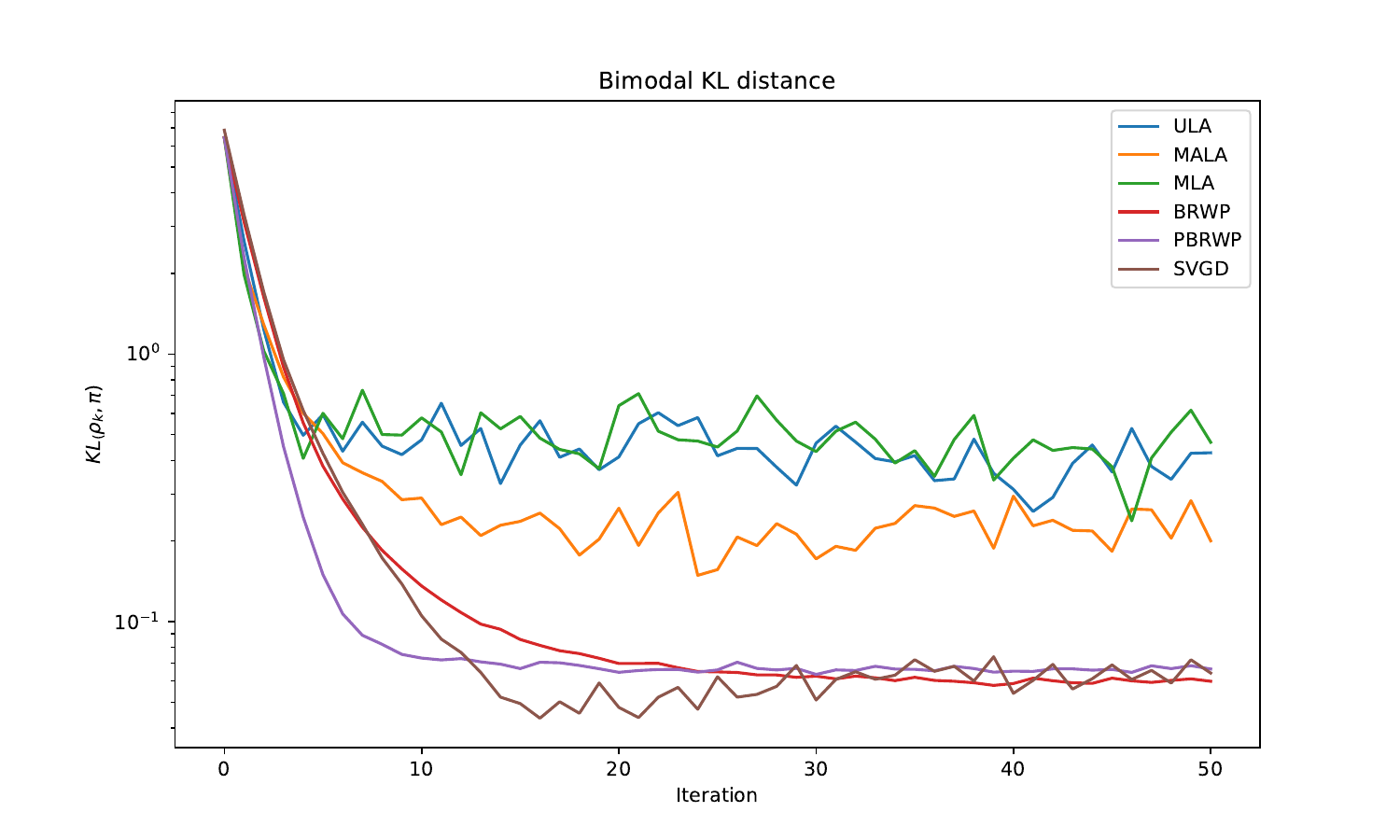}
    \caption{Evolution of the KL divergence between baselines and BRWP-based methods for the bimodal distribution, initialized with large variance $\gN(0, 6I)$. Applied with 100 particles, fixed step-size $\eta =0.1$ and regularization parameter $T=0.05$. In this case, the preconditioning accelerates the convergence of BRWP. SVGD is able to get to a lower divergence around iteration 20, but the divergence increases again at later iterations.}
    \label{fig:KL_banana_bigvar}
\end{figure}
\begin{figure}
    \centering
    \setlength{\tabcolsep}{1pt}
    \renewcommand{\arraystretch}{0}
    \noindent\makebox[\textwidth]{
    \begin{tblr}{
        colspec={cccccc},
        }
     ULA & MALA & MLA & SVGD & BRWP & PBRWP\\
     \adjincludegraphics[width=0.17\textwidth,trim={3.2cm 2.6cm 3.2cm 2.6cm},clip,valign=c]{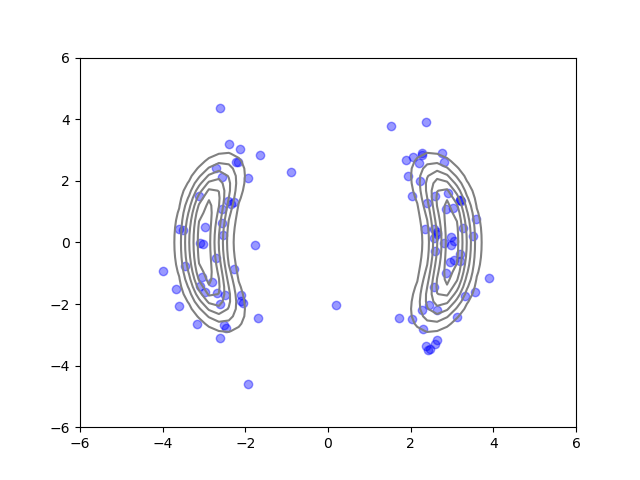}& \adjincludegraphics[width=0.17\textwidth,trim={3.2cm 2.6cm 3.2cm 2.6cm},clip,valign=c]{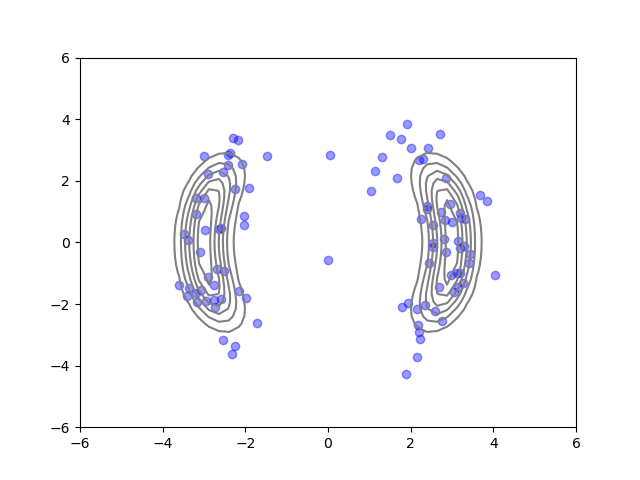}& \adjincludegraphics[width=0.17\textwidth,trim={3.2cm 2.6cm 3.2cm 2.6cm},clip,valign=c]{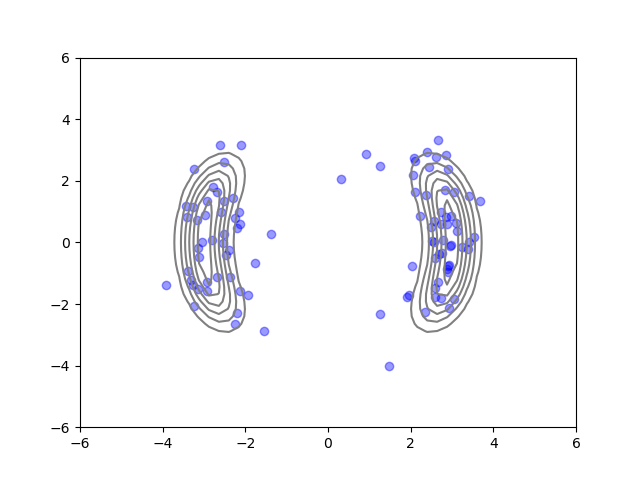}& 
     \adjincludegraphics[width=0.17\textwidth,trim={3.2cm 2.6cm 3.2cm 2.6cm},clip,valign=c]{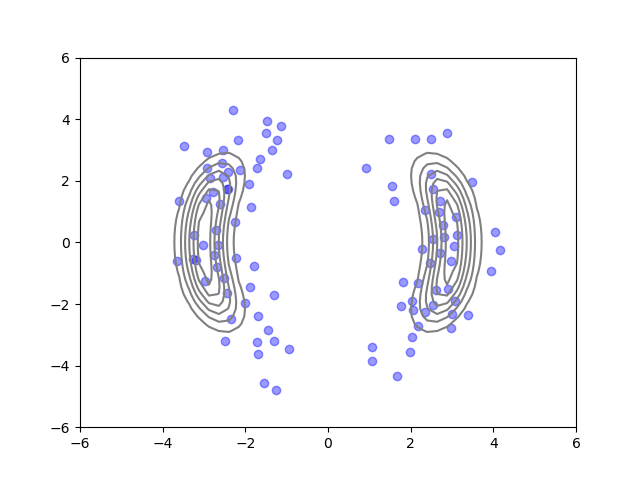}& \adjincludegraphics[width=0.17\textwidth,trim={3.2cm 2.6cm 3.2cm 2.6cm},clip,valign=c]{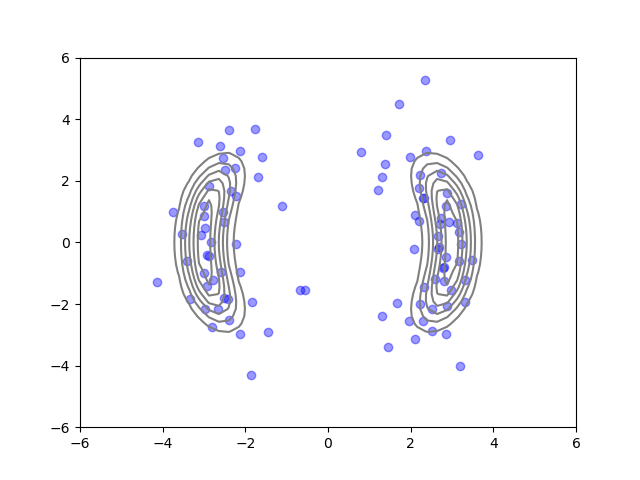}& \adjincludegraphics[width=0.17\textwidth,trim={3.2cm 2.6cm 3.2cm 2.6cm},clip,valign=c]{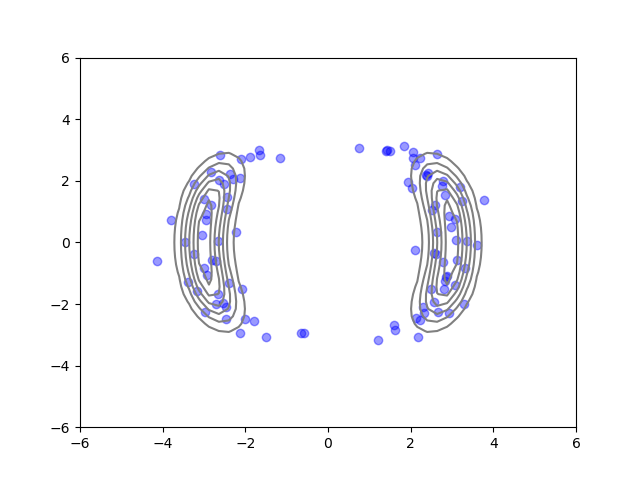}\\
     \adjincludegraphics[width=0.17\textwidth,trim={3.2cm 2.6cm 3.2cm 2.6cm},clip,valign=c]{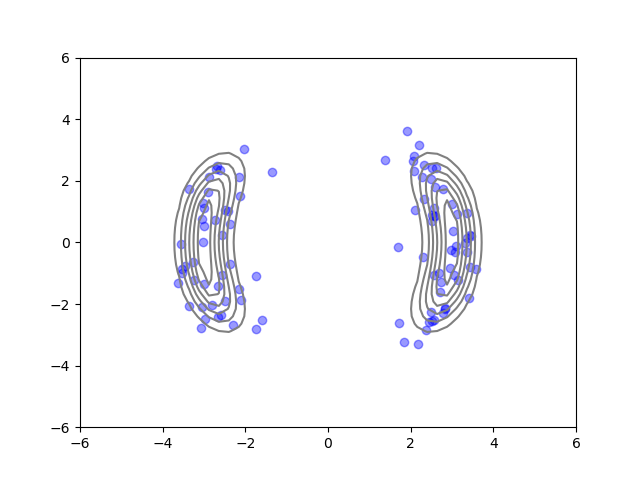}& \adjincludegraphics[width=0.17\textwidth,trim={3.2cm 2.6cm 3.2cm 2.6cm},clip,valign=c]{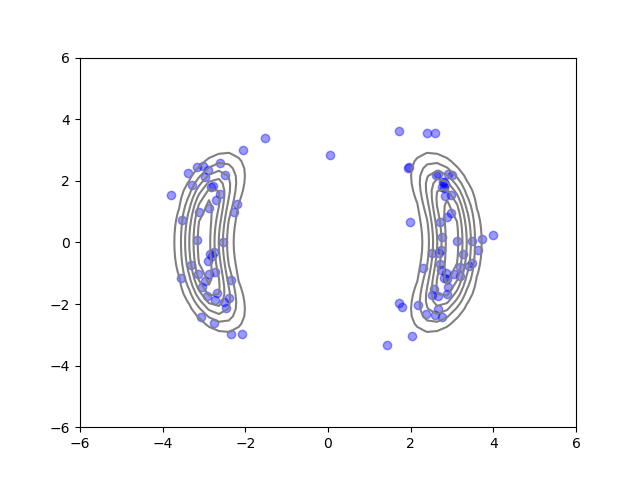}& \adjincludegraphics[width=0.17\textwidth,trim={3.2cm 2.6cm 3.2cm 2.6cm},clip,valign=c]{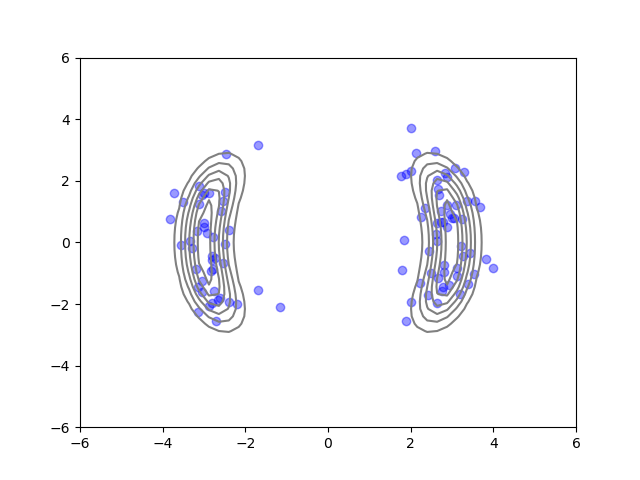}& 
     \adjincludegraphics[width=0.17\textwidth,trim={3.2cm 2.6cm 3.2cm 2.6cm},clip,valign=c]{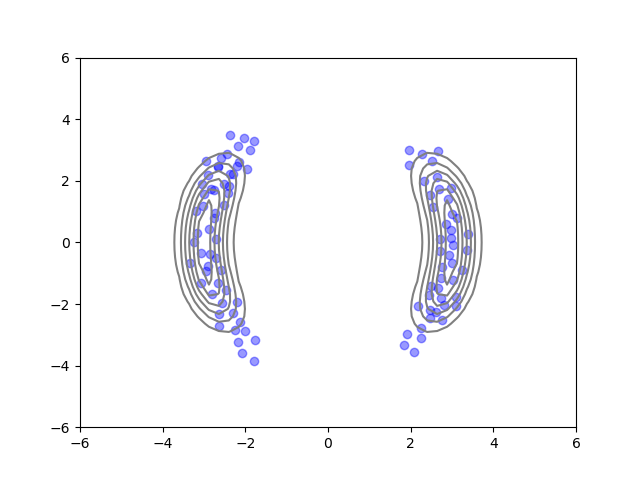}&
     \adjincludegraphics[width=0.17\textwidth,trim={3.2cm 2.6cm 3.2cm 2.6cm},clip,valign=c]{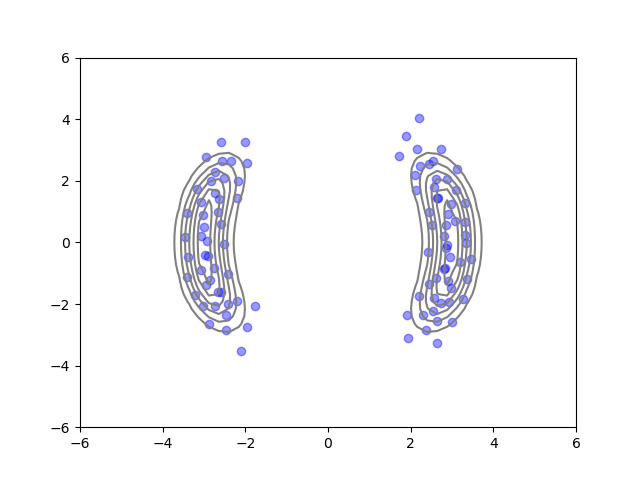}& \adjincludegraphics[width=0.17\textwidth,trim={3.2cm 2.6cm 3.2cm 2.6cm},clip,valign=c]{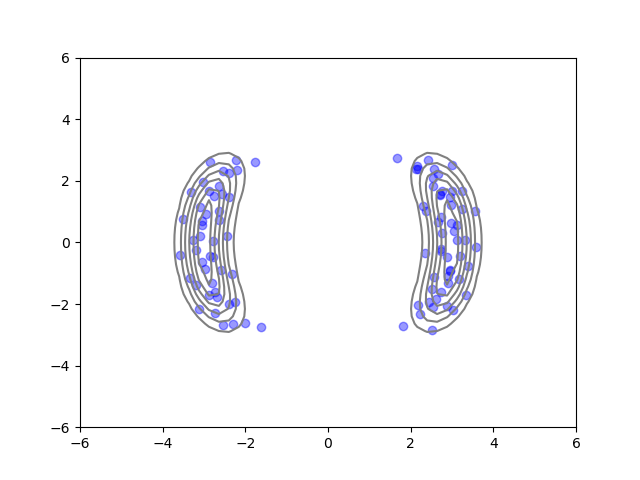}\\
     \adjincludegraphics[width=0.17\textwidth,trim={3.2cm 2.6cm 3.2cm 2.6cm},clip,valign=c]{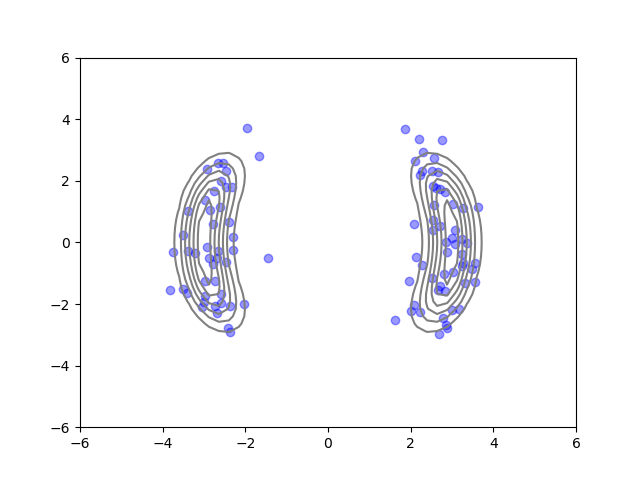}& \adjincludegraphics[width=0.17\textwidth,trim={3.2cm 2.6cm 3.2cm 2.6cm},clip,valign=c]{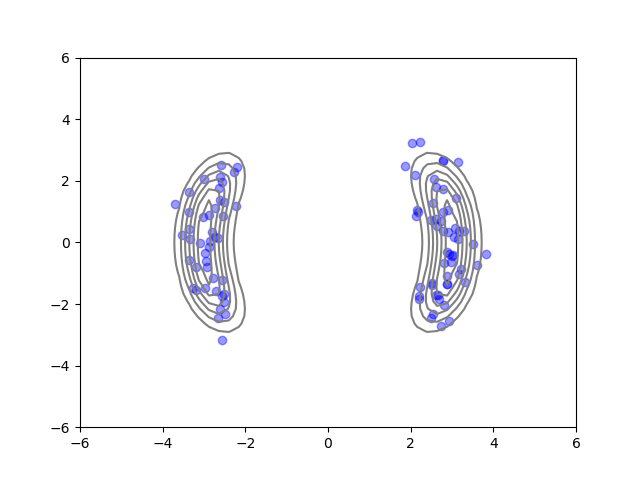}& \adjincludegraphics[width=0.17\textwidth,trim={3.2cm 2.6cm 3.2cm 2.6cm},clip,valign=c]{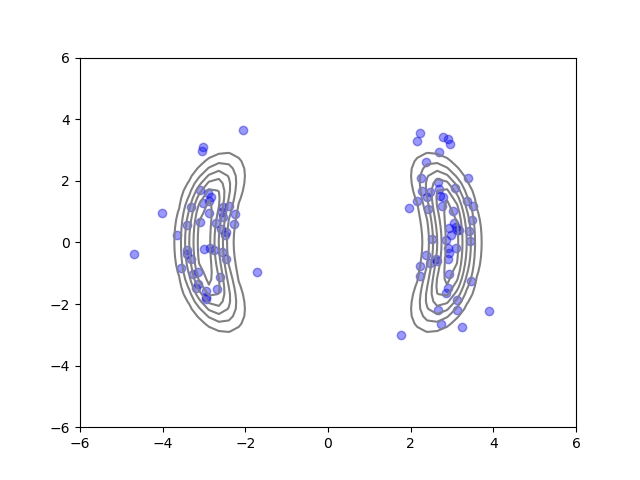}&
     \adjincludegraphics[width=0.17\textwidth,trim={3.2cm 2.6cm 3.2cm 2.6cm},clip,valign=c]{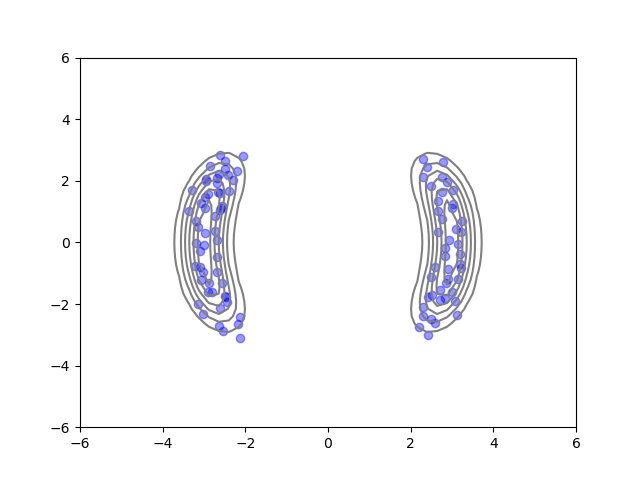}&
     \adjincludegraphics[width=0.17\textwidth,trim={3.2cm 2.6cm 3.2cm 2.6cm},clip,valign=c]{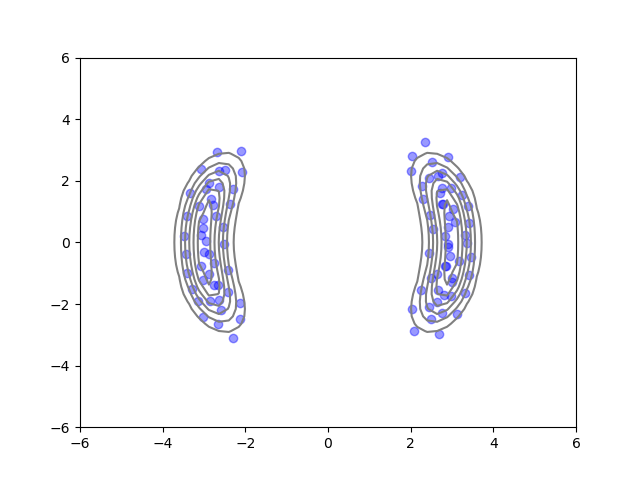}& \adjincludegraphics[width=0.17\textwidth,trim={3.2cm 2.6cm 3.2cm 2.6cm},clip,valign=c]{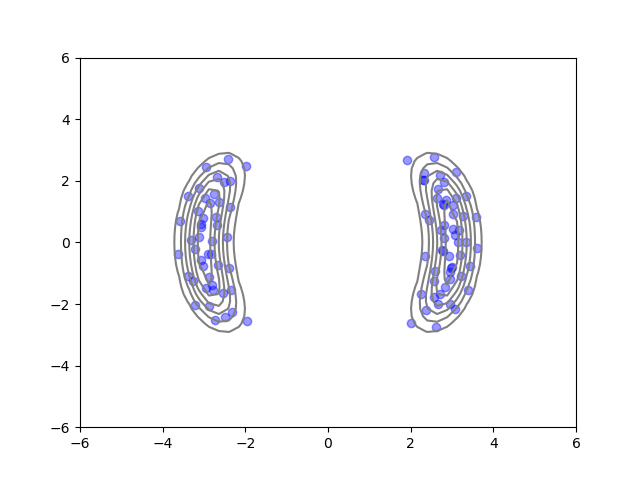}\\
    \end{tblr}}
    \caption{Evolution of the various methods for the bimodal distribution, with large initial variance $\gN(0, 6I)$, which surrounds the moons. Evaluated with 100 particles, at iterations 2, 5 and 10 in the top to bottom rows respectively. We observe again the structure behavior at convergence of the noise-free BRWP and PBRWP methods. Moreover, the preconditioning affects the empirical covariance of the particles in different directions, which can be seen in the semicircular artifacts between the modes at low iterations.}
    \label{fig:banana_bigvar}
\end{figure}

\subsection{Annulus}
Consider the two-dimensional annulus defined with the potential
\begin{align}
        V(x) = \left(\left\|\begin{pmatrix}
            1 & 0\\ 0 & 2
        \end{pmatrix}x\right\|-3\right)^2,\quad 
        \nabla V(x) = 2 \frac{\left\|\begin{pmatrix}
            1 & 0\\ 0 & 2
        \end{pmatrix} x\right\|-3}{\left\|\begin{pmatrix}
           1 & 0\\ 0 & 2
        \end{pmatrix} x\right\|}\begin{pmatrix}
            1 & 0\\ 0 & 2
        \end{pmatrix}^2 x.
\end{align}

We consider $M = \diag([4,1])$, which takes into account the scaling of the annulus in both directions. The initialization is taken to be the off-center Gaussian $\gN((2,2), I)$, and the particles should diffuse along the elliptical potential well. \Cref{fig:annulus_particles} shows the evolution of various methods from this initialization, evaluated with 100 particles. All the methods are able to find the annulus quickly, with the main differences being how fast they can cover the whole annulus. We observe that the preconditioned methods MLA and PBRWP are able to diffuse faster than their non-preconditioned counterparts, namely covering the annulus at iteration 50. To verify this numerically, we approximate the KL divergence using the same numerical integration method as in the last section. \Cref{fig:KL_annulus} shows that the PBRWP particles converge significantly faster than BRWP, and to a lower minimum than the Langevin methods. SVGD converges to a similar structured particle ensemble with low KL divergence, albeit taking slightly more iterations.

\begin{figure}
    \centering
    \setlength{\tabcolsep}{1pt}
    \renewcommand{\arraystretch}{0}
    \noindent\makebox[\textwidth]{
    \begin{tblr}{
        colspec={cccccc},
        }
     ULA & MALA & MLA & SVGD & BRWP & PBRWP\\
     \adjincludegraphics[width=0.17\textwidth,trim={3.2cm 3.1cm 3.2cm 3.1cm},clip,valign=c]{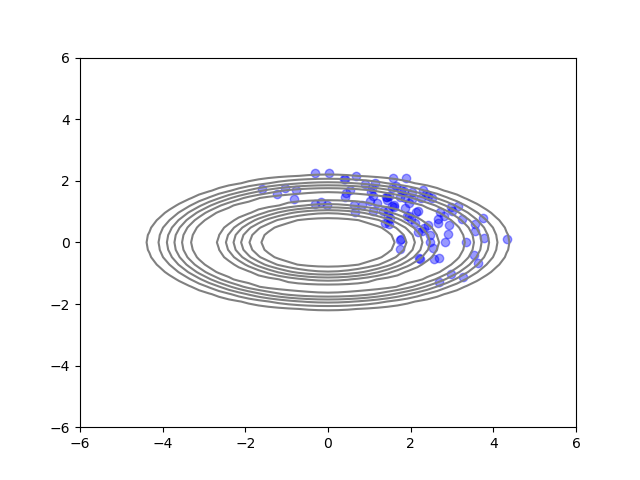}& \adjincludegraphics[width=0.17\textwidth,trim={3.2cm 3.1cm 3.2cm 3.1cm},clip,valign=c]{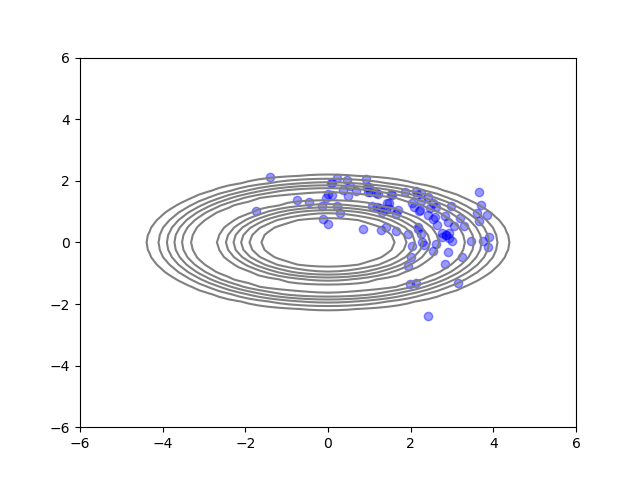}& \adjincludegraphics[width=0.17\textwidth,trim={3.2cm 3.1cm 3.2cm 3.1cm},clip,valign=c]{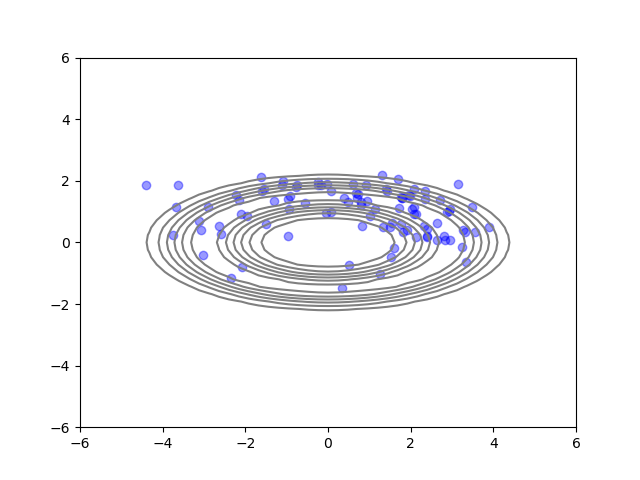}& 
     \adjincludegraphics[width=0.17\textwidth,trim={3.2cm 3.1cm 3.2cm 3.1cm},clip,valign=c]{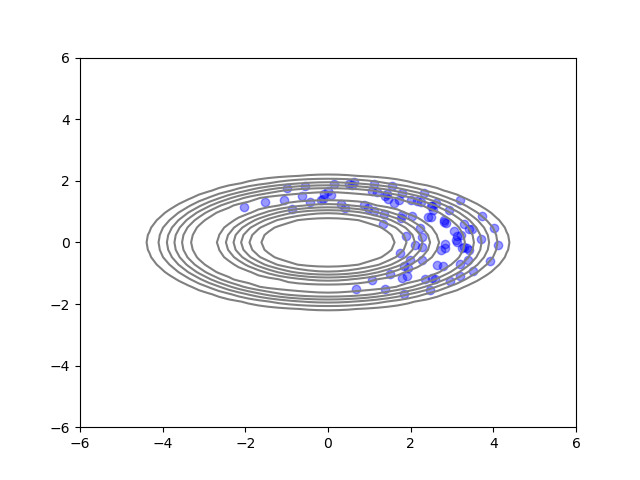}& 
     \adjincludegraphics[width=0.17\textwidth,trim={3.2cm 3.1cm 3.2cm 3.1cm},clip,valign=c]{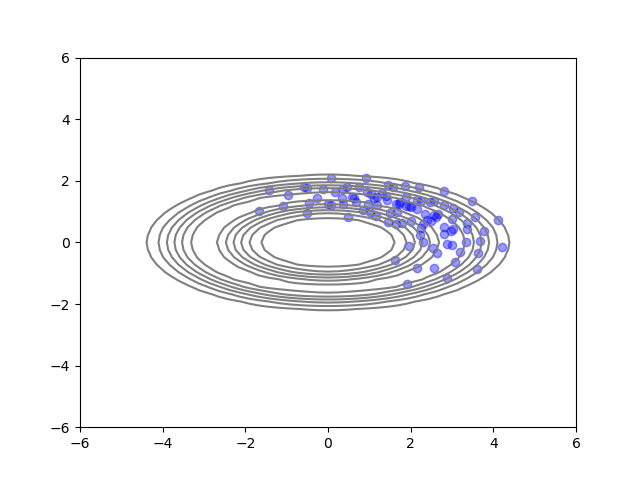}& \adjincludegraphics[width=0.17\textwidth,trim={3.2cm 3.1cm 3.2cm 3.1cm},clip,valign=c]{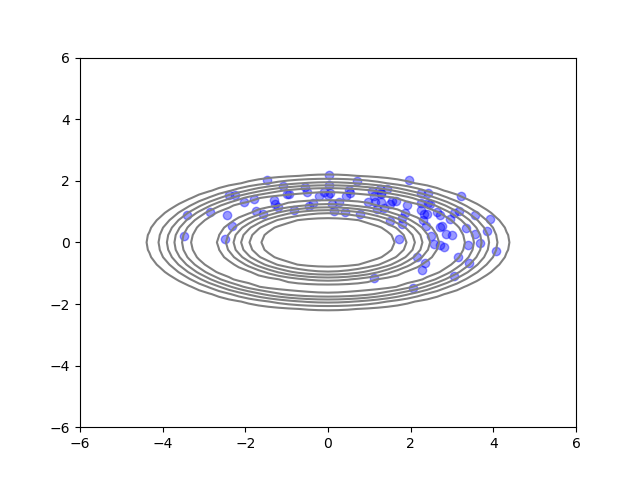}\\
     \adjincludegraphics[width=0.17\textwidth,trim={3.2cm 3.1cm 3.2cm 3.1cm},clip,valign=c]{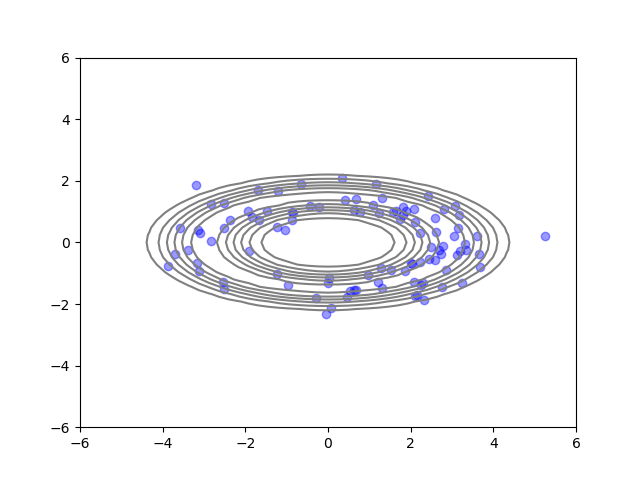}& \adjincludegraphics[width=0.17\textwidth,trim={3.2cm 3.1cm 3.2cm 3.1cm},clip,valign=c]{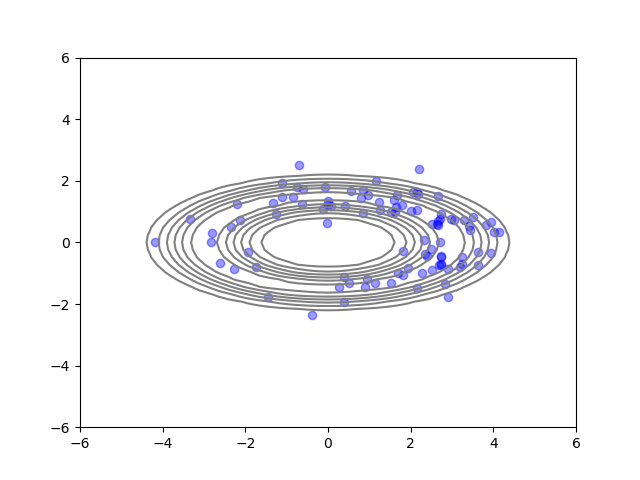}& \adjincludegraphics[width=0.17\textwidth,trim={3.2cm 3.1cm 3.2cm 3.1cm},clip,valign=c]{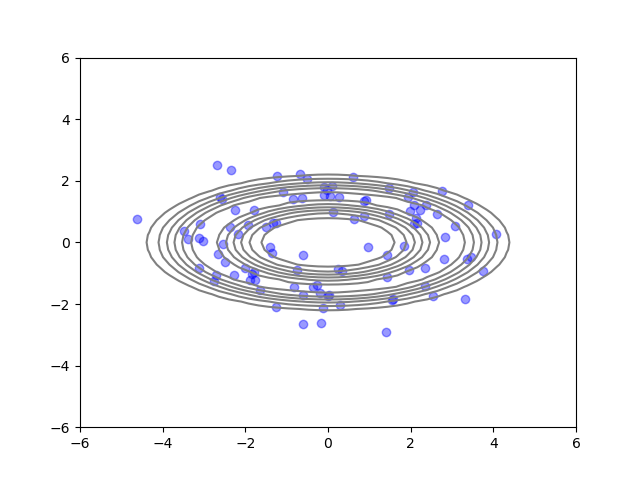}& \adjincludegraphics[width=0.17\textwidth,trim={3.2cm 3.1cm 3.2cm 3.1cm},clip,valign=c]{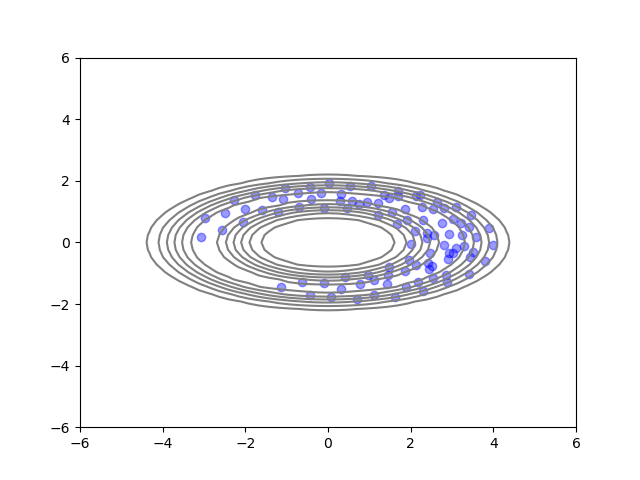}& \adjincludegraphics[width=0.17\textwidth,trim={3.2cm 3.1cm 3.2cm 3.1cm},clip,valign=c]{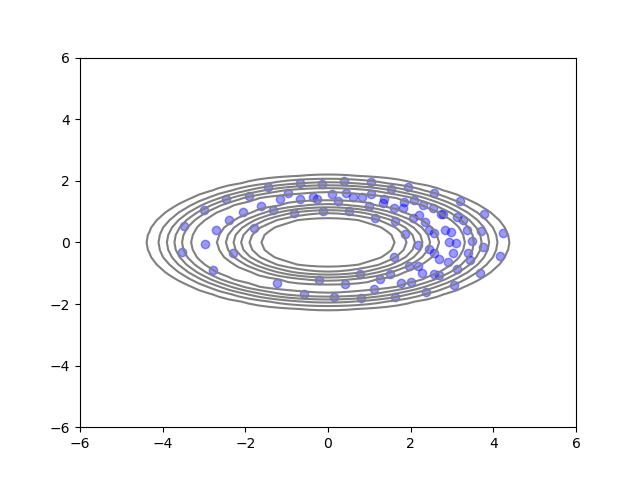}& \adjincludegraphics[width=0.17\textwidth,trim={3.2cm 3.1cm 3.2cm 3.1cm},clip,valign=c]{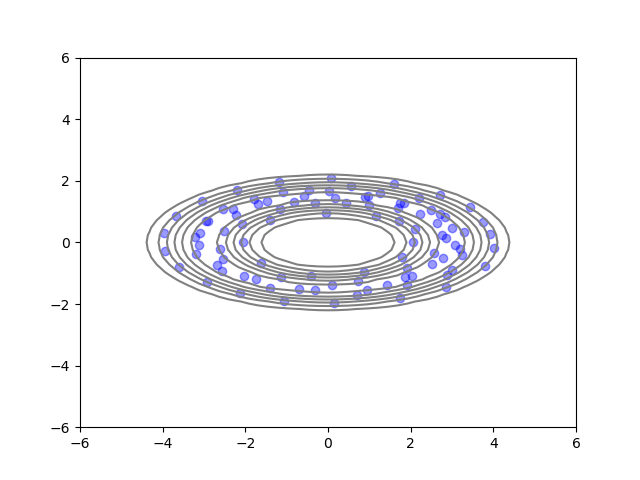}\\
     \adjincludegraphics[width=0.17\textwidth,trim={3.2cm 3.1cm 3.2cm 3.1cm},clip,valign=c]{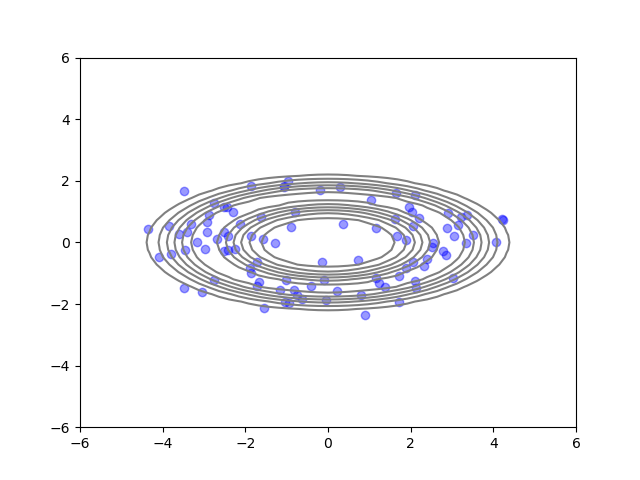}& \adjincludegraphics[width=0.17\textwidth,trim={3.2cm 3.1cm 3.2cm 3.1cm},clip,valign=c]{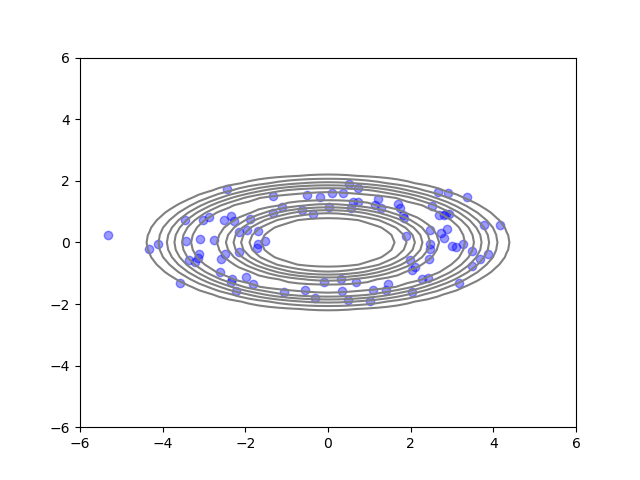}& \adjincludegraphics[width=0.17\textwidth,trim={3.2cm 3.1cm 3.2cm 3.1cm},clip,valign=c]{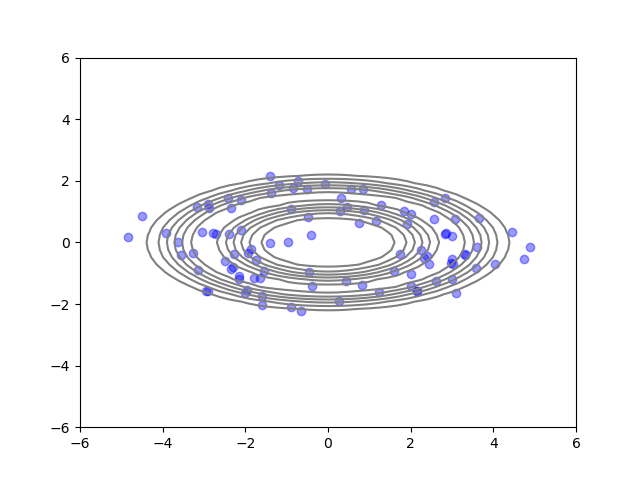}& \adjincludegraphics[width=0.17\textwidth,trim={3.2cm 3.1cm 3.2cm 3.1cm},clip,valign=c]{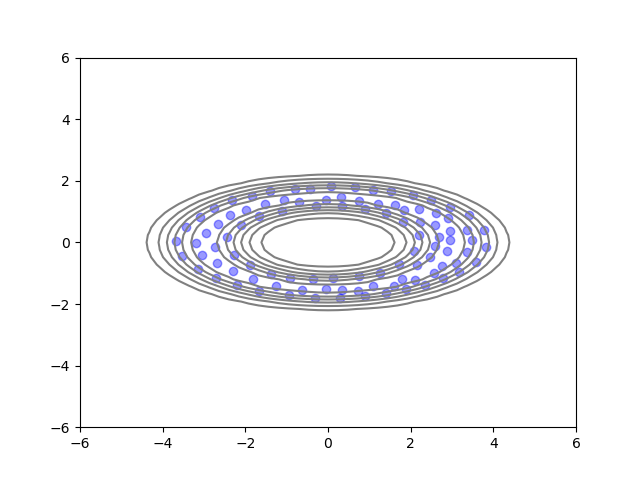}& \adjincludegraphics[width=0.17\textwidth,trim={3.2cm 3.1cm 3.2cm 3.1cm},clip,valign=c]{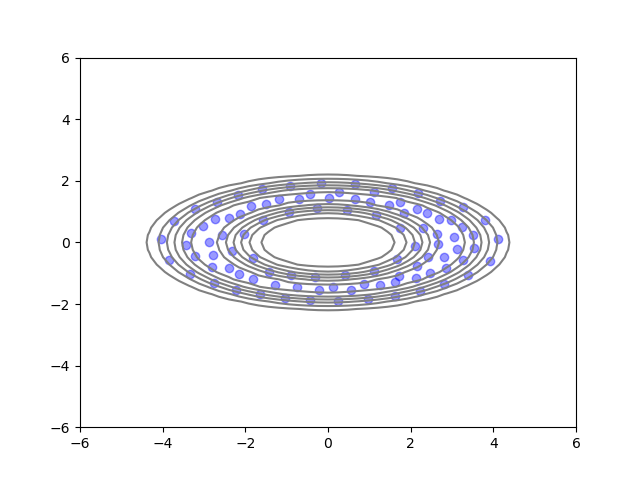}& \adjincludegraphics[width=0.17\textwidth,trim={3.2cm 3.1cm 3.2cm 3.1cm},clip,valign=c]{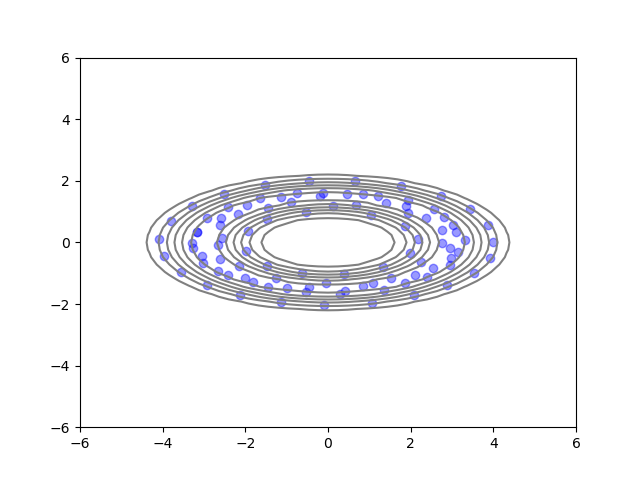}\\
    \end{tblr}}
    \caption{Evolution of the various methods for the scaled annulus. Evaluated with 100 particles, at iterations 10, 50, and 200 from top to bottom respectively. We observe that PBRWP and MLA diffuse faster than their non-preconditioned counterparts. Moreover, PBRWP retains a similar level-set structure to BRWP. Both BRWP and PBRWP spread more than SVGD at convergence.}
    \label{fig:annulus_particles}
\end{figure}
\begin{figure}
    \centering
    \includegraphics[width=0.9\linewidth]{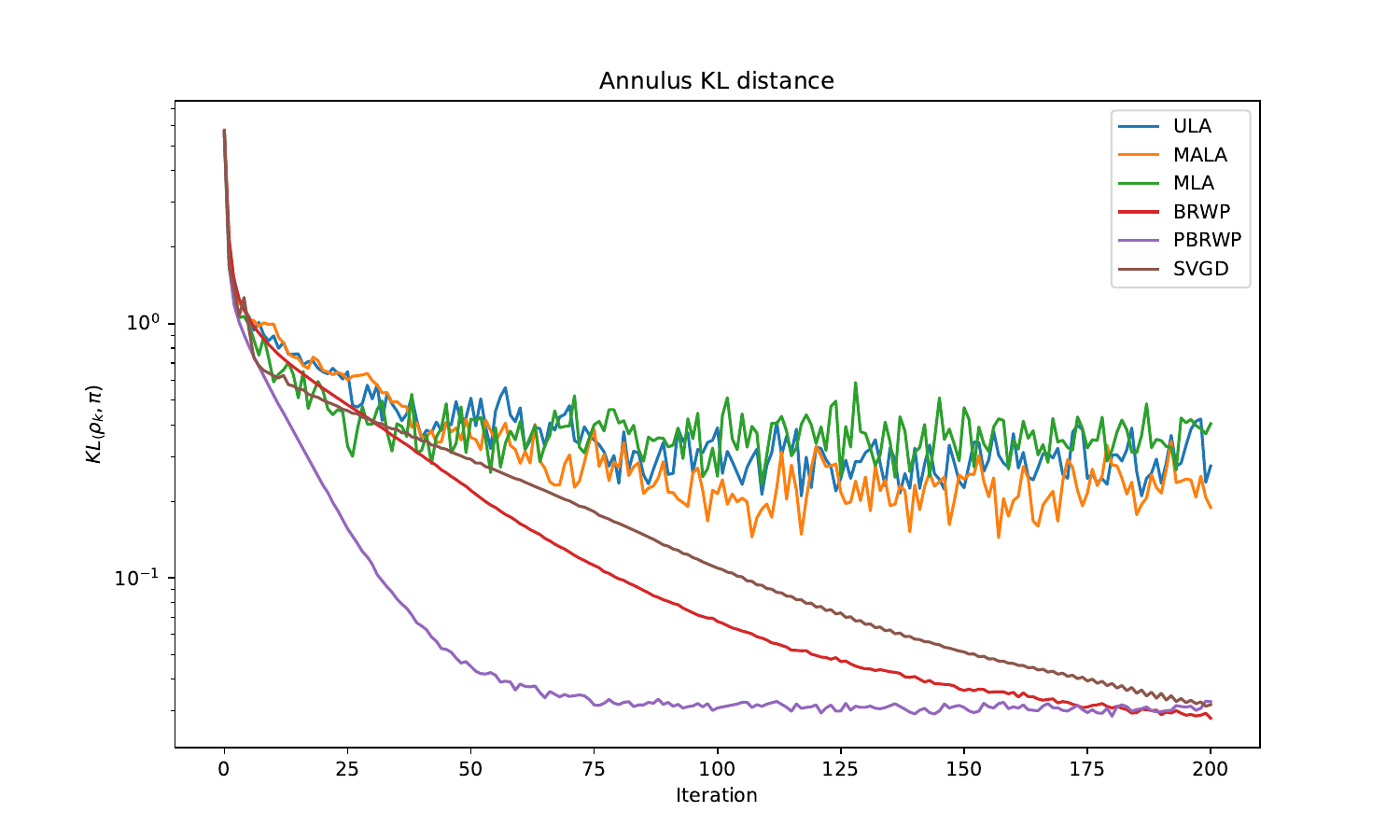}
    \caption{Evolution of the KL divergence between baselines and BRWP-based methods for the scaled annulus. Applied with 100 particles, fixed step-size $\eta =0.1$ and regularization parameter $T=0.05$. We observe that the deterministic SVGD and BRWP-based methods converge more smoothly. }
    \label{fig:KL_annulus}
\end{figure}

\subsection{High dimensional examples and modifications}\label{ssec:highDimExs}
In high dimensions, approximating the normalizing constant with a Monte Carlo integral is highly inaccurate \cite{tan2024noise}. This manifests empirically as non-interaction between particles, or excessive interaction leading to diverging particles. Therefore, we require another approximation using Laplace's method \cite{bleistein1975asymptotic}, similarly to \cite{han2025splitting,han2024convergence}. This approximation has also been used in the other direction, e.g. \cite{tibshirani2025laplace} using integrals to approximate infimal convolutions.

Laplace's method consists of the following approximation: for a $\gC^2$ function $f$ and a continuous function $g$, assume that $f$ has a unique global minimizer $x^*$ and satisfies some coercivity condition. Then, we have the following approximation:

\begin{equation}
    \sqrt{\det(\nabla^2 f(x^*))} \frac{\exp(f(x^*)/T)}{(2 \pi T)^{d/2}} \int_{\R^d} g(x) \exp(-f(x)/T) \dd{x} \rightarrow g(x^*) \quad \text{as } T \rightarrow 0^+.
\end{equation}

Applying this with $f(x) = \beta\|x - \rvx_j\|_M^2/4$ and $g(x) = \exp(-\beta V(x)/2)$, we obtain the approximation
\begin{align}
    \gZ(\rvx_j) &= \int_{\R^d}\exp(-\frac{\beta}{2}\left(V(z) + \frac{\|z - \rvx_j\|_{M}^2}{2T}\right)) \dd{z}\\
    &\approx \sqrt{\det M}\exp(-\frac{\beta}{2}V(\rvx_j))  C(\beta,T),
\end{align}
where $C = C(\beta, T)$ is a constant independent of $\rvx_j$. We note that actually only a H\"older type condition is required for a similar asymptotic to hold, albeit with a different normalizing factor \cite[Thm. 2]{tibshirani2025laplace}. Using this, the $\log \gZ(\rvx_j)$ term may be approximated by $-\beta V(\rvx_j)/2 + \log C(\beta,T) + \frac{1}{2}\log\det M$. For constant $M$, the latter two terms will disappear in the softmax. We refer to this as the Laplace approximation in the following section.

In high dimensions, the matching term $\|\rvx_i - \rvx_j\|_M$ can be large for $i \ne j$, leading to the interaction matrix $\softmax(U_{i,\cdot})_j$ being close to identity. Similarly to the seminal work on transformers \cite{vaswani2017attention}, we additionally propose to set the diffusion parameter as $\beta = d^{-1/2}$, which increases the diffusion. We refer to this as \emph{scaling} in the following section.

The resulting modified iteration using scaling and the Laplace approximation is as follows, where $\beta=d^{-1/2}$,
\begin{subequations}
    \begin{gather}
    \tX^{(k+1)} = \tX^{(k)} - \frac{\eta}{2} M \nabla V(\tX^{(k)}) + \frac{\eta}{2T}\left(\tX^{(k)} - \tX^{(k)} \softmax(W^{(k)})^\top\right),\\
    W_{i,j}^{(k)} = -\beta \frac{\|\rvx_i^{(k)}-\rvx_j^{(k)}\|_M^2}{4 T}+\frac{\beta}{2}V(\rvx_j^{(k)}).
\end{gather}
\end{subequations}

\subsection{High-dimensional Gaussian distributions}
We demonstrate the effect of taking $\beta=d^{-1/2}$ instead of $\beta=1$, as well as the Laplace approximation in \Cref{fig:highdimGaussian}. We test on a 50-dimensional Gaussian with diagonal covariance $\Sigma=\diag(0.1,0.2,...,5)$ and condition number 50. We let $M=\Sigma$ be the preconditioner as before, and plot the particles projected onto the first and last dimensions in the horizontal and vertical directions respectively. Recall that in the Gaussian case, the normalizing constant can be computed exactly in the finite particle setting. We compare PBRWP with step-size $\eta=0.1$, $T=0.02,0.2,0.9$, with and without both the scaling and Laplace approximation after convergence at iteration 1000.

We observe that the Laplace approximation does not seem to qualitatively affect the diffusion behavior of the particles, with or without the $\beta=d^{-1/2}$ scaling modification. However, the scaling is able to counter the mode-collapse phenomenon that comes with observing marginals in high dimensions. By reducing the scale of the interaction matrix, the particles are able to repel each other from further away. Motivated by this, we use the scaling for high dimensions to avoid mode collapse.

\begin{figure}
    \centering
    \renewcommand{\arraystretch}{0}
    \noindent\makebox[\textwidth]{
    \begin{tblr}{
        colspec={cccc},
        }
    $\beta=1$ & $\beta=d^{-1/2}$ & $\beta=1$, Laplace & $\beta=d^{-1/2}$, Laplace \\\hline
     \adjincludegraphics[width=0.19\textwidth,trim={5cm 2.6cm 5cm 2.6cm},clip,valign=c,angle=90,origin=c]{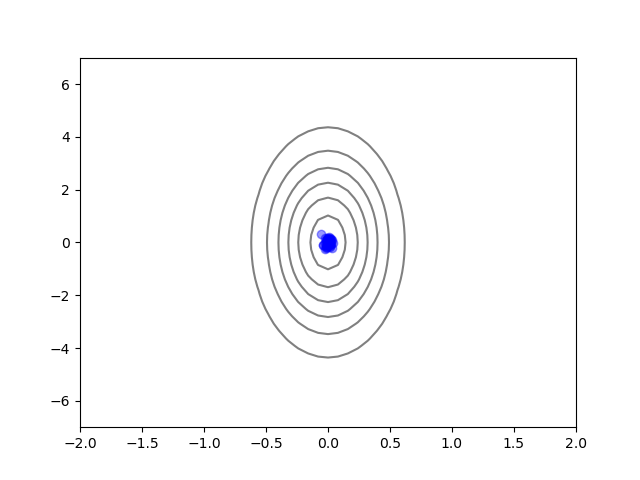}
     & \adjincludegraphics[width=0.19\textwidth,trim={5cm 2.6cm 5cm 2.6cm},clip,valign=c,angle=90,origin=c]{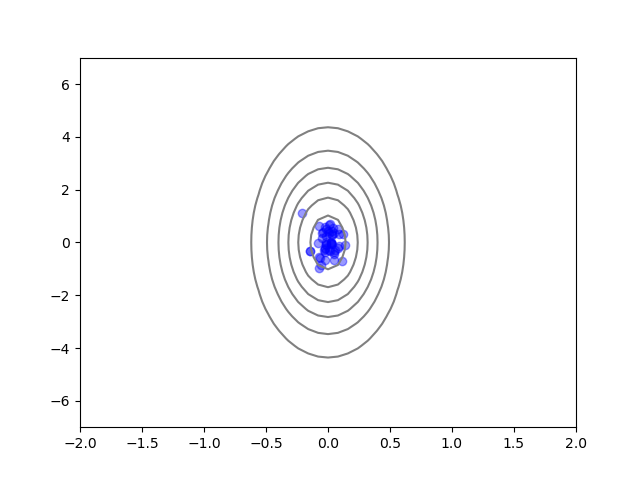}
     & \adjincludegraphics[width=0.19\textwidth,trim={5cm 2.6cm 5cm 2.6cm},clip,valign=c,angle=90,origin=c]{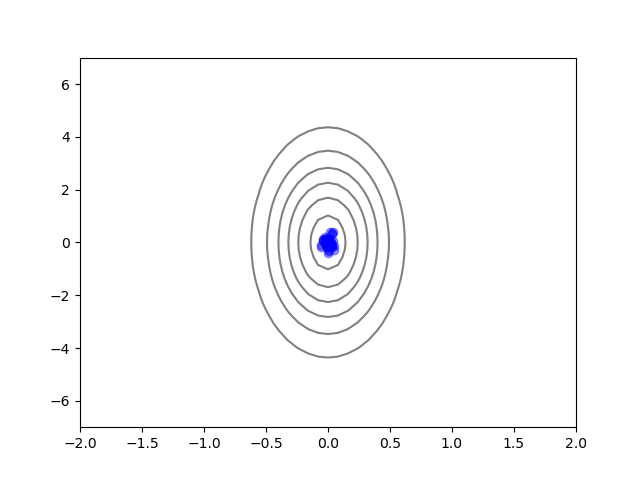}
     & \adjincludegraphics[width=0.19\textwidth,trim={5cm 2.6cm 5cm 2.6cm},clip,valign=c,angle=90,origin=c]{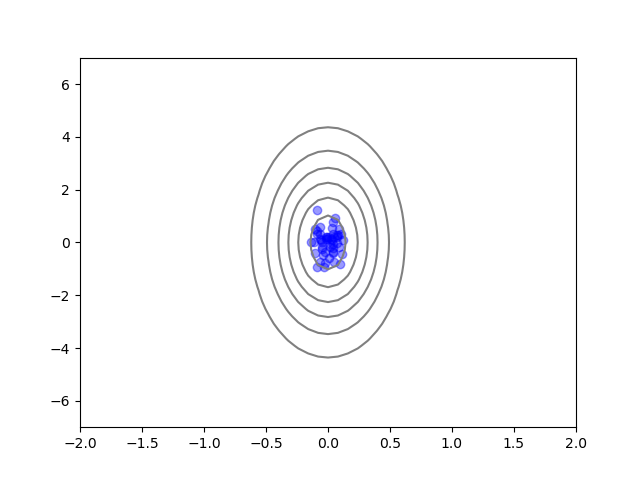}\\
     \adjincludegraphics[width=0.19\textwidth,trim={5cm 2.6cm 5cm 2.6cm},clip,valign=c,angle=90,origin=c]{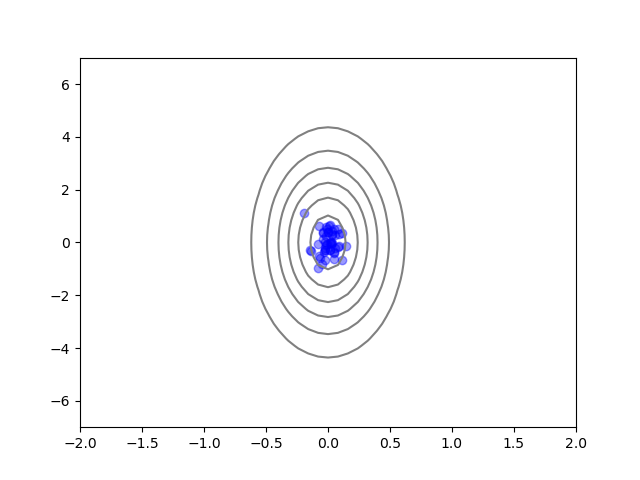}
     & \adjincludegraphics[width=0.19\textwidth,trim={5cm 2.6cm 5cm 2.6cm},clip,valign=c,angle=90,origin=c]{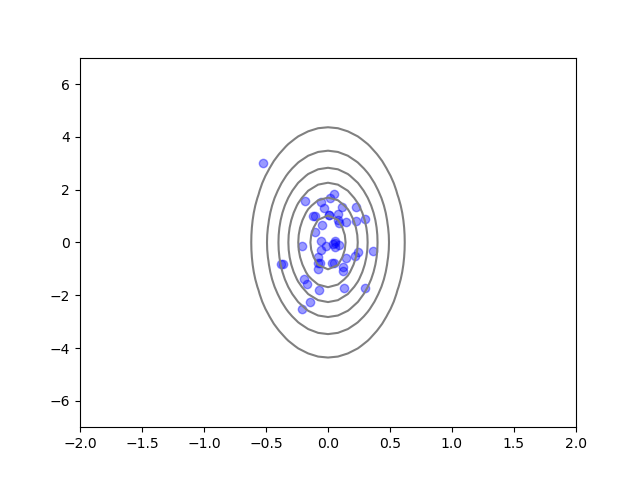}
     & \adjincludegraphics[width=0.19\textwidth,trim={5cm 2.6cm 5cm 2.6cm},clip,valign=c,angle=90,origin=c]{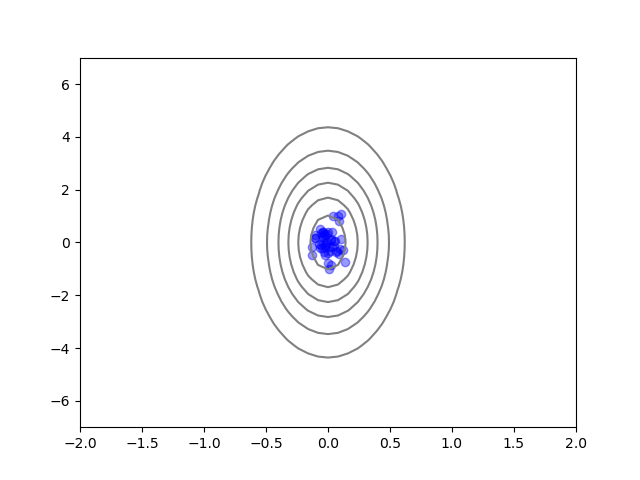}
     & \adjincludegraphics[width=0.19\textwidth,trim={5cm 2.6cm 5cm 2.6cm},clip,valign=c,angle=90,origin=c]{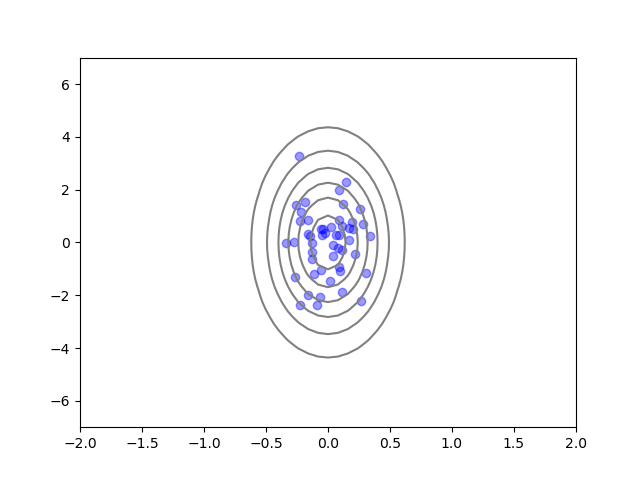}\\
     \adjincludegraphics[width=0.19\textwidth,trim={5cm 2.6cm 5cm 2.6cm},clip,valign=c,angle=90,origin=c]{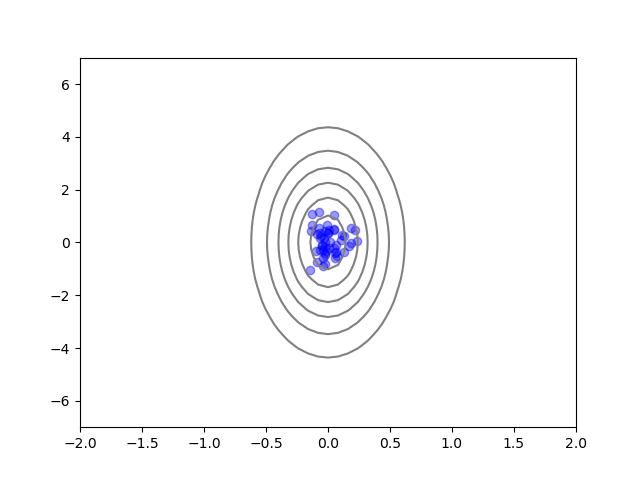}
     & \adjincludegraphics[width=0.19\textwidth,trim={5cm 2.6cm 5cm 2.6cm},clip,valign=c,angle=90,origin=c]{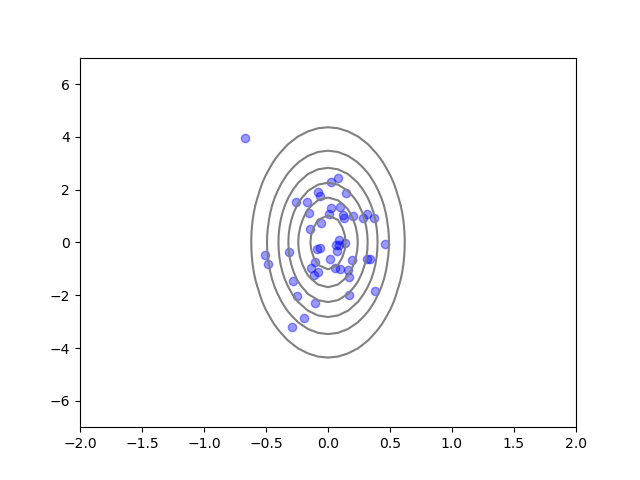}
     & \adjincludegraphics[width=0.19\textwidth,trim={5cm 2.6cm 5cm 2.6cm},clip,valign=c,angle=90,origin=c]{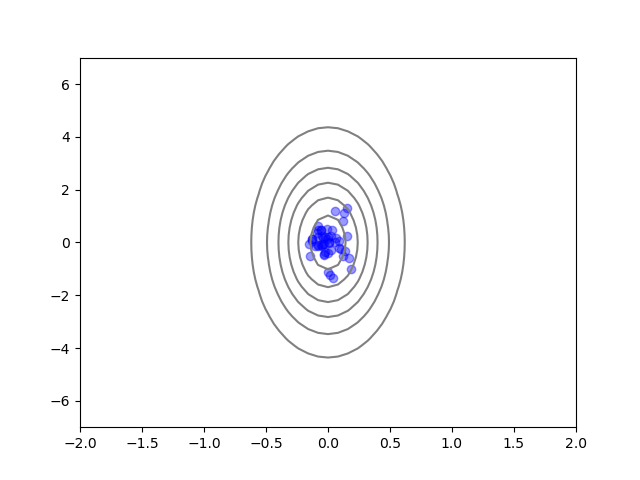}
     & \adjincludegraphics[width=0.19\textwidth,trim={5cm 2.6cm 5cm 2.6cm},clip,valign=c,angle=90,origin=c]{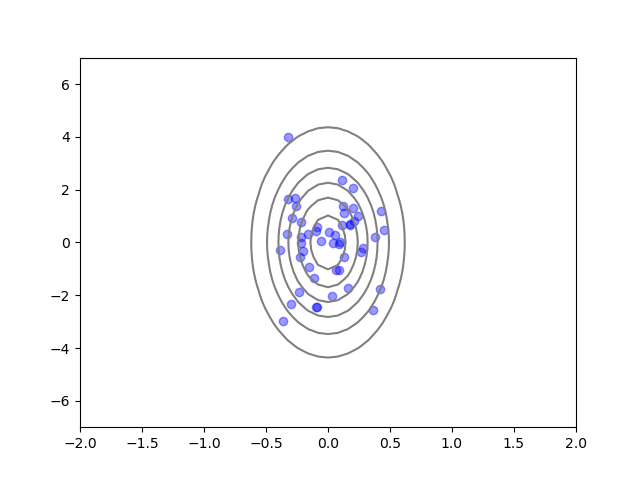}
    \end{tblr}}
    \caption{Evolution of the high-dimensional modifications for the 50-dimensional Gaussian, at convergence in 1000 iterations. Evaluated with 50 particles, step-size $\eta=0.1$ and regularizations $T=0.02,0.2,0.9$ in the top, middle and bottom rows respectively. We observe little difference when using the Laplace approximation compared to the ground-truth, suggesting this is reasonable. The $\beta=d^{-1/2}$ scaling increases the diffusion, which is crucial for reducing mode-collapse in high dimensions.}
    \label{fig:highdimGaussian}
\end{figure}

\subsection{Deconvolution}
We consider the Bayesian problem corresponding to the (convex) total-variation regularized (TV) objective 
\cite{rudin1992nonlinear},
\begin{equation}\label{eq:BIPPotential}
    V(x) = \frac{1}{2\sigma^2} \|Ax-y\|_2^2 + \lambda \mathrm{TV}(x),
\end{equation}
where $A$ is a convolution operator, $y$ is a corrupted image, and $\mathrm{TV}(x) = \|D x\|_1$ denotes the discrete total variation functional. For image deconvolution, the forward operator $A$ takes the form $A = \gF^* \Lambda \gF$, where $\gF$ is the (complex, unitary) matrix of the discrete Fourier transform, $\gF^*$ is its inverse, and $\Lambda$ is a diagonal matrix. In this case, the Hessian of the first term takes the simple form $\sigma^{-2}A^* A = \sigma^{-2}\gF^* \Lambda^* \Lambda \gF$. Similarly to SALSA \cite{afonso2010fast}, we may use a regularized version as a preconditioner $M = (A^* A + \tau I)^{-1}$, where $\tau>0$ is some constant, taken to be $\tau=0.5$. We note that sampling from the distribution $\gN(0, M)$ is possible using the equation $\sqrt{M} = \gF^* (\Lambda^* \Lambda + \tau I)^{-1/2}\gF$, which allows for the use of MLA. We further note that using the regularized Hessian $(\sigma^2A^* A + \tau I)^{-1}$ does not provide adequate preconditioning for PBRWP or MLA, likely due to the ill-conditioning of the TV term with respect to this metric.

\begin{figure}
    \centering
    \setlength{\tabcolsep}{1pt}
    \renewcommand{\arraystretch}{0}
    \noindent\makebox[\textwidth]{
    \begin{tblr}{
        colspec={cccccc},
        }
      ULA & MYULA & SVGD & BRWP & MLA & PBRWP  \\\hline
     \adjincludegraphics[height=2cm,trim={65px 35px 51px 38px},clip,valign=c]{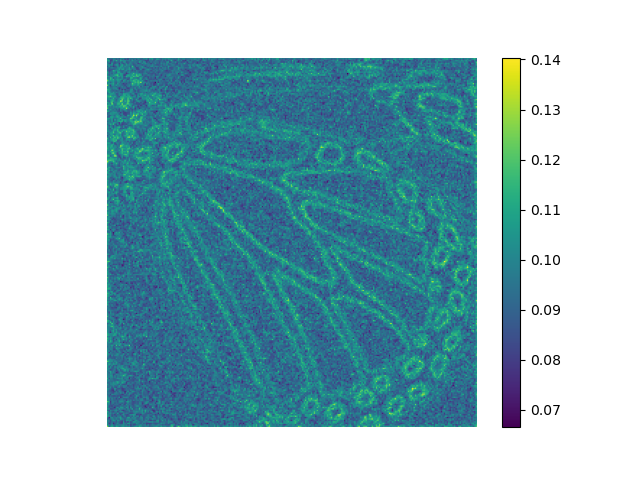}
     & \adjincludegraphics[height=2cm,trim={65px 35px 51px 38px},clip,valign=c]{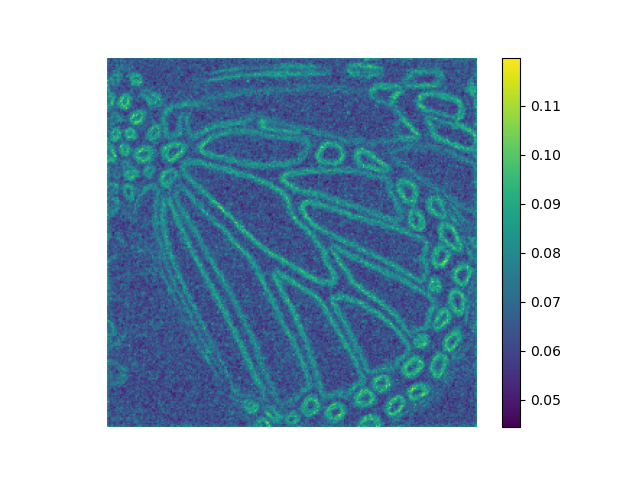}
     & \adjincludegraphics[height=2cm,trim={65px 35px 51px 38px},clip,valign=c]{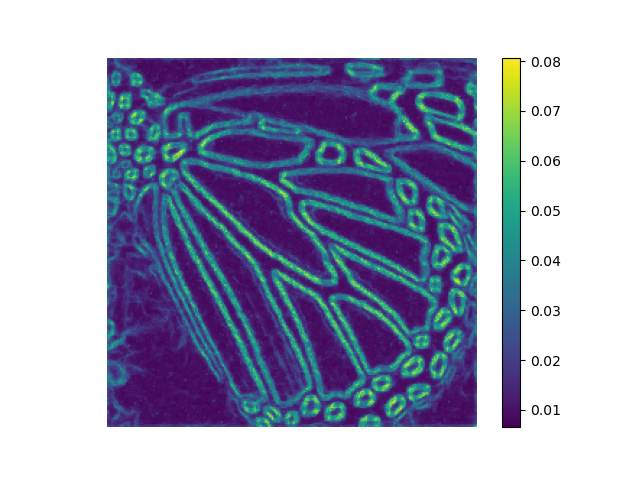}
     & \adjincludegraphics[height=2cm,trim={65px 35px 51px 38px},clip,valign=c]{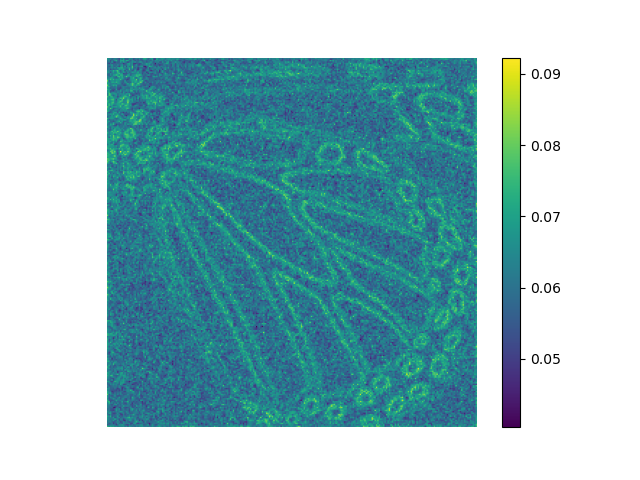}
     & \adjincludegraphics[height=2cm,trim={65px 35px 51px 38px},clip,valign=c]{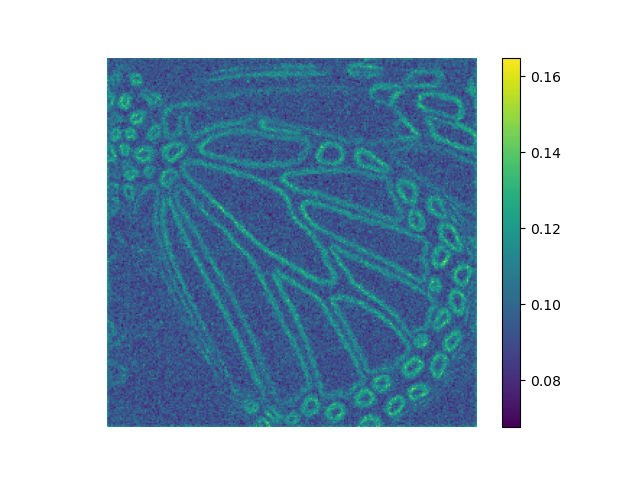}
     & \adjincludegraphics[height=2cm,trim={65px 35px 51px 38px},clip,valign=c]{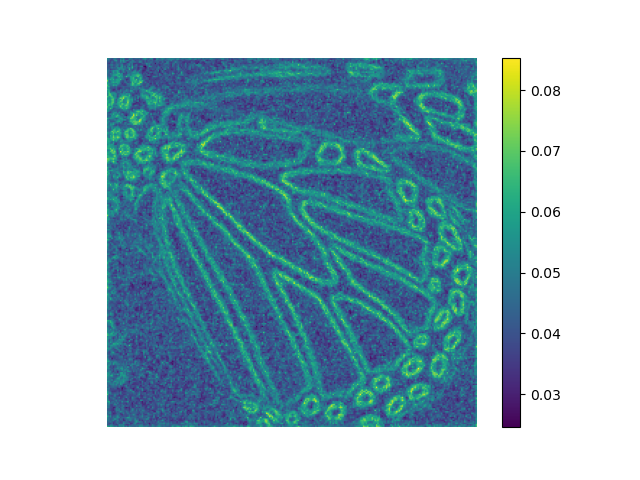}\\
     \adjincludegraphics[height=2cm,trim={65px 35px 51px 38px},clip,valign=c]{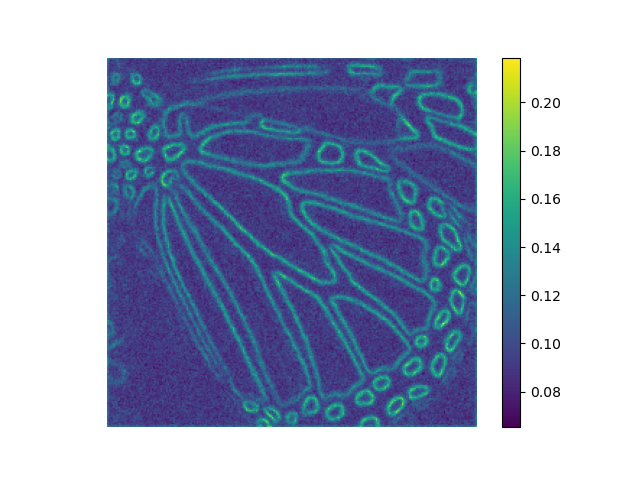}
     & \adjincludegraphics[height=2cm,trim={65px 35px 51px 38px},clip,valign=c]{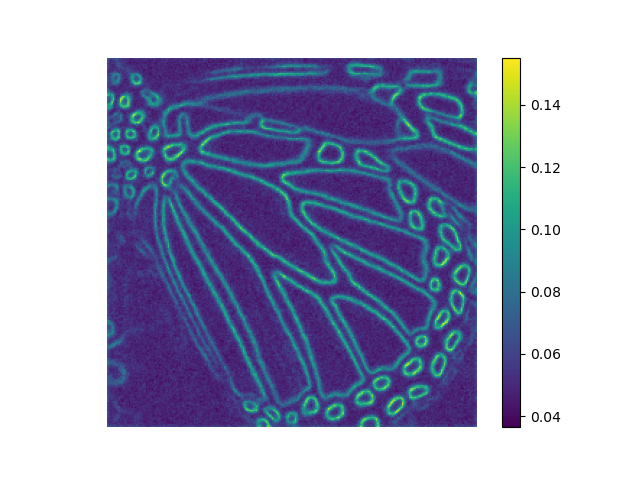}
     & \adjincludegraphics[height=2cm,trim={65px 35px 51px 38px},clip,valign=c]{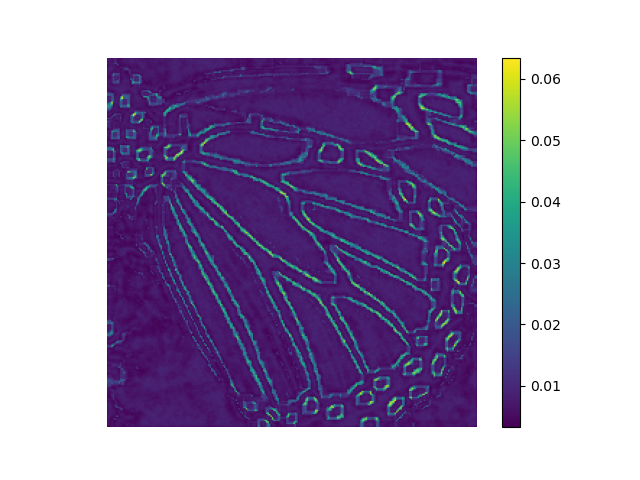}
     & \adjincludegraphics[height=2cm,trim={65px 35px 51px 38px},clip,valign=c]{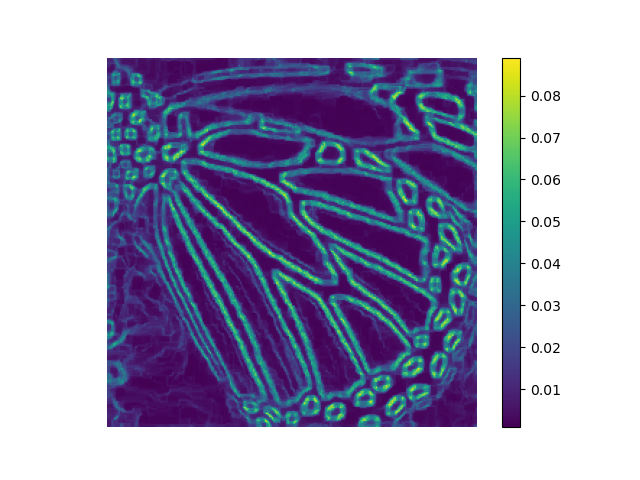}
     & \adjincludegraphics[height=2cm,trim={65px 35px 51px 38px},clip,valign=c]{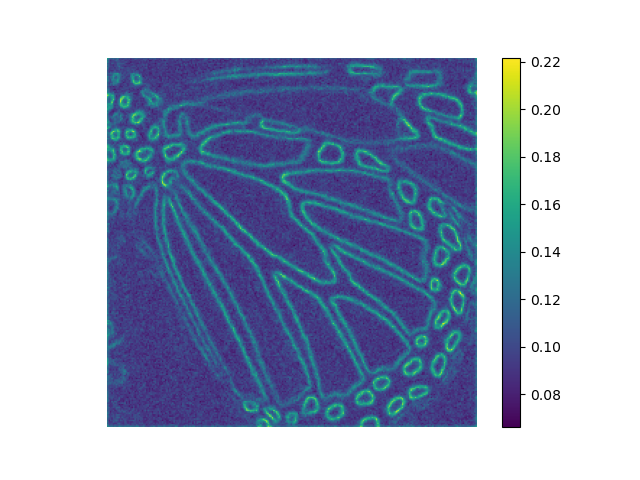}
     & \adjincludegraphics[height=2cm,trim={65px 35px 51px 38px},clip,valign=c]{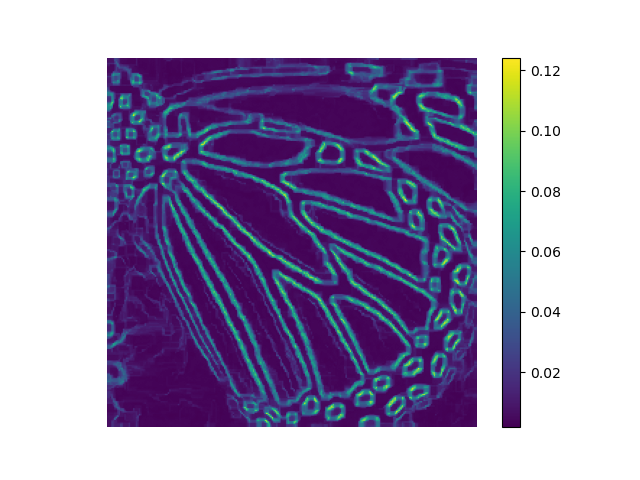}\\
     \adjincludegraphics[height=2cm,trim={65px 35px 51px 38px},clip,valign=c]{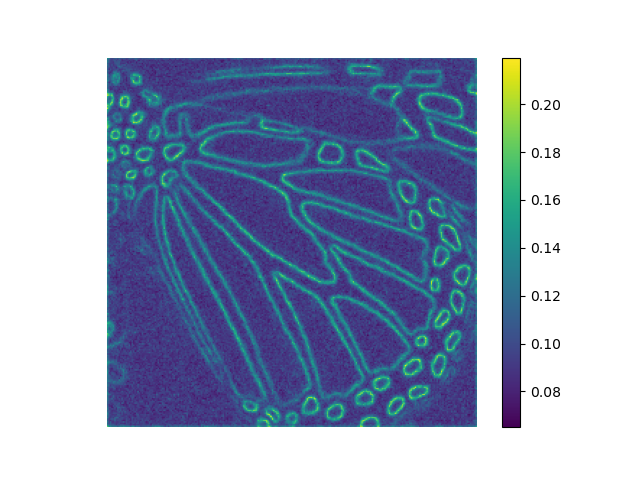}
     &\adjincludegraphics[height=2cm,trim={65px 35px 51px 38px},clip,valign=c]{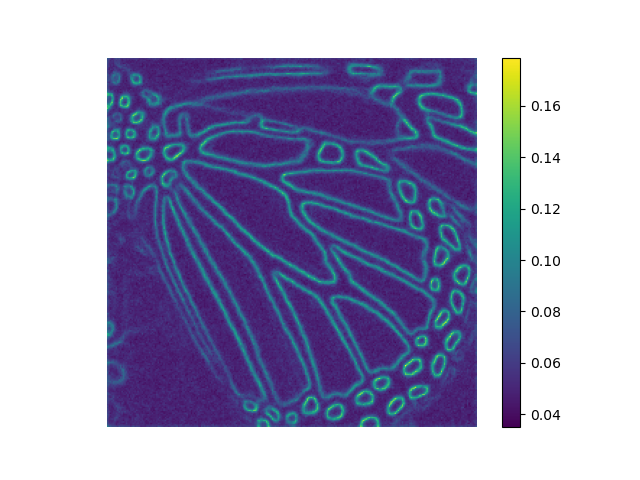}
     & \adjincludegraphics[height=2cm,trim={65px 35px 51px 38px},clip,valign=c]{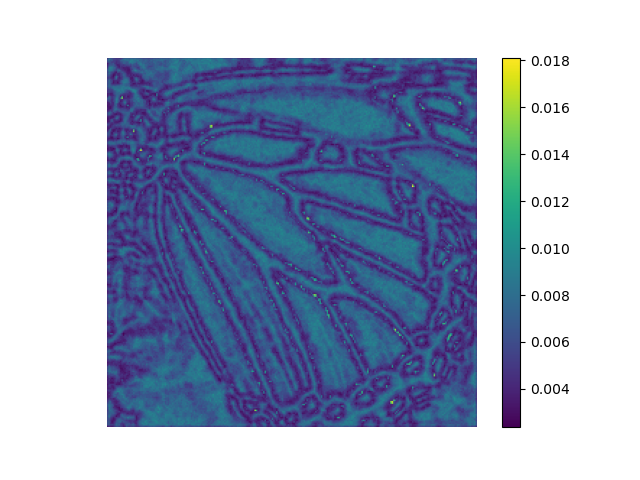}
     &\adjincludegraphics[height=2cm,trim={65px 35px 51px 38px},clip,valign=c]{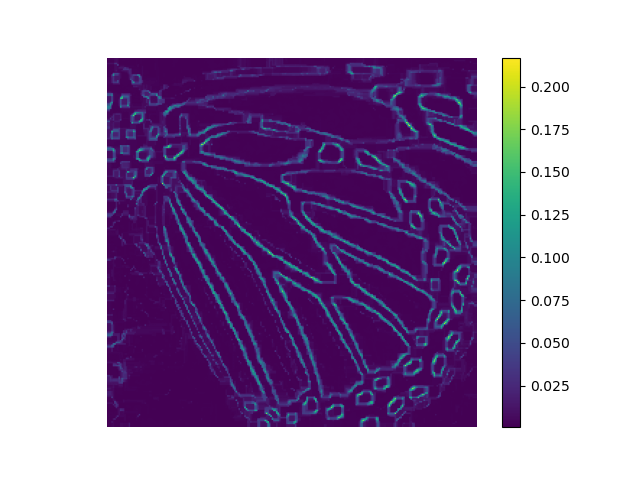}
     &\adjincludegraphics[height=2cm,trim={65px 35px 51px 38px},clip,valign=c]{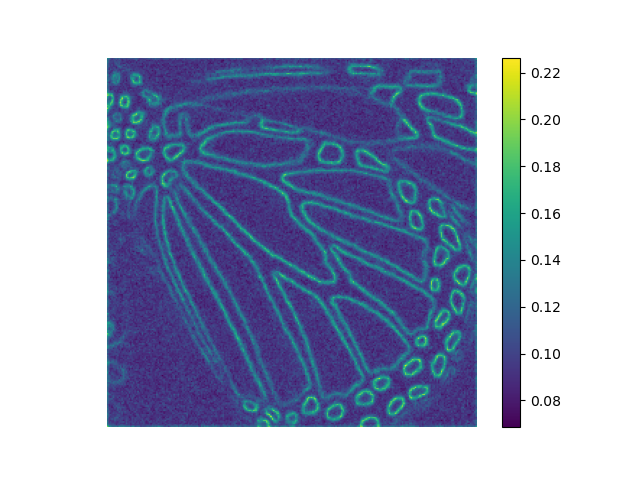}
     &\adjincludegraphics[height=2cm,trim={65px 35px 51px 38px},clip,valign=c]{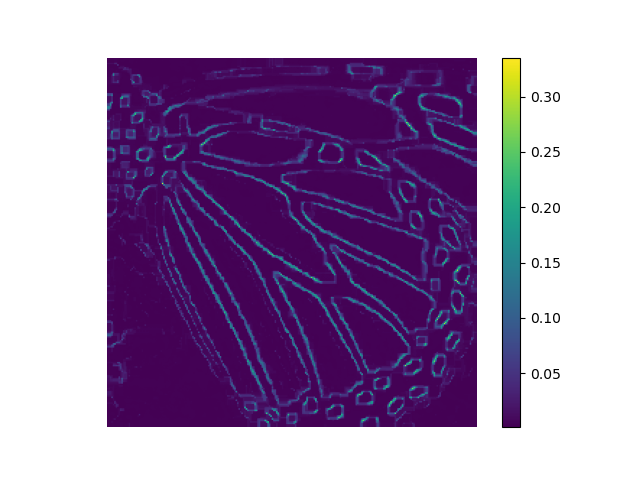}\\
    \end{tblr}}
    \caption{Standard deviations for TV-regularized deconvolution. Run with 40 particles, and evaluated at iterations 20, 200 and 2000 in the top, middle and bottom rows respectively. We observe that the pixelwise variance of the noise-free methods BRWP and PBRWP are lower than their Langevin counterparts. In the 200 iteration regime, we see that the features of the noise-free BRWP and PBRWP methods have more contrast and lower variances in the piecewise constant areas. SVGD tends to collapse, leading to small variance throughout.}
    \label{fig:deconv}
\end{figure}

We compare the mean and (norm of the channel-wise) standard deviation of the TV-regularized objective. In addition, we compute the PSNR of the finite-particle posterior mean to the true TV-regularized solution, equivalent to running a fixed number of parallel chains for the Langevin methods. We do this to compare the small-particle behavior of the noise-free methods versus the Langevin methods. The minimizer of the objective \cref{eq:BIPPotential} is computed using subgradient descent for 30,000 iterations, applied with decreasing step-sizes {$1\mathrm{e}-\{4,5,6\}$} for 10,000 iterations each. The image considered is the butterfly image from the set3c dataset, which has a dimension of $3\times 256 \times 256$. 

We compare against ULA, MLA, MYULA \cite{durmus2018efficient}, and SVGD. We note that the acceptance probability for MALA is almost degenerate for this high-dimensional problem and requires extremely small step-sizes, and is thus omitted due to slow convergence. MYULA takes the following form: for a composite potential of the form $\exp(-f(x) + g(x))$, a step-size $\tau>0$ and a regularization parameter $\theta>0$,
\begin{equation}\tag{MYULA}
    X_{k+1} = \left(1-\frac{\tau}{\theta}\right)X_k - \tau \nabla f(X_k) + \frac{\tau}{\theta} \prox_{\theta g}(X_k) + \sqrt{2\tau} Z_{k+1},
\end{equation}
where $Z_{k+1}$ are independent standard Gaussians. This is the unadjusted Langevin algorithm, applied to the regularized potential where $g$ is replaced by its Moreau--Yosida regularization with parameter $\theta$. Applied to the deconvolution potential \cref{eq:BIPPotential}, we identify $f(x)=\frac{1}{2\sigma^2} \|Ax-y\|^2$ and $g(x) = \lambda \TV(x)$, and take the regularization parameter to be $\lambda = 10$. The proximal of $g$ is computed using the Chambolle--Pock algorithm \cite{condat2013primal} using the implementation in the DeepInverse Python library \cite{tachella2025deepinverse}, and the parameter is fixed as $\theta = 1\mathrm{e}-4$.

The step-sizes for the baseline methods are chosen with a grid-search such that the PSNR is maximized after 2000 iterations, and can be found in the appendix. We compare with 40 particles due to GPU memory limitations from the $\mathcal{O}(N^2)$ scaling of fully parallelized PBRWP\footnote{It is possible to remove some parallelization by computing the interaction matrix column-wise, yielding $\mathcal{O}(N)$ memory constraint.}. 

The experimental setting differs from typical settings as follows, to highlight the difference between Langevin methods and the noise-free methods BRWP and PBRWP:
\begin{enumerate}
    \item \textit{Low sample count.} Langevin algorithms usually run many iterations (around $1\mathrm{e}5$) of the Markov chain to construct an estimate of the posterior mean. In our case, since the samples are approximately convergent for high iterations, taking means over iterations is not insightful, and we replace it with the mean over particles. The particles for Langevin methods can be computed using parallel chains.
    \item \textit{Short ``burn-in'' phase.} While typical Bayesian imaging methods run burn-in periods for up to $1\mathrm{e}6$ iterations, the noise-free methods again do not benefit from these long iterations. As our objective is rapid convergence, we consider only methods that are able to converge in the short iteration regime.
    \item \textit{Comparison with MAP estimate.} In high dimensions, computing the posterior mean is intractable, and is usually estimated. While it is possible to consider the convergence of the algorithm to the posterior mean by applying a Metropolis--Hastings correction step \cite{habring2024subgradient}, this would compare the algorithms against an estimator with a-priori unknown uncertainty. We choose the MAP estimate instead as it is computable up to arbitrary accuracy for this convex task, and to compare the mode concentration of the algorithms.
\end{enumerate}

\begin{table}[]
\centering
\caption{PSNR (in dB) of the mean of 40 particles compared to the MAP estimate for the deconvolution problem. We observe that the preconditioned methods converge to the neighborhood of the MAP faster than their non-preconditioned counterparts, indicated by the values at 20 and 200 iterations. Moreover, the noise-free methods are much more tightly concentrated around the MAP estimate, including SVGD which tends to collapse in high dimensions. However, the particle means of the preconditioned methods are slightly further away from the posterior mean at high iterations. }
\label{tab:psnr_deconvBIP}
\begin{tabular}{@{}ccccccc@{}}
\toprule
Iteration & ULA   & MYULA & SVGD & BRWP & MLA   & PBRWP \\ \midrule
20        & 23.28 & 23.53 & 24.81 & 23.52   & 23.86 & 24.27     \\
200       & 25.43 & 25.79 & 31.28 & 27.81   & 25.78 & 29.90     \\
2000      & 26.07 & 26.38 & 38.34 & 38.38   & 26.00 & 37.35     \\ \bottomrule
\end{tabular}
\end{table}

\Cref{tab:psnr_deconvBIP} plots the peak signal-to-noise ratio (PSNR) of the particle means, compared to the MAP estimate \cref{eq:BIPPotential} (equivalently, the minimizer of the potential). We observe that both BRWP and PBRWP converge to much higher values of PSNR, indicating better concentration around the MAP. Moreover, PBRWP has higher PSNR at lower iterations, indicating faster convergence. The lower PSNR of the preconditioned methods compared to the non-preconditioned methods can be explained by the differing noise and diffusion scales. \Cref{fig:deconv} plots the pixel-wise standard deviation of the particles at various iteration counts. We observe again that PBRWP has higher contrast compared to BRWP at low iterations, indicating faster convergence. Moreover, the noise-free methods have very low noise outside the edges, with SVGD having the reverse phenomenon of collapsing on the edges. This low-rank phenomenon arises since diffusion comes from inter-particle interactions.

\subsection{Bayesian neural networks}\label{ssec:BNN}
In this section, we empirically investigate the possibility of a variable preconditioning matrix $M$, applied to a high-dimensional non-convex potential. Recall the PBRWP iteration \cref{eq:tensorForm} for particles $\{\rvx_i\}_{i=1}^N$:
\begin{gather*}
    \tX^{(k+1)} = \tX^{(k)} - \frac{\eta}{2} M \nabla V(\tX^{(k)}) + \frac{\eta}{2T}\left(\tX^{(k)} - \tX^{(k)} \softmax(W^{(k)})^\top\right),\\
    W_{i,j}^{(k)} = -\beta \frac{\|\rvx_i^{(k)}-\rvx_j^{(k)}\|_M^2}{4 T} - \log \gZ(\rvx_j^{(k)}).
\end{gather*}
In the normalizing term, we empirically propose to replace the term $\|\rvx_i^{(k)}-\rvx_j^{(k)}\|_M^2$ with a $j$- and $k$-dependent scaling $\|\rvx_i^{(k)}-\rvx_j^{(k)}\|_{M_j^{(k)}}^2$. Using Laplace's approximation, we have 
\begin{align}
    \gZ(\rvx_j; M_j) &= \int_{\R^d} \exp({-\frac{\beta}{2}(V(z) + \frac{\|z - \rvx_j\|_{M_j}^2}{2T})} )\dd{z}\\
    &\approx C(\beta, T) \sqrt{\det M_j} \exp(-\frac{\beta}{2} V(\rvx_j)).
\end{align}
Simplifying, the variable-preconditioner version of PBRWP takes the following form:
\begin{gather}
    \tX^{(k+1)} = \tX^{(k)} - \frac{\eta}{2} \tM^{(k)} \nabla V(\tX^{(k)}) + \frac{\eta}{2T}\left(\tX^{(k)} - \tX^{(k)} \softmax(W^{(k)})^\top\right),\\
    W_{i,j}^{(k)} = -\beta\frac{\|\rvx_i^{(k)}-\rvx_j^{(k)}\|_{M_j^{(k)}}^2}{4 T} - \frac{1}{2}\log \det(M_j^{(k)}) +\frac{\beta}{2} V(\rvx_j), \label{eq:variableW}
\end{gather}
where $\tM^{(k)} \nabla V(\tX^{(k)}) \coloneqq \begin{bmatrix}M^{(k)}_1 \nabla V(\rvx_1) & ...  & M_N^{(k)} \nabla V(\rvx_N)\end{bmatrix}$.

For a neural network, the choice of preconditioner is a-priori unclear. For our experiments, we choose the diagonal preconditioner corresponding to the Adam optimizer \cite{kingma2014adam}, with exact implementation given in \ref{appsec:AdamPC}. This has recently been linked to a diagonal approximation to the empirical Fisher information matrix  \cite{kunstner2019limitations,hwang2024fadam}. The diagonal form of the preconditioning matrix in this case makes it particularly amenable for the log-determinant in \cref{eq:variableW}.

To test, we train neural networks on various regression datasets, with $N=10$ particles as in \cite{tan2024noise,wang2022accelerated}. The potential $V$ is given by the mean-squared error on the training dataset. \Cref{tab:RMSE} shows the root-mean-square error of the trained networks using PBRWP against various baselines. For the Adam baseline \cite{kingma2014adam}, we directly train from the initializations, where the step-size is chosen with a grid search to minimize the test RMSE. We observe that PBRWP is able to uniformly outperform BRWP on such a task, and is more competitive with kernel-based methods. This is analogous to Adam being empirically better than SGD for training neural networks.

\begin{table}[]
\centering
\caption{Test root-mean-square-error (RMSE) on test datasets on various Bayesian neural network tasks. Bold indicates smallest in row.  We observe that the adaptive Fisher preconditioned BRWP uniformly outperforms BRWP on each of the BNN tasks. Adam and the noise-free methods both generally exhibit high variance in this setting, which may be due to the relatively small neural network architecture and sensitivity to initialization. We may further interpret the high variance of these methods as being able to find better trained models.}
\label{tab:RMSE}
\begin{adjustbox}{width=1\textwidth}
\begin{tabular}{@{}cr|rrrrr@{}}
\toprule
Dataset & \multicolumn{1}{c}{Adam}&\multicolumn{1}{c}{PBRWP} & \multicolumn{1}{c}{BRWP} & \multicolumn{1}{c}{AIG} & \multicolumn{1}{c}{WGF} & \multicolumn{1}{c}{SVGD} \\ \midrule
Boston   & $3.350_{\pm8.33\mathrm{e}-1}$& $2.866_{\pm 5.94\mathrm{e}-1}$ & $3.309_{\pm 5.31\mathrm{e}-1}$ & $2.871_{\pm 3.41\mathrm{e}-3}$ & $3.077_{\pm 5.52\mathrm{e}-3}$ & $\pmb{2.775_{\pm 3.78\mathrm{e}-3}}$ \\
Combined & $3.971_{\pm1.79\mathrm{e}-1}$&$\pmb{3.925_{\pm 1.52\mathrm{e}-1}}$& ${3.975_{\pm 3.94\mathrm{e}-2}}$ & $4.067_{\pm 9.27\mathrm{e}-1}$ & $4.077_{\pm 3.85\mathrm{e}-4}$ & ${4.070_{\pm 2.02\mathrm{e}-4}}$ \\
Concrete & $4.698_{\pm4.85\mathrm{e}-1}$&$\pmb{4.387_{\pm 4.88\mathrm{e}-1}}$ & $4.478_{\pm 2.05\mathrm{e}-1}$ & ${4.440_{\pm 1.34\mathrm{e}-1}}$ & $4.883_{\pm 1.93\mathrm{e}-1}$ & $4.888_{\pm 1.39\mathrm{e}-1}$ \\
Kin8nm   & $0.089_{\pm 2.72\mathrm{e}-3}$ & $\pmb{0.087_{\pm2.67\mathrm{e}-3}}$ & ${0.089_{\pm 6.06\mathrm{e}-6}}$ & $0.094_{\pm 5.56\mathrm{e}-6}$ & $0.096_{\pm 3.36\mathrm{e}-5}$ & $0.095_{\pm 1.32\mathrm{e}-5}$ \\
Wine     & $0.629_{\pm4.01\mathrm{e}-2}$ & $0.612_{\pm 4.17\mathrm{e}-2}$ & $0.623_{\pm 1.35\mathrm{e}-3}$ & $0.606_{\pm 1.40\mathrm{e}-5}$ & $0.614_{\pm 3.48\mathrm{e}-4}$ & $\pmb{0.604_{\pm 9.89\mathrm{e}-5}}$ \\ \bottomrule
\end{tabular}
\end{adjustbox}
\end{table}

\section{Conclusion}
This work proposes the Preconditioned Backwards Regularized Wasserstein Proximal (PBRWP) method \Cref{alg:PBRWP}, a method of sampling without Langevin noise. Deriving the Cole--Hopf transform from a pair of anisotropic forward-backward heat equations, we obtain a coupled forward-time Fokker--Planck equation and backward-time Hamilton--Jacobi equation, both regularized with (elliptic) second-order differential operators. By applying a similar semi-implicit discretization to \cite{tan2024noise}, we obtain particle swarm dynamics, interpretable as a self-attention mechanism. A discrete-time non-asymptotic convergence analysis is given in the case of quadratic potentials. Experiments demonstrate faster convergence with slightly more empirical bias, consistent with preconditioned Langevin methods.

While a fixed preconditioner $M$ has be shown to provide acceleration for some simple variational problems, the experiments in \Cref{ssec:BNN} empirically demonstrate that it is also possible to have variable preconditioners, such as that approximating the empirical Fisher information matrix. Future work would consider rigorous justification of variable preconditioners $M = M(x)$, such as those arising from inhomogeneous heat equations, or from Bregman divergences, similarly to mirror descent methods. Given that the PBRWP may be considered as a self-attention framework, future work may consider the converse implication. Additional theoretical analysis is also important. One part is the continuous-time convergence of the modified Fokker--Planck type equation \cref{eq:ModFP}, which is not amenable to standard energy methods due to the variable kernel. The other is the discretization error from working with finitely many particles.  While this arises from central limit and ergodicity results for standard Langevin methods, the final iteration approximation to the desired measure is less clear in the noiseless setting due to the per-particle convergence. 

\subsubsection*{Acknowledgements}
H. Y. Tan is partially supported by NSF 443948-SN-2199, No. FA9550-23-1-0087, GSK.ai and the Masason Foundation. S. Osher is partially supported by NSF 443948-SN-2199. W. Li is partially supported by AFSOR YIP award No. FA9550-23-1-0087, NSF DMS-2245097, and NSF RTG:2038080.
\bibliographystyle{siamplain}
\bibliography{refs}

\newpage
\appendix
\section{Proofs}
\subsection{Anisotropic Green's function}\label{app:greens}
Recall (\ref{eq:aniKernel}):
\begin{equation}
    G_{t,M}(x,y) \coloneqq \frac{1}{(4\pi \beta^{-1} t)^{d/2} |M|^{1/2}} e^{-\beta\frac{(x-y)^\top M^{-1}(x-y)}{4t}}.
\end{equation}
\begin{proposition}
    The anisotropic kernel $G_{t,M}(x,y)$ is a Green's function for the following PDE:
    \begin{equation*}
    \begin{cases}
        \partial_t u - \beta^{-1} \nabla \cdot(M \nabla u) = 0,\\
        u(0,x) = \delta(y).
    \end{cases}
    \end{equation*}
\end{proposition} 
\begin{proof}
    \begin{align*}
        \partial_t G_{t,M}(x,y) = \left(-\frac{d}{2t} + \beta\frac{(x-y)^\top M^{-1} (x-y)}{4 t^2}\right)G_{t,M}(x,y)
    \end{align*}
    \begin{align*}
        \nabla_x G_{t,M}(x,y) = -\frac{\beta M^{-1}(x-y)}{2t}G_{t,M}(x,y)
    \end{align*}
    \begin{align*}
        \nabla_x\cdot (M G_{t,m}(x,y)) &= \nabla_x \cdot\left(-\frac{\beta(x-y)}{2 t}G_{t,M}(x,y)\right) \\
        &= \left(-\frac{d\beta}{2 t} + \beta^2\frac{(x-y)^\top M^{-1} (x-y)}{4t^2}\right) G_{t,M}(x,y)
    \end{align*}
    Therefore $\partial_t G_{t,M} = \beta^{-1} \nabla \cdot (M \nabla G_{t,m})$. The boundary condition follows from a change of variables with the standard heat kernel.
\end{proof}

\subsection{Mean field control problem}\label{appsec:mfc}
Recall the coupled PDEs (\ref{eqs:PBRWPPDE}) that are used to define the preconditioned regularized Wasserstein proximal operator:
\begin{equation}\label{appeqs:PBRWPPDE}
        \begin{cases}
            \partial_t \rho(t,x) + \nabla \cdot(\rho(t,x) \nabla \Phi(t, M^{-1}x)) = \beta^{-1}\nabla \cdot(M \nabla \rho)(t,x),  \\
             \partial_t \Phi(t, M^{-1}x) + \frac{1}{2}\|\nabla \Phi(t, M^{-1}x)\|_M^2= -\beta^{-1} \Tr(M^{-1}(\nabla^2 \Phi)(t, M^{-1}x)), \\
             \rho(0,x) = \rho_0(x),\quad \Phi(T,M^{-1}x) = -V(x).
        \end{cases}
    \end{equation}
We claim that the coupled PDEs arise from a Lagrange multiplier, applied to the following mean field control problem.
\begin{proposition}
    Consider the following MFC problem
    \begin{equation}
        \inf_{\rho, u, q} \int_0^T \int_{\R^d} \frac{1}{2}\|u(t,x)\|^2_{M^{-1}} \rho(t,x) \dd{x} \dd{t} + \int_{\R^d} V(x) q(x) \dd{x},
    \end{equation}
subject to the continuity equation and boundary conditions    \begin{equation}
        \partial_t \rho(t,x) + \nabla \cdot (\rho(t,x) Mu(t,x)) = \beta^{-1} \nabla\cdot(M \nabla \rho(t,x)),\quad \rho(0,x) = \rho_0(x),\quad \rho(T,x) = q.
    \end{equation}
    The solution of the above problem satisfies the PDE system \labelcref{appeqs:PBRWPPDE}.
\end{proposition}
\begin{proof}
    We follow the method in \cite{li2023kernel}. We consider the following augmented Lagrangian $\gL$, where we introduce $\Phi:[0,T] \times \R^d \rightarrow \R$ as a Lagrange multiplier.
    \begin{align*}
        \gL(\rho, u, \rho_T, \Phi) &\coloneqq \int_0^T \int_{\R^d} \frac{1}{2}\|u(t,x)\|^2_{M^{-1}} \rho(t,x) \dd{x}\dd{t} + \int_{\R^d} V(x) \rho(T,x)\\
        & \quad +  \int_0^T \int_{\R^d} \Phi(t, M^{-1}x) \left[\partial_t \rho(t,x) + \nabla \cdot (\rho M u) - \beta^{-1} \nabla \cdot(M \nabla \rho)\right] \dd{x} \dd{t}\\
        &= \int_0^T \int_{\R^d} \frac{1}{2} \|u\|^2_{M^{-1}} \rho \dd{x} \dd{t} + \int_{\R^d} V(x) \rho(T,x) \dd{x}\\
        &\quad + \int_0^T \int_{\R^d} -\partial_t \Phi(t, M^{-1}x) \rho(t,x) - \nabla[\Phi(t, M^{-1}x)] \cdot (\rho M u) \\&\hspace{5.8cm}+ \beta^{-1} \nabla[\Phi(t, M^{-1}x)] \cdot(M \nabla \rho) \dd{x} \dd{t} \\
        &\quad + \int_{\R^d} \Phi(T,M^{-1}x) \rho(t,x) \dd{x} - \int_{\R^d} \Phi(0,M^{-1}x) \rho(0,x) \dd{x}.
    \end{align*}
    Note that
    \begin{gather}
    \nabla[\Phi(t, M^{-1}x)] = M^{-1}(\nabla \Phi)(t, M^{-1}x),\\
    \nabla\cdot[M\nabla[\Phi(t, M^{-1}x)]] = \Tr(M^{-1} (\nabla^2 \Phi)(t, M^{-1}x)).\label{appeq:secorder}
    \end{gather}
    We may compute the first variations, so the stationary point satisfies the following. 
    \begin{align*}
        \frac{\delta \gL}{\delta u} = 0 &\Rightarrow 
        \rho M u - \rho  \nabla \Phi(t, M^{-1}x) = 0\\
        &\Rightarrow Mu = \nabla \Phi(t, M^{-1}x).
    \end{align*}
    \begin{align*}
        \frac{\delta \gL}{\delta \Phi} = 0 \Rightarrow \partial_t \rho + \nabla\cdot(\rho Mu) - \beta^{-1} \nabla \cdot(M \nabla \rho)=0.
    \end{align*}
    Using $Mu = \nabla \Phi(t, M^{-1}x)$ we obtain the regularized Fokker--Planck component
    \begin{equation}
        \partial_t \rho + \nabla\cdot(\rho(t,x) \nabla \Phi(t, M^{-1}x)) - \beta^{-1} \nabla \cdot(M \nabla \rho)=0
    \end{equation}
    From the variation with $\rho$,
    \begin{gather*}
        \frac{\partial \gL}{\partial \rho}=0 \\\Rightarrow \frac{1}{2}\|u\|^2_{M^{-1}} - \partial_t \Phi(t, M^{-1}x) - \nabla[\Phi(t, M^{-1}x)]\cdot (Mu) - \beta^{-1} \nabla\cdot[M \nabla[\Phi(t, M^{-1}x)]]=0.
    \end{gather*}
    Using the previous identity $Mu = \nabla \Phi(t, M^{-1}x)$ and \labelcref{appeq:secorder}, we have
    \begin{equation*}
        \frac{1}{2}\|\nabla \Phi(t, M^{-1}x)\|^2_M  - \partial_t \Phi(t, M^{-1}x) - \|\nabla \Phi(t, M^{-1}x)\|_M^2 - \beta^{-1} \Tr(M^{-1} (\nabla^2 \Phi)(t, M^{-1}x)) = 0.
    \end{equation*}
    Rearranging yields the regularized Hamilton--Jacobi equation 
    \begin{equation}
        \partial_t \Phi(t, M^{-1}x) + \frac{1}{2}\|\nabla \Phi(t, M^{-1}x)\|^2_M = -\beta^{-1} \Tr(M^{-1} (\nabla^2 \Phi)(t, M^{-1}x))
    \end{equation}
    The final saddle point condition yields the boundary condition
    \begin{equation}
        \frac{\delta \gL}{ \delta \rho_T}=0 \Rightarrow  \Phi(T,M^{-1}x) + V(x)=0.
    \end{equation}
\end{proof}

\subsection{Finite second moment}\label{appssec:moment}
Here we prove Proposition \ref{prop:MomentEstimate}, that the PRWPO of a probability measure in $\gP_2$ stays in $\gP_2$.
\begin{proposition}
    Let $\rho_0 \in \gP_2(\R^d)$. Assume that $V \in \gC^1$ satisfies the following regularity conditions:
    \begin{enumerate}
        \item $V$ is lower bounded,
        \item $x \cdot \nabla V(x)$ is lower bounded, i.e. $\exists C\ge 0$ such that 
        \begin{equation*}
            x \cdot \nabla V(x) \ge -C,\quad \forall x \in \R^d.
        \end{equation*}
    \end{enumerate}
    Then the following second moment estimate holds for any $T>0$:
    \begin{equation*}
        \int \|x\|_M^2 \wprox_{T,V}^M \rho_0(x) \dd{x} \le 4T(d\beta^{-1}+C/2) + \mathbb{E}_{\rho_0}[\|x\|^2_M]
    \end{equation*}
    In particular, the PRWPO of $\rho_0$ also has finite second moment.
\end{proposition}
\begin{proof}
Recall the equivalent kernel definition of the PRWPO is given by 
    \begin{equation*}
        \wprox_{T,V}^M \rho_0(x) \eqqcolon \rho_T(x) = \int \frac{\exp(-\frac{\beta}{2}(V(x) + \frac{\|x-y\|^2_M}{2T}))}{\int \exp(-\frac{\beta}{2}(V(z) + \frac{\|z-y\|^2_M}{2T}))\dd{z}} \rho_0(y) \dd{y},
    \end{equation*}
    which is a probability measure. The second moment is given by 
    \begin{align}
        \int \|x\|_M^2 \rho_T(x) \dd{x} &= \iint\|x\|_M^2 \frac{\exp(-\frac{\beta}{2}(V(x) + \frac{\|x-y\|_M^2}{2T}))}{\int \exp(-\frac{\beta}{2}(V(z) + \frac{\|z-y\|_M^2}{2T}))\dd{z}} \rho_0(y) \dd{y} \dd{x} \\
        &= \int \rho_0(y) \int \|x\|_M^2 \frac{\exp(-\frac{\beta}{2}(V(x) + \frac{\|x-y\|_M^2}{2T}))}{\int \exp(-\frac{\beta}{2}(V(z) + \frac{\|z-y\|_M^2}{2T}))\dd{z}}\dd{x} \dd{y}, \label{eq:bigMoment}
    \end{align}
    where the integral can be exchanged using the non-negative form of Fubini-Tonelli. It remains to bound the inner integral. 

    Let $W_y(x) \coloneqq -\frac{\beta}{2} (V(x) + \frac{\|x-y\|^2_M}{2T})$. We therefore wish to bound
    \begin{equation}\label{eq:x2Moment}
        \int \|x\|^2 \frac{\exp(W_y(x))}{\int \exp(W_y(z))\dd{z} }\dd{x}.
    \end{equation}
    We claim that this is bounded by $(T+\|y\|^2)$ times some fixed absolute constant. This would immediately give the result that $\rho_T \in \mathcal{P}_2$. To do this, we observe the following divergence formula. 
    \begin{align*}
        &\quad \nabla_x \cdot (x \exp(W_y(x))) \\
        &= \exp(W_y(x)) \left[d + x\cdot \nabla_x W_y(x)\right]\\
        &= \exp(W_y(x)) \left[d + x \cdot \left(-\frac{\beta}{2}\nabla V(x) - \frac{\beta}{2T}M^{-1}(x-y)\right)\right]
    \end{align*}
    Integrate both sides over $\R^d$. The integral of the divergence term is 0, since $W_y(x)$ is exponentially decaying in $x$ due to the lower-boundedness of $V$, and the $C^1$ regularity suffices for the divergence theorem to hold. Therefore,
    \begin{align}
        0 &= \int \exp(W_y(x)) \left[d + x \cdot \left(-\frac{\beta}{2}\nabla V(x) - \frac{\beta}{2T}M^{-1}(x-y)\right)\right]\dd{x} \\
        &= \int \exp(W_y(x)) \left[d  - \frac{\beta}{2}x \cdot \nabla V(x) - \frac{\beta}{2T} \|x\|^2_M + \frac{\beta}{2T}x^\top M^{-1} y\right]\dd{x}\\
        &\le \int \exp(W_y(x)) \left[d  + C\beta/2 - \frac{\beta}{2T} \|x\|^2_M + \frac{\beta}{4T}\|x\|^2_M + \frac{\beta}{4T}\|y\|^2_M\right]\dd{x} \label{eq:xdVBound}\\
        &= \int \exp(W_y(x)) \left[d  + C\beta/2 - \frac{\beta}{4T} \|x\|^2_M +   \frac{\beta}{4T}\|y\|^2_M\right]\dd{x},
    \end{align}
    where the inequality uses that $x\cdot \nabla V(x) \ge -C$, and Young's inequality $x\cdot y \le (\|x\|^2 + \|y\|^2)/2$. Rearranging yields
    \begin{equation*}
        \int \|x\|_M^2 \exp(W_y(x)) \dd{x} \le \left[(4T)(d\beta^{-1}+C/2) + \|y\|_M^2\right] \int \exp(W_y(x)) \dd{x}.
    \end{equation*}
    Substituting this inequality for \labelcref{eq:x2Moment} into \labelcref{eq:bigMoment}, we have that
    \begin{align*}
        \mathbb{E}_{\rho_t}[\|x\|_M^2] &= \int \rho_0(y) \int \|x\|^2_M \frac{\exp(W_y(x))}{\int \exp(W_y(z))\dd{z} }\dd{x} \dd{y}\\
        &\le \int \rho_0(y) \left[(4T)(d\beta^{-1}+C/2) + \|y\|^2_M\right] \dd{y}\\
        &\le 4T(d\beta^{-1}+C/2) + \mathbb{E}_{\rho_0}[\|x\|^2_M].
    \end{align*}
\end{proof}

The assumption that $x \cdot \nabla V(x)$ can be slightly relaxed. If instead $x^\top \hat M^{-1} \nabla V(x) \ge -C$ for some $C \ge 0$, where $\hat M$ is psd and commutes with $M$, then following the same proof with the divergence form
\begin{align*}
    &\quad \nabla_x \cdot (x^\top \hat M^{-1} \exp(W_y(x))) \\
    &= \exp(W_y(x)) \left[\Tr(\hat M^{-1}) + x^\top \hat M^{-1} \left(-\frac{\beta}{2}\nabla V(x) - \frac{\beta}{2T}M^{-1}(x-y)\right)\right]
\end{align*}
gives that 
\begin{equation}
        \int \|x\|_{\hat MM}^2 \wprox_{T,V}^M \rho_0(x) \dd{x} \le 4T\left(\Tr(\hat{M}^{-1})\beta^{-1}+C/2\right) + \mathbb{E}_{\rho_0}[\|x\|^2_{\hat{M} M}].
\end{equation}

One can also relax the condition to $x \cdot \nabla V(x) \ge -C_1 - C_2 \|x\|^2$ for some $C_1>0,\, C_2 <T^{-1}$ to yield a similar affine growth control in second moment. This is done by appropriately changing Young's inequality in \ref{eq:xdVBound}.

\section{Proofs for Gaussian analysis}
Recall that we wish to sample from some distribution $\exp(-\beta V(x))$, as in Section \ref{sec:Gaussian}. Here we focus on the target distribution $\gN(0, \beta^{-1} \Sigma)$, corresponding to potential $V(x) = x^\top \Sigma^{-1} x/2$. For a fixed preconditioning matrix $M \succ 0$, we will find it convenient to work in a rotated basis. For any covariance matrix $\Sigma$, we define a corresponding (positive definite) $\Xi$ as $\Xi = \sqrt{M}^{-1} \Sigma \sqrt{M}^{-1}$.

Recall from (\ref{eq:kernelFormPrecond}) the normalized kernel 
\begin{equation}
    K_M(x,y) = \frac{\exp(-\frac{\beta}{2} (V(x) + \frac{\|x-y\|_M^2}{2T}))}{\int_{\R^d} \exp(-\frac{\beta}{2} (V(z) + \frac{\|z-y\|_M^2}{2T})) \dd{z}}.
\end{equation}
We may compute for a fixed $y$, noting $K_M(x,y)$ is a Gaussian density in $x$,
\begin{gather*}
    K_M(x,y) \propto \exp(-\frac{\beta}{2} \left(\frac{x^\top \Sigma^{-1} x}{2} + \frac{(x-y)^\top M^{-1} (x-y)}{2 T}\right)) \\
    =  \exp(-\frac{1}{2}\left[ x^\top \left(\frac{\beta \Sigma^{-1}}{2} + \frac{\beta M^{-1}}{2T}\right) x - y^\top \frac{\beta M^{-1}}{2 T} x - x^\top \frac{\beta M^{-1}}{2 T} y + ...\right])\\
     \propto \mathrm{exp}\bigg(-\frac{1}{2}\bigg[\left(x- \left(\frac{\Sigma^{-1}}{2} + \frac{M^{-1}}{2 T}\right)^{-1}\frac{M^{-1}}{2T}y\right)^\top \left(\frac{\beta \Sigma^{-1}}{2 } + \frac{\beta M^{-1}}{2 T}\right) \\ \left(x- \left(\frac{\Sigma^{-1}}{2} + \frac{M^{-1}}{2 T}\right)^{-1}\frac{M^{-1}}{2T}y\right)\bigg]\bigg).
\end{gather*}
We therefore have that
\begin{equation}\label{eqsupp:kernelMarkovGaussian}
K_M(\cdot, y) \sim \gN\left(\left(\frac{\Sigma^{-1}}{2} + \frac{M^{-1}}{2 T}\right)^{-1}\frac{M^{-1}}{2 T}y, \left(\frac{\beta \Sigma^{-1}}{2} + \frac{\beta M^{-1}}{2 T}\right)^{-1}\right).
\end{equation}
Utilizing that $\int K_M(x,y)\dd{x}=1$, the kernel can be written explicitly as
\begin{multline}\label{eqsupp:KMy}
    K_M(x,y) = \frac{1}{B}\mathrm{exp}\bigg(-\frac{1}{2}\bigg[\left(x- \left(\frac{\Sigma^{-1}}{2} + \frac{M^{-1}}{2T}\right)^{-1}\frac{M^{-1}}{2T}y\right)^\top \left(\frac{\beta \Sigma^{-1}}{2 } + \frac{\beta M^{-1}}{2T}\right) \\ \left(x- \left(\frac{\Sigma^{-1}}{2 } + \frac{M^{-1}}{2T}\right)^{-1}\frac{M^{-1}}{2 T}y\right)\bigg]\bigg),
\end{multline}
\begin{equation}
    B = (2\pi)^{d/2} \det(\left(\frac{\beta\Sigma^{-1}}{2} + \frac{\beta M^{-1}}{2 T}\right)^{-1})^{1/2}.
\end{equation}

\subsection{Additional definitions}\label{appssec:scaledfisher}
Consider the scaled Fisher information
\begin{equation}
    \gI_{\pi}^M(\rho) = \int \langle \nabla \log \frac{\rho(x)}{\pi(x)}, M \nabla \log \frac{\rho(x)}{\pi(x)}\rangle \rho(x)\dd{x}.
\end{equation}

Supposing that the terminal distribution $\pi$ satisfies the scaled log-Sobolev inequality, stating that for every locally Lipschitz function $g$, that 
\begin{equation}
    \frac{2}{c} \int \|\nabla g(x)\|_{M^{-1}}^2 \dd\pi \ge \int g^2 \log g^2 \dd \pi - \left(\int g^2 \dd \pi\right) \log(\int g^2 \dd\pi),
\end{equation}
applying with $g(x) = \sqrt{\dd \rho/\dd\pi}$ yields for any appropriate density
\begin{equation}
    \KL(\rho \| \pi) = \int \rho(x) \log \frac{\rho(x)}{\pi(x)} \dd x \le \frac{1}{2c}\gI_{\pi}^M(\rho).
\end{equation}
We note that by the Bakry--Emery criterion, a distribution $\exp(-U(x))$ satisfies the scaled log-Sobolev inequality with parameter $\kappa$ if $\nabla^2 U\succeq \kappa M^{-1}$. 

From \cite[Prop. 1]{jiang2021mirror}, if $\pi$ satisfies the scaled log-Sobolev inequality, we have exponential convergence $\KL(\rho_t|| \pi) \le \exp(-2 \kappa t) \KL(\rho_0\| \pi)$. Moreover, we have the following expression, which may be derived by direct computation,
\begin{align*}
    \gI^M_{\gN(0, \Sigma_\infty)}(\gN(0, \Sigma_t)) &= \Tr(\Sigma_t (\Sigma_\infty^{-1} - \Sigma_t^{-1}) M (\Sigma_\infty^{-1} - \Sigma_t^{-1}))\\
&= \Tr( \Xi_t(\Xi_\infty^{-1} -\Xi_t^{-1})^2),
\end{align*}
where we denote $\Xi_t = \sqrt{M}^{-1} \Sigma_t \sqrt{M}^{-1},\, \Xi_\infty = \sqrt{M}^{-1} \Sigma_\infty \sqrt{M}^{-1}$.
\subsection{Regularized Wasserstein proximal of Gaussian}\label{app:wproxGaussianCF}
\begin{proposition}
    For a distribution $\rho_k \sim \gN(\mu_k, \Sigma_k)$, the Wasserstein proximal satisfies $\wprox^M_{T,V} \rho_k \sim \gN(\tilde\mu_k, \tilde\Sigma_k)$, where
    \begin{gather}
        \tilde\mu_{k} = (I + T M\Sigma^{-1})^{-1} \mu_k, \\
        \tilde\Sigma_k = 2\beta^{-1} T\left(T\Sigma^{-1} + M^{-1}\right)^{-1} + \left(T\Sigma^{-1} + M^{-1}\right)^{-1} M^{-1} \Sigma_k M^{-1}\left(T\Sigma^{-1} + M^{-1}\right)^{-1}.
    \end{gather}
\end{proposition}
\begin{proof}
    We compute the closed-form update from a distribution $\rho_k \sim \gN(\mu_k, \Sigma_k)$ from the update equation
\begin{subequations}
\begin{gather}
    \wprox^M_{T,V} \rho_{k}  = \tilde\rho_k\propto \int_{\R^d} K_M(x,y) \rho_k(y)\dd{y},\\
    K_M(x,y) = \frac{\exp(-\frac{\beta}{2} (V(x) + \frac{\|x-y\|_M^2}{2T}))}{\int_{\R^d} \exp(-\frac{\beta}{2} (V(z) + \frac{\|z-y\|_M^2}{2T})) \dd{z}}. \label{eqsupp:kernelForm}
\end{gather}
\end{subequations}
We have from \cref{eqsupp:KMy} that
\begin{equation}\label{appeq:kernelMarkovGaussian}
    K_M(x, \cdot) \propto \gN((I+TM\Sigma^{-1})x, 2\beta^{-1} T M(T\Sigma^{-1}+M^{-1})M).
\end{equation}
From \cite{petersen2008matrix}, we have that the product of two Gaussian densities is again a Gaussian density, multiplied by a desired normalizing constant that is given by
\begin{equation}
    \gN_{(I+TM\Sigma^{-1})x}(\mu_k, \Sigma_k + 2\beta^{-1} T M(T\Sigma^{-1}+M^{-1})M).
\end{equation}
We abuse the subscript notation to denote the Gaussian density evaluated at the subscript. The integral of the Gaussian density appears in both the numerator and denominator and is removed from the expression. We thus have that 
\begin{align*}
    \tilde\rho_k &\propto \int_{\R^d} K_M(x,y) \rho_k(y) \dd{y} \\
    &\propto \gN_{(I+TM\Sigma^{-1})x}(\mu_k, \Sigma_k + 2\beta^{-1} T M(T\Sigma^{-1}+M^{-1})M)\\
    &\propto \gN_{x}((I+TM\Sigma^{-1})^{-1}\mu_k,\\
    &\qquad (I+TM\Sigma^{-1})^{-1}[\Sigma_k + 2\beta^{-1} T M(T\Sigma^{-1}+M^{-1})M](I+TM\Sigma^{-1})^{-\top}).
\end{align*}
Manipulating yields the desired expressions for $\tilde\mu_k$ and $\tilde\Sigma_k$.
\end{proof}

\subsection{Proof of discrete time convergence for Gaussians}\label{app:discreteTime}

Here we prove Theorem \ref{thm:discreteTime}. We restate it here for convenience.

\begin{theorem}
    Suppose $V(x) = x^\top \Sigma^{-1} x/2$, and fix a $\beta >0$. For the target distribution $\pi = \gN(0, \beta^{-1} \Sigma)$ and positive definite preconditioner $M$, let $c, C>0$ be such that $cM \preceq \Sigma \preceq CM$. Let $T \in(0,c)$ so that the inverse PRWPO is well-defined. Then there exists a stationary distribution to the discrete-time PBRWP iterations, and it is Gaussian $\hat\pi = \gN(0, \hat\Sigma)$, satisfying $\wprox_{T,V}^M(\gN(0, \hat\Sigma)) = \gN(0, \beta^{-1} \Sigma)$. Moreover, it is unique within distributions with sufficiently decaying Fourier transforms, specifically as given in Subsection \ref{app:maxT}.

    Suppose further that the initial Gaussian distribution $\rho_0 = \gN(0, \Sigma_0)$ has covariance commuting with $\Sigma$, let $\rho_k$ be the iterations generated by PBRWP, and let $\tilde\rho_{k} = \wprox_{T,V}^M(\rho_{k})$. Let $\lambda$ be such that $\sqrt{M}^{-1} \Sigma_k \sqrt{M}^{-1}$ are uniformly bounded below by some $\lambda$ (which holds as $\Sigma_k \rightarrow \Sigma$). Let the step-size $\eta$ be bounded by 
    \begin{equation*}
        \eta\beta^{-1} \le \min((\max_i |\lambda_i(\tilde \Xi_k^{-1} - \tilde\Xi_\infty^{-1})|)^{-1}/2, 3\lambda_{\min}(\tilde\Xi_k)/32)
    \end{equation*}
    then the exact discrete-time PBRWP method has the following decay in KL divergence:
    \begin{multline}
        \KL(\tilde\rho_{k+1} \| \pi) - \KL(\tilde\rho_{k} \| \pi) \\ \le -\frac{\eta}{2C[\beta+2 T(1+TC^{-1})^{-1}(1+Tc^{-1})^2\lambda^{-1}]}\KL(\tilde\rho_{k}\| \pi).
    \end{multline}
\end{theorem}
\begin{proof}
We show convergence rates and existence first. Let us define $\tilde\Sigma_\infty = \beta^{-1} \Sigma$, motivated by the fact that the PRWPO of the stationary covariance is the target covariance. Further recall that $\tilde\Sigma_k$ is the covariance of the PRWPO of $\gN(0, \Sigma_k)$, in closed form in Proposition \ref{prop:PRWPOGaussian}. Recall the discrete covariance update for PBRWP from Corollary \ref{cor:GaussianUpdate}:
\begin{equation}
    \Sigma_{k+1} = (I - \eta M\Sigma^{-1} + \eta \beta^{-1} M\tilde\Sigma^{-1}_{k}) \Sigma_k (I - \eta M\Sigma^{-1} + \eta \beta^{-1} M\tilde\Sigma^{-1}_{k})^\top .
\end{equation}
In terms of the $M$-conjugated matrices $\Xi_k = \sqrt{M}^{-1} \Sigma_k \sqrt{M}^{-1},\, \tilde\Xi_k = \sqrt{M}^{-1} \tilde\Sigma_k \sqrt{M}^{-1}$, this becomes
\begin{equation}\label{appeq:XiEvo}
    \Xi_{k+1} = (I - \eta \beta^{-1} \tilde\Xi_\infty^{-1} + \eta \beta^{-1} \tilde\Xi^{-1}_{k}) \Xi_k (I - \eta \beta^{-1} \tilde\Xi_\infty^{-1} + \eta \beta^{-1} \tilde\Xi^{-1}_{k}).
\end{equation}
We define the following auxiliary matrix expressions for convenience.
\begin{gather}
    K_+ \coloneqq I + T \sqrt{M}\Sigma^{-1} \sqrt{M},\quad K_- \coloneqq I - T\sqrt{M}\Sigma^{-1} \sqrt{M}.
\end{gather}
Then, the (conjugated) covariance of the Wasserstein proximal $\tilde\Xi_k = \sqrt{M}^{-1} \tilde\Sigma_k \sqrt{M}^{-1}$, $\tilde\Sigma_k = \Cov(\wprox_{T,V}^M \gN(0, \Sigma_k))$ has the following iterations
\begin{gather}\label{appeq:wproxXi}
    \tilde{\Xi}_k = 2\beta^{-1} T K_+^{-1}  +  K_+^{-1} \Xi_k K_+^{-1}.
\end{gather}
At the stationary point, we can compute the inverse PRWPO as
\begin{gather}
    \Xi_\infty = K_- \tilde{\Xi}_\infty K_+ = K_+ \tilde{\Xi}_\infty K_-, \quad \tilde{\Xi}_\infty = \beta^{-1} \sqrt{M}^{-1}\Sigma\sqrt{M}^{-1}.
\end{gather}
They satisfy the following:
\begin{gather}
    \tilde{\Xi}_k - \tilde{\Xi}_\infty = K_+^{-1} (\Xi_k - \Xi_\infty)K_+^{-1}, \label{appeq:XiTildeDiff}\\
    \tilde{\Xi}_k = K_+^{-1}(\Xi_k + 2\beta^{-1} T K_+)K_+^{-1}.\label{appeq:XiTildeXiDef}
\end{gather}
The update of the $\tilde\Xi$ can be written as follows, using (\ref{appeq:XiEvo}) and (\ref{appeq:XiTildeXiDef}),
\begin{align}
    \tilde\Xi_{k+1} &= K_+^{-1}(\Xi_{k+1} + 2\beta^{-1} T K_+)K_+^{-1} \notag\\
    &= K_+^{-1}\left[(I - \eta \beta^{-1} \tilde\Xi_\infty^{-1} + \eta \beta^{-1} \tilde\Xi^{-1}_{k}) \Xi_k (I - \eta \beta^{-1} \tilde\Xi_\infty^{-1} + \eta \beta^{-1} \tilde\Xi^{-1}_{k}) + 2\beta^{-1} T K_+\right]K_+^{-1} .\label{appeq:wproxEvo}
\end{align}
\textbf{Claim.} We will show that $\tilde\Xi_k \rightarrow \tilde\Xi_\infty$, and accordingly that $\Xi_k \rightarrow \Xi_\infty$.

By the assumption that $\Sigma_0$ commutes with $\Sigma$, we have that $\Xi_k$, $\tilde\Xi_k$, $K_\pm$, $\tilde\Xi_\infty^{-1}$ all commute.

We recall a useful Taylor expansion first.
\begin{proposition}
    Let $A$ be square with eigenvalues $\lambda_i \ge 0$. Then
    \begin{equation}
        \log \det (I+A) = \Tr(A) + R
    \end{equation}
    where 
    \begin{equation}
        0 \ge R = \sum_i (\log(\lambda_i+1)-\lambda_i) \ge -\sum \frac{\lambda_i^2}{1+\lambda_i} \ge -\sum_i \lambda_i^2.
    \end{equation}
    The inequalities other than the last one also hold for $\lambda_i > -1$.
\end{proposition}
\begin{proof}
    Using $\det(A) = \prod_i \lambda_i$ and $\Tr(A) = \sum_i \lambda_i$.
\end{proof}

Recall that for Gaussians, the KL divergence to $\pi \sim \gN(0, \beta^{-1} \Sigma) = \gN(0, \tilde\Sigma_\infty)$ takes the particular form \cite{hershey2007approximating}:
\begin{align*}
    \KL(\tilde\rho_k \| \pi) &= \int \tilde\rho_k(x) \log \frac{\tilde\rho_k(x)}{\pi(x)} \dd{x} \\
    &= \frac{1}{2}\left[\log \det \tilde{\Sigma}_\infty - \log \det \tilde\Sigma_k - d + \Tr(\tilde \Sigma_\infty^{-1} \tilde\Sigma_k)\right].
\end{align*}
The difference between two consecutive iterations is
\begin{align}
    & \quad \KL(\tilde\rho_{k+1} \| \pi) - \KL(\tilde\rho_{k} \| \pi) \notag \\
    &= \frac{1}{2} \left[\log \frac{|\tilde\Sigma_{k}|}{|\tilde\Sigma_{k+1}|} + \Tr(\tilde\Sigma_\infty^{-1}(\tilde\Sigma_{k+1} - \tilde\Sigma_k))\right] \notag\\
    &= \frac{1}{2} \left[\log \frac{|\tilde\Xi_{k}|}{|\tilde\Xi_{k+1}|} + \Tr(\tilde\Xi_\infty^{-1}(\tilde\Xi_{k+1} - \tilde\Xi_k))\right]. \label{appeq:KLConsecDiff}
\end{align}
Applying the Taylor expansion to the first term,
\begin{align}
    \log \frac{|\tilde\Xi_{k}|}{|\tilde\Xi_{k+1}|} &= \log \frac{|\tilde\Xi_{k}|}{|\tilde\Xi_{k} + (\tilde\Xi_{k+1}-\tilde\Xi_{k})|} \notag\\
    &= -\log |I + (\tilde\Xi_{k+1}-\tilde\Xi_{k}) \tilde\Xi_{k}^{-1}| \notag\\
    &= -\Tr((\tilde\Xi_{k+1}-\tilde\Xi_{k}) \tilde\Xi_{k}^{-1}) - R_1, \label{appeq:logDiffTrR1}
\end{align}
where
\begin{equation}\label{appeq:r1}
    R_1 = \sum_{i}\left[\log (1 + \lambda_i((\tilde\Xi_{k+1}-\tilde\Xi_{k}) \tilde\Xi_{k}^{-1})) - \lambda_i((\tilde\Xi_{k+1}-\tilde\Xi_{k}) \tilde\Xi_{k}^{-1})\right].
\end{equation}
Therefore we have 
\begin{equation}\label{appeq:KLDiffCondensedR1}
    \KL(\tilde\rho_{k+1} \| \pi) - \KL(\tilde\rho_{k} \| \pi) = \frac{1}{2} \left[\Tr((\tilde\Xi_\infty^{-1}-\tilde\Xi_k^{-1})(\tilde\Xi_{k+1} - \tilde\Xi_k)) - R_1\right].
\end{equation}
To bound $R_1$, we need to control the eigenvalues of $(\tilde\Xi_{k+1}-\tilde\Xi_{k}) \tilde\Xi_{k}^{-1}$. Since all the matrices in question commute, from \cref{appeq:wproxEvo}, we have
\begin{gather}
    \tilde\Xi_{k+1}-\tilde\Xi_{k} =  (2\eta\beta^{-1}(\tilde\Xi_k^{-1} - \tilde\Xi_\infty^{-1})\Xi_k + \eta^2\beta^{-2}(\tilde\Xi_k^{-1} - \tilde\Xi_\infty^{-1})^2 \Xi_k) K_+^{-2}.\label{appeq:tildeXiEvo}
\end{gather}
We now control the first term of \cref{appeq:KLDiffCondensedR1}. This expression is 
\begin{align*}
    &\quad \Tr((\tilde\Xi_\infty^{-1}-\tilde\Xi_k^{-1})(\tilde\Xi_{k+1} - \tilde\Xi_k)) \\
    &= \eta\beta^{-1} \Tr((\tilde\Xi_\infty^{-1}-\tilde\Xi_k^{-1})K_+^{-2}\left[2(\tilde\Xi_k^{-1}-\tilde\Xi_\infty^{-1})\Xi_k +\eta\beta^{-1} (\tilde\Xi_k^{-1}-\tilde\Xi_\infty^{-1})\Xi_k(\tilde\Xi_k^{-1}-\tilde\Xi_\infty^{-1})\right])\\
    &= -2 \eta \beta^{-1} \Tr((\tilde\Xi_k^{-1}-\tilde\Xi_\infty^{-1})^2K_+^{-2}\Xi_k) + \eta^2\beta^{-2} \Tr((\tilde\Xi_k^{-1}-\tilde\Xi_\infty^{-1})^3 K_+^{-2}\Xi_k),
\end{align*}
where we use \cref{appeq:XiTildeDiff}, \cref{appeq:XiEvo}, and finally collect terms. To control all of \cref{appeq:KLDiffCondensedR1}, we use the sum-form of trace, and are required to upper-bound the expressions
\begin{equation*}
    \lambda_i\left((\tilde\Xi_\infty^{-1}-\tilde\Xi_k^{-1})(\tilde\Xi_{k+1} - \tilde\Xi_k)\right) - \log (1 + \lambda_i((\tilde\Xi_{k+1}-\tilde\Xi_{k}) \tilde\Xi_{k}^{-1})) + \lambda_i((\tilde\Xi_{k+1}-\tilde\Xi_{k}) \tilde\Xi_{k}^{-1}).
\end{equation*}
Summing this expression over $i$ yields \cref{appeq:KLDiffCondensedR1}.

\textbf{Step-size condition}: assume $\eta\beta^{-1} \le (\max_i |\lambda_i(\tilde \Xi_k^{-1} - \tilde\Xi_\infty^{-1})|)^{-1}/2$. We abuse notation now to write $\lambda_i$ as the eigenvalue of a matrix with respect to the $i$-th shared eigenvector $v_i$. 

\textbf{Case 1}: suppose the eigenvalue of $(\tilde\Xi_k^{-1} - \tilde\Xi_\infty^{-1})$ is positive. Using \cref{appeq:tildeXiEvo}, we have that $\lambda_i((\tilde\Xi_{k+1}-\tilde\Xi_{k}) \tilde\Xi_{k}^{-1})>0$, so we can control
\begin{gather*}
    \log (1 + \lambda_i((\tilde\Xi_{k+1}-\tilde\Xi_{k}) \tilde\Xi_{k}^{-1})) - \lambda_i((\tilde\Xi_{k+1}-\tilde\Xi_{k}) \tilde\Xi_{k}^{-1}) \\
    \ge -\lambda_i^2((\tilde\Xi_{k+1}-\tilde\Xi_{k}) \tilde\Xi_{k}^{-1})\\
    = -\lambda_i((\tilde\Xi_k^{-1}-\tilde\Xi_\infty^{-1})^2 K_+^{-4}\Xi_k^2 \tilde\Xi_k^{-2} )[2\eta\beta^{-1} + \eta^{2}\beta^{-2} \lambda_i(\tilde\Xi_k^{-1}-\tilde\Xi_\infty^{-1})]^2.
\end{gather*}
Therefore
\begin{align*}
    &\qquad \lambda_i((\tilde\Xi_\infty^{-1}-\tilde\Xi_k^{-1})(\tilde\Xi_{k+1} - \tilde\Xi_k))- \log (1 + \lambda_i((\tilde\Xi_{k+1}-\tilde\Xi_{k}) \tilde\Xi_{k}^{-1})) + \lambda_i((\tilde\Xi_{k+1}-\tilde\Xi_{k}) \tilde\Xi_{k}^{-1}) \\
    &\le \lambda_i((\tilde\Xi_k^{-1}-\tilde\Xi_\infty^{-1})^2K_+^{-2}\Xi_k) \\  &\quad \cdot \left[-2\eta\beta^{-1} + \eta^2\beta^{-2} \lambda_i(\tilde\Xi_k^{-1}-\tilde\Xi_\infty^{-1}) + \lambda_i( K_+^{-2}\Xi_k \tilde\Xi_k^{-2} )[2\eta\beta^{-1} + \eta^2\beta^{-2} \lambda_i(\tilde\Xi_k^{-1}-\tilde\Xi_\infty^{-1})]^2\right]\\
    & \le \lambda_i((\tilde\Xi_k^{-1}-\tilde\Xi_\infty^{-1})^2K_+^{-2}\Xi_k) \\
    &\quad \cdot \left[-2\eta\beta^{-1} + \eta^2\beta^{-2} \lambda_i(\tilde\Xi_k^{-1}-\tilde\Xi_\infty^{-1}) + \lambda_i(\tilde\Xi_k^{-1} )[2\eta\beta^{-1} + \eta^2\beta^{-2} \lambda_i(\tilde\Xi_k^{-1}-\tilde\Xi_\infty^{-1})]^2\right] \\
    &\le \lambda_i((\tilde\Xi_k^{-1}-\tilde\Xi_\infty^{-1})^2K_+^{-2}\Xi_k) \\
    &\quad \cdot\left[-2\eta\beta^{-1} + \eta^2\beta^{-2} \lambda_i(\tilde\Xi_k^{-1}-\tilde\Xi_\infty^{-1}) + \frac{1}{2}\beta T^{-1}(1+Tc^{-1}) \eta^2\beta^{-2}[2 + \eta\beta^{-1} \lambda_i(\tilde\Xi_k^{-1}-\tilde\Xi_\infty^{-1})]^2\right],
\end{align*}
where we use that $\tilde\Xi_k \succeq K_+^{-2} \Xi_k$ in the second inequality, and then that $\tilde\Xi_k \succeq 2\beta^{-1} T K_+^{-1} \succeq 2\beta^{-1} T(1+Tc^{-1})^{-1} I$ in the third inequality.

\textbf{Case 2:} the eigenvalue of $(\tilde\Xi_k^{-1} - \tilde\Xi_\infty^{-1})$ is negative. Then
\begin{gather*}
    \log (1 + \lambda_i((\tilde\Xi_{k+1}-\tilde\Xi_{k}) \tilde\Xi_{k}^{-1})) - \lambda_i((\tilde\Xi_{k+1}-\tilde\Xi_{k}) \tilde\Xi_{k}^{-1}) \\
    \ge -\frac{\lambda_i^2((\tilde\Xi_{k+1}-\tilde\Xi_{k}) \tilde\Xi_{k}^{-1})}{1+\lambda_i((\tilde\Xi_{k+1}-\tilde\Xi_{k}) \tilde\Xi_{k}^{-1})}\\
    = -\frac{\lambda_i^2\left[(2\eta\beta^{-1}(\tilde\Xi_k^{-1} - \tilde\Xi_\infty^{-1})\Xi_k + \eta^2\beta^{-2}(\tilde\Xi_k^{-1} - \tilde\Xi_\infty^{-1})^2 \Xi_k) K_+^{-2}\tilde\Xi_k^{-1}\right]}{1+\lambda_i\left[(2\eta\beta^{-1}(\tilde\Xi_k^{-1} - \tilde\Xi_\infty^{-1})\Xi_k + \eta^2\beta^{-2}(\tilde\Xi_k^{-1} - \tilde\Xi_\infty^{-1})^2 \Xi_k) K_+^{-2}\tilde\Xi_k^{-1}\right]}\\
    \ge -\frac{\lambda_i^2\left[(2\eta\beta^{-1}(\tilde\Xi_k^{-1} - \tilde\Xi_\infty^{-1})\Xi_k) K_+^{-2}\tilde\Xi_k^{-1}\right]}{1+\lambda_i\left[(2\eta\beta^{-1}(\tilde\Xi_k^{-1} - \tilde\Xi_\infty^{-1})\Xi_k + \eta^2\beta^{-2}(\tilde\Xi_k^{-1} - \tilde\Xi_\infty^{-1})^2 \Xi_k) K_+^{-2}\tilde\Xi_k^{-1}\right]},
\end{gather*}
where we use the step-size condition to make sure the denominator is positive, then discard the second order term in the numerator. This uses \cref{appeq:wproxXi}, which gives the control $\lambda_i(\Xi_k K_+^{-2} \tilde\Xi_k^{-1}) \le \lambda_i(\tilde\Xi_k \tilde\Xi_k^{-1}) = 1$. Using the step-size condition on $\eta\beta^{-1}$ again and the quadratic structure, the denominator is at least $1/4$, and we have 
\begin{gather*}
    \log (1 + \lambda_i((\tilde\Xi_{k+1}-\tilde\Xi_{k}) \tilde\Xi_{k}^{-1})) - \lambda_i((\tilde\Xi_{k+1}-\tilde\Xi_{k}) \tilde\Xi_{k}^{-1}) \\
    \ge -4{\lambda_i^2\left[(2\eta\beta^{-1}(\tilde\Xi_k^{-1} - \tilde\Xi_\infty^{-1})\Xi_k) K_+^{-2}\tilde\Xi_k^{-1}\right]}
\end{gather*}
Therefore
\begin{gather*}
    \lambda_i((\tilde\Xi_\infty^{-1}-\tilde\Xi_k^{-1})(\tilde\Xi_{k+1} - \tilde\Xi_k)) - \log (1 + \lambda_i((\tilde\Xi_{k+1}-\tilde\Xi_{k}) \tilde\Xi_{k}^{-1})) + \lambda_i((\tilde\Xi_{k+1}-\tilde\Xi_{k}) \tilde\Xi_{k}^{-1})\\
    \le \lambda_i((\tilde\Xi_k^{-1}-\tilde\Xi_\infty^{-1})^2K_+^{-2}\Xi_k) \left[-2\eta\beta^{-1} + \eta^2\beta^{-2} \lambda_i(\tilde\Xi_k^{-1}-\tilde\Xi_\infty^{-1}) + 4\lambda_i(\Xi_k K_+^{-2}\tilde\Xi_k^{-2})(2\eta\beta^{-1})^2\right]\\
    \le \lambda_i((\tilde\Xi_k^{-1}-\tilde\Xi_\infty^{-1})^2K_+^{-2}\Xi_k) \left[-2\eta\beta^{-1} + 4\lambda_i(\tilde\Xi_k^{-1})(2\eta\beta^{-1})^2\right]
\end{gather*}

Taking $\eta\beta^{-1} \le \min((\max_i |\lambda_i(\tilde \Xi_k^{-1} - \tilde\Xi_\infty^{-1})|)^{-1}/2, 3\lambda_{\min}(\tilde\Xi_k)/32)$, we have in

\textbf{Case 1}:
\begin{align*}
     &\qquad \lambda_i((\tilde\Xi_\infty^{-1}-\tilde\Xi_k^{-1})(\tilde\Xi_{k+1} - \tilde\Xi_k) - \log (1 + \lambda_i((\tilde\Xi_{k+1}-\tilde\Xi_{k}) \tilde\Xi_{k}^{-1})) + \lambda_i((\tilde\Xi_{k+1}-\tilde\Xi_{k}) \tilde\Xi_{k}^{-1})) \\ 
     &\le \lambda_i((\tilde\Xi_k^{-1}-\tilde\Xi_\infty^{-1})^2K_+^{-2}\Xi_k) \\
     &\qquad \cdot \left[-2\eta\beta^{-1} + \eta^2\beta^{-2} \lambda_i(\tilde\Xi_k^{-1}-\tilde\Xi_\infty^{-1}) + \lambda_i(\tilde\Xi_k^{-1} )[2\eta\beta^{-1} + \eta^2\beta^{-2} \lambda_i(\tilde\Xi_k^{-1}-\tilde\Xi_\infty^{-1})]^2\right] \\
     &\le \eta\beta^{-1}\lambda_i((\tilde\Xi_k^{-1}-\tilde\Xi_\infty^{-1})^2K_+^{-2}\Xi_k) \left[-2 + 1/2 + \left(\frac{5}{2}\right)^2 \cdot \frac{3}{32}\right]\\
     &\le -\frac{1}{2}\eta\beta^{-1} \lambda_i((\tilde\Xi_k^{-1}-\tilde\Xi_\infty^{-1})^2K_+^{-2}\Xi_k)
\end{align*}

\textbf{Case 2}:
\begin{gather*}
    \lambda_i((\tilde\Xi_\infty^{-1}-\tilde\Xi_k^{-1})(\tilde\Xi_{k+1} - \tilde\Xi_k)) - \log (1 + \lambda_i((\tilde\Xi_{k+1}-\tilde\Xi_{k}) \tilde\Xi_{k}^{-1})) + \lambda_i((\tilde\Xi_{k+1}-\tilde\Xi_{k}) \tilde\Xi_{k}^{-1})\\
    \le \lambda_i((\tilde\Xi_k^{-1}-\tilde\Xi_\infty^{-1})^2K_+^{-2}\Xi_k) \left[-2\eta\beta^{-1} + 4\lambda_i(\tilde\Xi_k^{-1})(2\eta\beta^{-1})^2\right]\\
    \le \eta\beta^{-1} \lambda_i((\tilde\Xi_k^{-1}-\tilde\Xi_\infty^{-1})^2K_+^{-2}\Xi_k) \left[-2 + 16 \cdot \frac{3}{32}\right]\\
     = -\frac{1}{2}\eta\beta^{-1} \lambda_i((\tilde\Xi_k^{-1}-\tilde\Xi_\infty^{-1})^2K_+^{-2}\Xi_k)
\end{gather*}

In either case we have that 
\begin{gather}
    \lambda_i((\tilde\Xi_\infty^{-1}-\tilde\Xi_k^{-1})(\tilde\Xi_{k+1} - \tilde\Xi_k))- \log (1 + \lambda_i((\tilde\Xi_{k+1}-\tilde\Xi_{k}) \tilde\Xi_{k}^{-1})) + \lambda_i((\tilde\Xi_{k+1}-\tilde\Xi_{k}) \tilde\Xi_{k}^{-1})\\
    \le -\frac{1}{2}\eta\beta^{-1} \lambda_i((\tilde\Xi_k^{-1}-\tilde\Xi_\infty^{-1})^2K_+^{-2}\Xi_k)
\end{gather}
and summing over $i$ we have 
\begin{align*}
    &\qquad\KL(\tilde\rho_{k+1} \| \pi) - \KL(\tilde\rho_{k} \| \pi) \\
    &= \frac{1}{2}\sum_i\bigg[\lambda_i((\tilde\Xi_\infty^{-1}-\tilde\Xi_k^{-1})(\tilde\Xi_{k+1} - \tilde\Xi_k))\\
    &\qquad \hphantom{\frac{1}{2}\sum_i\bigg[} - \log (1 + \lambda_i((\tilde\Xi_{k+1}-\tilde\Xi_{k}) \tilde\Xi_{k}^{-1})) + \lambda_i((\tilde\Xi_{k+1}-\tilde\Xi_{k}) \tilde\Xi_{k}^{-1})\bigg]\\
    &\le -\frac{1}{4}\eta\beta^{-1} \Tr((\tilde\Xi_k^{-1}-\tilde\Xi_\infty^{-1})^2K_+^{-2}\Xi_k) \\
    &\le -\frac{1}{4}\eta\beta^{-1} {\lambda_{\min}(K_+^{-2}\Xi_k \tilde\Xi_k^{-1})} \Tr((\tilde\Xi_k^{-1}-\tilde\Xi_\infty^{-1})^2\tilde\Xi_k) \\
    &= -\frac{1}{4}\eta\beta^{-1} {\lambda_{\min}(K_+^{-2}\Xi_k \tilde\Xi_k^{-1})} \gI^M_{\pi}(\tilde\rho_{k}) \\
    &\le -\frac{1}{2C} \eta\beta^{-1}\lambda_{\min}(K_+^{-2}\Xi_k \tilde\Xi_k^{-1}) \KL(\tilde\rho_{k}\|\pi).
\end{align*}
where we use that the log-Sobolev constant of $\tilde\Xi_\infty = \sqrt{M}^{-1} \Sigma \sqrt{M}^{-1}$ is $C^{-1}$ as in \Cref{appssec:scaledfisher}. Finally note that 
\begin{align*}
    \lambda_{\min}(K_+^{-2}\Xi_k \tilde\Xi_k^{-1}) &\ge \frac{\lambda_{\min}(K_+^{-2}\Xi_k)}{2\beta^{-1} T(1+TC^{-1})^{-1} + \lambda_{\min}(K_+^{-2}\Xi_k)}\\
    &\ge \frac{(1+Tc^{-1})^{-2}\lambda_{\min}(\Xi_k)}{2\beta ^{-1}T(1+TC^{-1})^{-1} + (1+Tc^{-1})^{-2}\lambda_{\min}(\Xi_k)}\\
    &= \frac{1}{1+2\beta^{-1} T(1+TC^{-1})^{-1}(1+Tc^{-1})^2\lambda_{\min}(\Xi_k)^{-1}},
\end{align*}
where in the second inequality, we use that $x/(c+x)$ is increasing and $\lambda_{\min}(K_+^{-2}\Xi_k)\ge (1+Tc^{-1})^{-2}\lambda_{\min}(\Xi_k)$.

\textbf{Uniqueness:} now we have established existence of a stationary solution as the Gaussian $\gN(0, \Sigma_\infty)$ satisfying $\wprox_{T,V}^M (\gN(0, \Sigma_\infty)) = \gN(0, \beta^{-1}\Sigma)$. From the modified Fokker--Planck equation (\ref{eq:ModFP}), stationarity at a distribution $\nu$ is equivalent to satisfying $\wprox_{T,V}^M \nu \propto \exp(-\beta V) = \gN(0, \beta^{-1}\Sigma)$. Consider the definition of the regularized Wasserstein proximal as the coupled Fokker--Planck and Hamilton--Jacobi equation, plus the Cole--Hopf transformation. Since the forward- and backwards- anisotropic heat equations are well defined and are injective, the inverse of the Wasserstein proximal is unique.
\end{proof}

\section{Additional results}\label{appsec:finpartres}
\subsection{Sufficient conditions for a unique inverse proximal} \label{app:maxT} 
We wish to find sufficient conditions on a distribution $\rho$ such that the PRWPO (\ref{eq:kernelRho}) is invertible, that is, there exists a unique distribution $\rho_0$ for which $\rho = \wprox_{T,V}^M \rho_0$ is the preconditioned regularized Wasserstein proximal. Equivalently, from (\ref{eq:aniHeat}), we find sufficient conditions that the backward heat equation for $\hat{\eta}$ from time $T$ to 0 is well defined. The forward heat equation is always well defined.

Recall that the solution to the anisotropic (forward) heat equation
\begin{equation*}
    \partial_t f = \beta^{-1} \nabla \cdot (M \nabla f)
\end{equation*}
is given in Fourier space by 
\begin{equation}
    f(t,x) = \int_{\R^d} \hat{f}(\xi) e^{i \langle \xi, x\rangle - t\beta^{-1} \xi^\top M \xi}\dd \xi.
\end{equation}
In order for $f(-T, x)$ to be well defined by Plancherel's theorem, we simply need that the Fourier transform $\hat{f}(\xi)$ decays faster than $e^{-T\beta^{-1} \xi^\top M \xi}$. Since the terminal condition for $\hat{\eta}$ is $\hat{\eta}(T,x) = \rho(x) e^{-\beta \Phi(T, M^{-1}x)/2 } = \rho(T,x) e^{\beta V(x)/2}$, which evolves through the backward heat equation, we want $\rho(x) e^{\beta V(x)/2}$ to have Fourier transform decaying (at least) faster than $e^{-T\beta^{-1} \xi^\top M \xi}$.

\subsubsection{Application: maximum value of T}
Consider the quadratic potential case where $V = x^\top \Sigma^{-1} x/2$, and we wish to find $\rho_0$ whose regularized Wasserstein proximal is proportional to $\exp(-\beta V)$. The Fourier transform of a normal distribution $\rho \sim \gN(\mu, \Sigma)$ (up to some factors of $\sqrt{2\pi}$) is 
\begin{equation}
    \hat{\rho} = \mathbb{E}_{x \sim \rho}\left[e^{-i \langle \xi, x\rangle}\right] = e^{i \langle \xi, \mu \rangle - \frac{1}{2} \xi^\top \Sigma \xi}.
\end{equation}

In this case, $\rho \exp(\beta V/2) \propto \exp(-\beta V/2)$ is proportional to a Gaussian with covariance $2\beta^{-1} \Sigma$, and thus the decay of the Fourier transform is $\sim e^{-\beta^{-1} \xi^\top \Sigma \xi}$. This is at least faster than $e^{-T\beta^{-1} \xi^\top M \xi}$ if and only if $\Sigma \succeq TM$.

\subsection{Regularized Wasserstein proximal is a Wasserstein-2 contraction} From (\ref{eq:kernelMarkovGaussian}) or \cref{appeq:kernelMarkovGaussian}, we have that $K_M$ is a composition of the Gaussian convolution Markov kernel $\delta_y \mapsto \gN(y, (\frac{\beta\Sigma^{-1}}{2 } + \frac{\beta M^{-1}}{2 T})^{-1})$ with the (linear) map $\omega : \R^d \rightarrow \R^d$, defined as $y \mapsto (\frac{\Sigma^{-1}}{2} + \frac{M^{-1}}{2 T})^{-1}\frac{M^{-1}}{2T}y$. Since $c^{-1} M^{-1}\succeq \Sigma^{-1} \succeq C^{-1} M^{-1}$, we have that $\frac{\beta\Sigma^{-1}}{2} + \frac{\beta M^{-1}}{2T} \succeq (\frac{\beta C^{-1}}{2} + \frac{\beta}{2 T}) M^{-1}$. Therefore the eigenvalues of the scaling matrix satisfy
\begin{equation*}
    \lambda_i\left(\frac{M^{-1}}{2 T}\left(\frac{\Sigma^{-1}}{2} + \frac{M^{-1}}{2 T}\right)^{-2}\frac{M^{-1}}{2T}\right) 
    \le \left(\frac{C^{-1}}{2} + \frac{1}{2 T}\right)^{-2} \left(\frac{1}{2 T}\right)^2< 1.
\end{equation*}
In particular, the scaling map $\omega$ is contractive with constant $\zeta \coloneqq(\frac{C^{-1}}{2} + \frac{1}{2T})^{-1} (\frac{1}{2 T})<1$.

Let $\mu,\nu \in \gP_2(\R^d)$ be any probability measures with finite second moment, and let $\gamma \in \Pi(\mu,\nu)$. We wish to bound
\begin{equation*}
    \gW_2(\wprox^M_{T, V}\mu, \wprox^M_{T, V}\nu).
\end{equation*}
Note that $\wprox^M_{T, V} \mu= \gN(0, (\frac{\beta \Sigma^{-1}}{2} + \frac{\beta M^{-1}}{2T})^{-1}) * \omega_\# \mu$. Since $\omega$ is a contractive map with constant $\zeta$, we have that 
\begin{equation}
    \gW_2(\omega_\#\mu, \omega_\#\nu) \le \zeta\gW_2(\mu,\nu).
\end{equation}

From \cite[Lem. 5.2]{santambrogio2015optimal}, since the kernel given by the pdf of $\gN(0,(\frac{\beta \Sigma^{-1}}{2} + \frac{\beta M^{-1}}{2 T})^{-1})$ is even, we have moreover that for any $\mu,\nu$ probability measures with finite second moment, that
\begin{equation}
    \gW_2\left(\gN\left(0,\left(\frac{\beta \Sigma^{-1}}{2 } + \frac{\beta M^{-1}}{2T}\right)^{-1}\right) * \mu, \gN\left(0,\left(\frac{\beta \Sigma^{-1}}{2 } + \frac{\beta M^{-1}}{2 T}\right)^{-1}\right)* \nu\right) \le \gW_2(\mu,\nu). 
\end{equation}
Therefore, we have that for quadratic $V$, and any $\mu,\nu$ in Wasserstein-2 space, that
\begin{equation}
    \gW_2(\wprox^M_{T\cdot V}\mu, \wprox^M_{T\cdot V}\nu) \le \zeta \gW_2(\mu, \nu).
\end{equation}
This also holds for any $p \in [1,\infty)$.

\subsection{Mean reversion from high variance initialization}
We list some simple definitions and properties before stating the inequality \cref{appprop:contractDiffusion}. We fix $\beta=1$ for simplicity.

\begin{definition}[Relative Smoothness/Convexity]
Let $\Psi:\mathcal{X} \rightarrow \R$ be a differentiable convex function, defined on a convex set $\mathcal{X}$ (with non-empty interior), which will be used as a reference. Let $f:\mathcal{X} \rightarrow \R$ be another differentiable convex function.

$f$ is \emph{$L$-smooth relative to $\Psi$} if for any $x,y \in \mathrm{int}(\mathcal{X})$,
\begin{equation}
    f(y) \le f(x) + \langle \nabla f(x), y-x \rangle + L B_\Psi(y,x).
\end{equation}

$f$ is \emph{$\mu$-strongly-convex relative to $\Psi$} if for any $x,y \in \mathrm{int}(\mathcal{X})$,
\begin{equation}
    f(y) \ge f(x) + \langle \nabla f(x), y-x \rangle + \mu B_\Psi(y,x).
\end{equation}
In the special case where $\Psi(x) = \frac{1}{2}x^\top A x$ for some (positive definite) symmetric $A$, we write that $f$ is $L$-smooth relative to $A$ and $\mu$-strongly-convex relative to $A$ respectively. In particular, we have that $B_{\Psi}(x,y) = \frac{1}{2}(x-y)^\top A (x-y)$.
\end{definition}

\begin{proposition}[Control on spectral radius]\label{prop:SpecStoc}
    For a row- (or column-) stochastic matrix $A \in \R^{n \times n}$, the spectral norm $\|A\|_2$ equals its largest singular value $\sigma_1(A)$, and moreover satisfies $\|A\|_2 \le \sqrt{n}$.
\end{proposition}
\begin{proof}
$\|A\|_2 \le \sqrt{n} \|A\|_{\infty} = \sqrt{n}$, where $\|\cdot\|_\infty$ is the induced $\infty$-norm, equal to the maximum absolute row sum.
\end{proof}

\begin{proposition}\label{prop:RelStrongConvMono}
    A function $f:\R^n \rightarrow \R$ is $\mu$-strongly convex relative to $A$ if and only if for all $x,y \in \R^n$,
    \begin{equation}
        \langle \nabla f(x) - \nabla f(y), x-y\rangle \ge \mu \langle A(x-y), x-y\rangle.
    \end{equation}
\end{proposition}
\begin{proof}
    Note that $\mu$-relative strong convexity is equivalent to $f-\mu \Psi$ being convex. Hence, $f$ being $\mu$-strongly convex relative to $\frac{1}{2}x^\top Ax$ is equivalent to the following by the monotone gradient condition, 
    \begin{equation*}
        \langle\nabla (f-\mu\Psi)(x) - \nabla(f-\mu\Psi)(y), x-y \rangle \ge 0.
    \end{equation*}
    \begin{equation*}
        \Leftrightarrow \langle \nabla f(x) - \nabla f(y), x-y\rangle \ge \mu \langle \nabla \Psi(x) - \nabla \Psi(y), x-y \rangle = \mu \langle A(x-y), x-y \rangle.
    \end{equation*}
\end{proof}

\begin{proposition}[Contraction-diffusion inequality]\label{appprop:contractDiffusion}
    Suppose that $V$ is $\mu$-strongly convex relative to $M^{-1}$, and that $V$ is minimized at $\hat{x} \in \R^d$. Then the following holds for any $\tX$:
    \begin{gather}
         \langle \Delta, M^{-1}(\tX - \hat{\tX}) \rangle_{\text{Frob}} \le \|\tX - \hat{\tX}\|_{2,M} \left(-\frac{\mu}{2} \|\tX - \hat{\tX}\|_{2,M} + \frac{(1+\sqrt{N})}{2T} \|\tX - \frac{1}{N}\tX \mathbf{1}_N \mathbf{1}_N^\top\|_{2,M}\right),
    \end{gather}
    where $\hat\tX$ is the tensorized minimizer
    \begin{equation*}
        \hat{\tX} = \begin{bmatrix}
            \hat{x} & ... & \hat{x}
        \end{bmatrix} = \hat{x} \mathbf{1}_N^\top \in \R^{d \times N}.
    \end{equation*}
\end{proposition}
\begin{proof}
    Let $\Delta = \Delta(\tX^{(k)})$ be the descent direction defined as follows
    \begin{equation}
        \Delta = -\frac{1}{2} M \nabla V(\tX) + \frac{1}{2T}(\tX - \tX \softmax(W)^\top).
    \end{equation}
    such that the PBRWP iteration (\ref{eq:tensorForm}) becomes $\tX^{(k+1)} = \tX^{(k)} - \eta \Delta$.

    Compute, noting that $\hat{\tX}\softmax(W)^\top = \hat{\tX}$ since $\mathbf{1}_N$ is a left-eigenvector of $\softmax(W)^\top$ of eigenvalue 1:
    \begin{align*}
        \langle \Delta, M^{-1}(\tX - \hat{\tX}) \rangle_{\text{Frob}} &= -\frac{1}{2}\langle M \nabla V(\tX) - M \nabla V(\hat{\tX}), M^{-1}(\tX - \hat{\tX})\rangle_{\text{Frob}}\\
        & \quad + \frac{1}{2T}\langle (\tX - \hat{\tX}) (I_{N \times N} -  \softmax(W)^\top), M^{-1}(\tX - \hat{\tX})\rangle_{\text{Frob}} \\
        &= -\frac{1}{2}\sum_{i=1}^N \langle \nabla V(\rvx_i) - \nabla V(\hat{x}), \rvx_i - \hat{x} \rangle \\
        & \quad + \frac{1}{2T}\langle (\tX - \hat{\tX}) (I_{N \times N} -  \softmax(W)^\top), M^{-1}(\tX - \hat{\tX})\rangle_{\text{Frob}} 
    \end{align*}
    where in the second equality, we use that $\langle AX,Y\rangle_{\text{Frob}}=\langle X,A^\top Y\rangle_{\text{Frob}}$. To control the second term we utilize \Cref{prop:SpecStoc} for the column-stochastic matrix $\softmax(W)^\top$. We have, recalling $\hat{\tX} = \hat{x} \mathbf{1}_N^\top$, and $\mathbf{1}_N^\top \softmax(W)^\top = \mathbf{1}_N^\top$,
    \begin{gather*}
        \tX - \hat{\tX} = (\tX - \frac{1}{N} \tX \mathbf{1}_N \mathbf{1}_N^\top) + \frac{1}{N} \tX \mathbf{1}_N \mathbf{1}_N^\top + \hat{x} \mathbf{1}_N^\top \\
        \Rightarrow 
        (\tX - \hat{\tX}) (I_{N \times N} -  \softmax(W)^\top) = (\tX - \frac{1}{N} \tX \mathbf{1}_N \mathbf{1}_N^\top) (I_{N \times N} -  \softmax(W)^\top).
    \end{gather*}
    \begin{align*}
        &\quad \langle (\tX - \hat{\tX}) (I_{N \times N} -  \softmax(W)^\top), M^{-1}(\tX - \hat{\tX})\rangle_{\text{Frob}} \\
        &= \langle M^{-1/2}(\tX - \frac{1}{N} \tX \mathbf{1}_N \mathbf{1}_N^\top) (I_{N \times N} -  \softmax(W)^\top), M^{-1/2}(\tX - \hat{\tX})\rangle_{\text{Frob}} \\
        & \le \|M^{-1/2}(\tX - \frac{1}{N} \tX \mathbf{1}_N \mathbf{1}_N^\top) (I_{N \times N} -  \softmax(W)^\top)\|_2 \|M^{-1/2}(\tX-\hat{\tX})\|_2 \\
        &\le (1+\sqrt{N}) \|M^{-1/2}(\tX - \frac{1}{N}\tX \mathbf{1}_N \mathbf{1}_N^\top)\|_2 \|M^{-1/2} (\tX - \hat{\tX})\|_2
    \end{align*}
    We use \Cref{prop:SpecStoc} to control the spectral norm (induced 2-norm) of $\softmax(W)^\top$. In addition, we use the tracial matrix H\"older inequality $|\langle A,B\rangle_{\text{Frob}}| = |\Tr(A^\top B)| \le \|A\|_p \|B\|_q$ with $p^{-1} + q^{-1}=1$, and $\|A\|_p$ is the $p$-Schatten norm (equal to the Frobenius norm for $p=2$).
    
    From the assumption on $V$ and \Cref{prop:RelStrongConvMono}, we have control on the other term as well:
    \begin{equation*}
        \langle \nabla V(x) - \nabla V(y), x-y \rangle \ge \mu \|x-y\|_M^2.
    \end{equation*}
    Putting everything together, we have
    \begin{gather}
         \langle \Delta, M^{-1}(\tX - \hat{\tX}) \rangle_{\text{Frob}} \\
         \le -\frac{\mu}{2} \sum_{i=1}^N \|\rvx_i-\hat{x}\|_M^2 + \frac{(1+\sqrt{N})}{2T} \|M^{-1/2}(\tX - \frac{1}{N}\tX \mathbf{1}_N \mathbf{1}_N^\top)\|_2 \|M^{-1/2} (\tX - \hat{\tX})\|_2\\
         = \|\tX - \hat{\tX}\|_{2,M} \left(-\frac{\mu}{2} \|\tX - \hat{\tX}\|_{2,M} + \frac{(1+\sqrt{N})}{2T} \|\tX - \frac{1}{N}\tX \mathbf{1}_N \mathbf{1}_N^\top\|_{2,M}\right)
    \end{gather}
    where for a positive definite symmetric matrix $M \in \R^{d\times d}$, we define the scaled $(2,M)$ norm of a $d \times N$ matrix by 
    \begin{equation}
        \|A\|_{2,M}^2 \coloneqq \|M^{-1/2}A\|_2^2 = \Tr(A^\top M^{-1}A).
    \end{equation}
\end{proof}
\textbf{Interpretation: } if the particles are far from the minimizer and have small standard deviation, then they will flow together (on average) to the minimizer. We note that we do not use the actual values of the interaction matrix $W$.

\Cref{fig:lowdim_meanplot} demonstrates the change of regime that occurs from a high variance initial distribution. The target distribution is the two-dimensional standard Gaussian with $M=I$.
\begin{figure}
    \centering
    \subfloat[\centering $N=6, \eta=0.1$]{{\includegraphics[height=3.5cm]{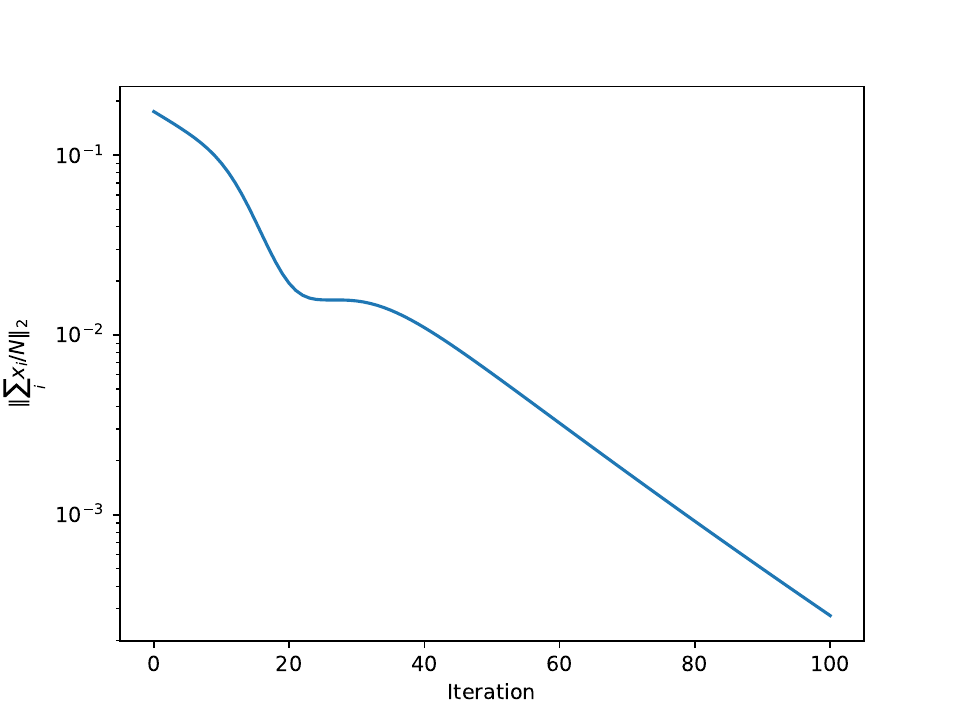}}}
    \subfloat[\centering $N=100, \eta=0.1$]{{\includegraphics[height=3.5cm]{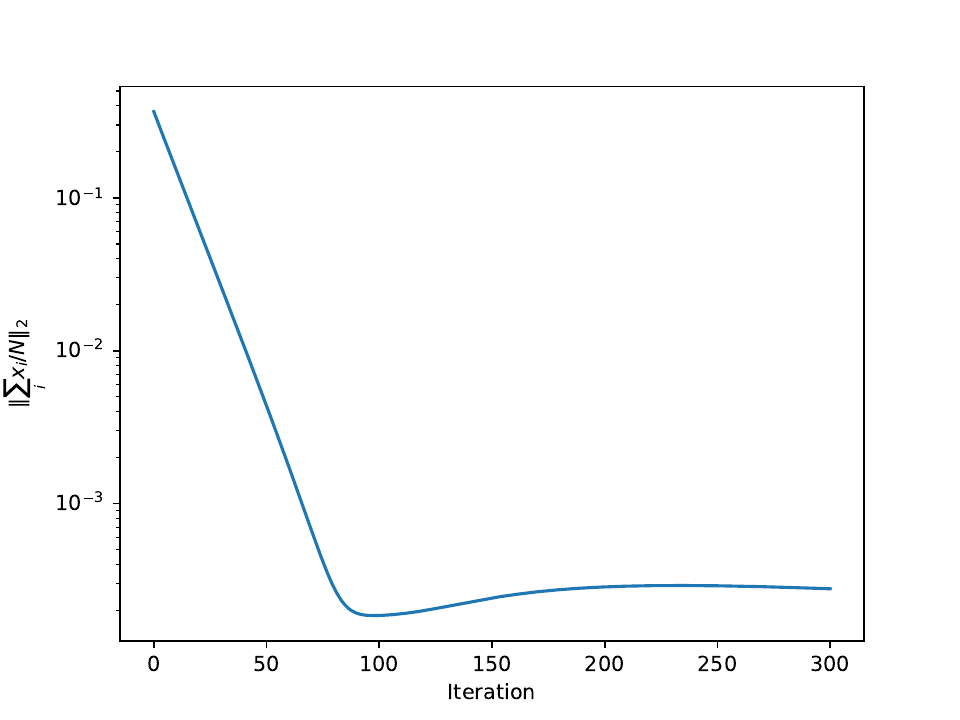}}}
    \subfloat[\centering $N=100, \eta=0.02$]{{\includegraphics[height=3.5cm]{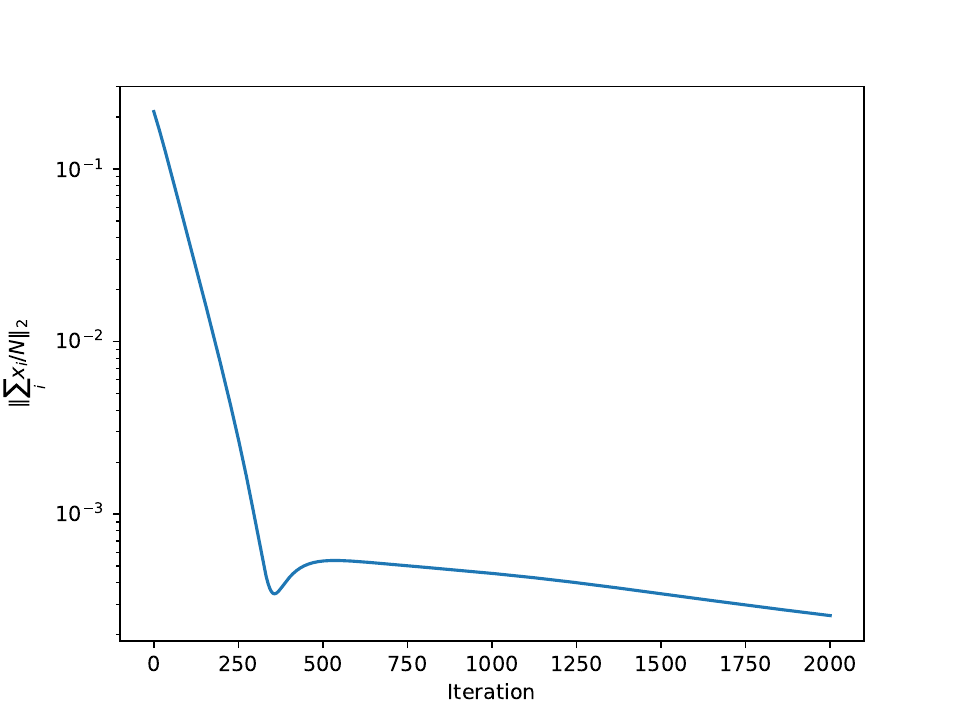}}}
    \caption{Plot of the norm of the average of the points $\|\frac{1}{N}\sum_i \rvx_i\|$ with respect to iteration, for the target distribution $\gN(0, I_2)$ and identity preconditioner, $T=2$, and various choices of $N$ and step-size. The mean does not converge linearly, and instead has two separate regimes. }
    \label{fig:lowdim_meanplot}
\end{figure}

\subsection{Particle diffusion boundedness}
We refer back to the Gaussian case where the target distribution is $\gN(0, \beta^{-1}\Sigma)$.

\begin{proposition}
    Suppose $M = \Sigma$ (and $T<1$), and that $\rvx_1$ is an exterior point of the convex hull of $\{\rvx_i\}$. Suppose $\delta>0$ is such that for all $j \ne 1$, 
    \begin{equation}
    \rvx_1^\top M^{-1} \rvx_j \le \rvx_1^\top M^{-1} \rvx_1 - \delta \|\rvx_1\|_M,
    \end{equation}
    Then, if $\rvx_1$ has sufficiently large norm, in particular assuming,
    \begin{equation}
        \delta \|\rvx_1\|_M \ge 2\beta^{-1} T\log (\frac{2(N-1)}{T}),
    \end{equation}
    then for small step-size, the updated point satisfies $\|\rvx_1^{(k+1)}\|_M < \|\rvx_1^{(k)}\|_M$. 
\end{proposition}
\begin{proof}
Recall the particle evolution
\begin{align*}
    \rvx_i^{(k+1)} = \rvx_i^{(k)} - \frac{\eta}{2} M \nabla V(\rvx_i^{(k)}) + \frac{\eta}{2T}\left( \sum_{k=1}^N \softmax(W_{i,\cdot}^{(k)})_j (\rvx_i^{(k)}-\rvx_j^{(k)})\right),
\end{align*}
where in the Gaussian case, the interaction weight matrix $W$ takes the special form
\begin{equation}
    W_{ij}=\frac{\beta}{4 T} \left(2 \rvx_i^\top M^{-1} \rvx_j - \rvx_j^\top (M + T M \Sigma^{-1} M)^{-1} \rvx_j\right).
\end{equation}

We control the values of $W_{1j}$. In this case $M=\Sigma$ we have
\begin{equation*}
    W_{ij}=\frac{\beta}{4 T} \left(2 \rvx_i^\top M^{-1} \rvx_j - (1+T)^{-1}\rvx_j^\top M^{-1} \rvx_j\right)
\end{equation*}
which controls
\begin{equation*}
    W_{11} = \frac{\beta}{4T} (2-(1+T)^{-1}) \|\rvx_1\|_M^2 
\end{equation*}
and for $j \ne 1$,
\begin{align}
    W_{1j} &\le \frac{\beta}{4 T} (2\|\rvx_1\|_M^2  - (1+T)^{-1} \|\rvx_1\|_M^2 - 2\delta\|\rvx_1\|_M)\\
    &= W_{11} - \frac{\delta\beta}{2T} \|\rvx_1\|_M.
\end{align}
We have that
\begin{equation}
    \sum_{i=1}^N \exp(W_{1j}) \le \exp(W_{11}) [1+(N-1)\exp(-\delta \beta\|\rvx_1\|_M/(2T))]
\end{equation}
Considering $\langle \rvx_1^{(k+1)} - \rvx_1^{(k)}, \rvx_1^{(k)}\rangle_M$, we want to show it is negative for big $R$. Then the norm of the iteration decreases. The iteration takes the form

\begin{gather}
    \frac{1}{\eta}\langle \rvx_1^{(k+1)} - \rvx_1^{(k)}, \rvx_1^{(k)}\rangle_M\\
    = -\frac{1}{2} \|\rvx_1\|_M^2 + \frac{1}{2T}\sum_{j=2}^N \softmax (W_{1\cdot})_j \langle \rvx_1 - \rvx_j, \rvx_1\rangle_M
\end{gather}

By maximality we have
\begin{equation}
0<\langle \rvx_1 - \rvx_j, \rvx_1\rangle_M <
\begin{cases}
     \|\rvx_1\|_M^2- \delta^2/2, & \langle \rvx_1,\rvx_j\rangle_M  > 0;\\
    2\|\rvx_1\|_M^2, & \langle \rvx_1,\rvx_j\rangle_M < 0.
\end{cases}
\end{equation}

If $\sum_{j=2}^N \softmax(W_{1\cdot})_j \le \frac{T}{2}$, then naturally 
\begin{gather*}
    -\frac{1}{2} \|\rvx_1\|_M^2 + \frac{1}{2T}\sum_{j=2}^N \softmax (W_{1\cdot})_j \langle \rvx_1 - \rvx_j, \rvx_1\rangle_M\\
    < -\frac{1}{2}\|\rvx_1\|_M^2 + \frac{1}{2T} \frac{T}{2} 2\|\rvx_1\|_M^2 = 0.
\end{gather*}

So we wish to find a sufficient condition on $\|\rvx_1\|_M$ such that $\sum_{j=2}^N \softmax(W_{1\cdot})_j \le \frac{T}{2}$. It is sufficient that $\softmax(W_{1\cdot})_1 \ge 1-T/2$, which is satisfied if 
\begin{gather}
    [1+(N-1)\exp(-\delta \beta\|\rvx_1\|_M/(2 T))]^{-1} \ge 1-\frac{T}{2}\\
    \Leftrightarrow [1+\exp(-\delta\beta \|\rvx_1\|_M/(2T) + 
    \log(N-1))]^{-1} \ge 1-\frac{T}{2}
\end{gather}
Note $(1+x)^{-1}>1-x$, so it is sufficient that
\begin{gather*}
    \exp(-\delta \beta\|\rvx_1\|_M/(2 T) + \log(N-1) ) \le \frac{T}{2}\\
    \Leftrightarrow \delta \|\rvx_1\|_M \ge 2\beta^{-1} T\log (\frac{2 (N-1)}{T}).
\end{gather*}
\end{proof}

\subsection{Change of variables for different $\beta$}\label{appssec:BetaCOV}
In this section, we show that the diffusion parameter $\beta$ in PBRWP can be chosen to be $\beta=1$ without loss of generality, by changing the step-size, potential, and regularization parameter $T$. In particular, PBRWP with potential $V$, step-size $\eta$, regularization parameter $T$ and diffusion $\beta$ is equivalent to PBRWP with potential $\hat{V} = \beta V$, step-size $\hat\eta =  \eta \beta^{-1}$, and parameters $\hat{T} = T\beta^{-1}, \hat\beta = 1$.

This can be seen by looking at the algorithm of PBRWP. For a step-size $\eta>0$, recall that the iteration can be written as 
\begin{subequations}\label{appeq:PBRWPIter}
\begin{align}
    \tX^{(k+1)} &= \tX^{(k)} - \frac{\eta}{2} M \nabla V(\tX^{(k)}) + \frac{\eta}{2T}\left(\tX^{(k)} - \tX^{(k)} \softmax(W^{(k)})^\top\right),\\
    W_{i,j} &= -\beta \frac{\|\rvx_i-\rvx_j\|_M^2}{4T} - \log \gZ(\rvx_j),\\
    \gZ(\rvx_j) &=\int_{\R^d} e^{-\frac{\beta}{2} (V(z) + \frac{\|z-\rvx_j\|^2_M}{2T})} \dd{z}.
\end{align}
\end{subequations}
The interaction matrix and normalizing constants do not change under the change of variables $(V, \beta, T) \mapsto (\beta V, 1, T\beta^{-1})$. To keep the iterations \cref{appeq:PBRWPIter} the same under this change of variables, one additionally takes $\eta \mapsto \eta\beta^{-1} $. This can be seen by rewriting the update \cref{appeq:PBRWPIter} as follows:

\begin{subequations}
\begin{align}
    \tX^{(k+1)} &= \tX^{(k)} - \frac{\eta\beta^{-1} }{2} M \nabla (\beta V)(\tX^{(k)}) + \frac{\eta\beta^{-1} }{2T\beta^{-1} }\left(\tX^{(k)} - \tX^{(k)} \softmax(W^{(k)})^\top\right),\\
    W_{i,j} &= - \frac{\|\rvx_i-\rvx_j\|_M^2}{4T\beta^{-1}} - \log \gZ(\rvx_j),\\
    \gZ(\rvx_j) &=\int_{\R^d} e^{-\frac{1}{2} (\beta V(z) + \frac{\|z-\rvx_j\|^2_M}{2T\beta^{-1} })} \dd{z}.
\end{align}
\end{subequations}

\section{Sensitivity to hyperparameters for low-dimensional experiments}
We give some KL divergence plots varying the sensitivity of PBRWP with respect to the hyperparameter $T$, for the bimodal and annulus targets. Both experiments use the Monte Carlo approximation to the normalizing constant and the initial random points have a fixed seed. The KL divergence is computed as in the main text.

\cref{appfig:bananaAblation,appfig:annulusAblation} show that the convergence is relatively stable, with larger step-sizes incurring a standard oscillation phenomenon, and the bias changing across different choices of $T$.
\begin{figure}[ht]
    \centering
    \includegraphics[width=0.7\linewidth,trim={0 0 0 2.1cm},clip]{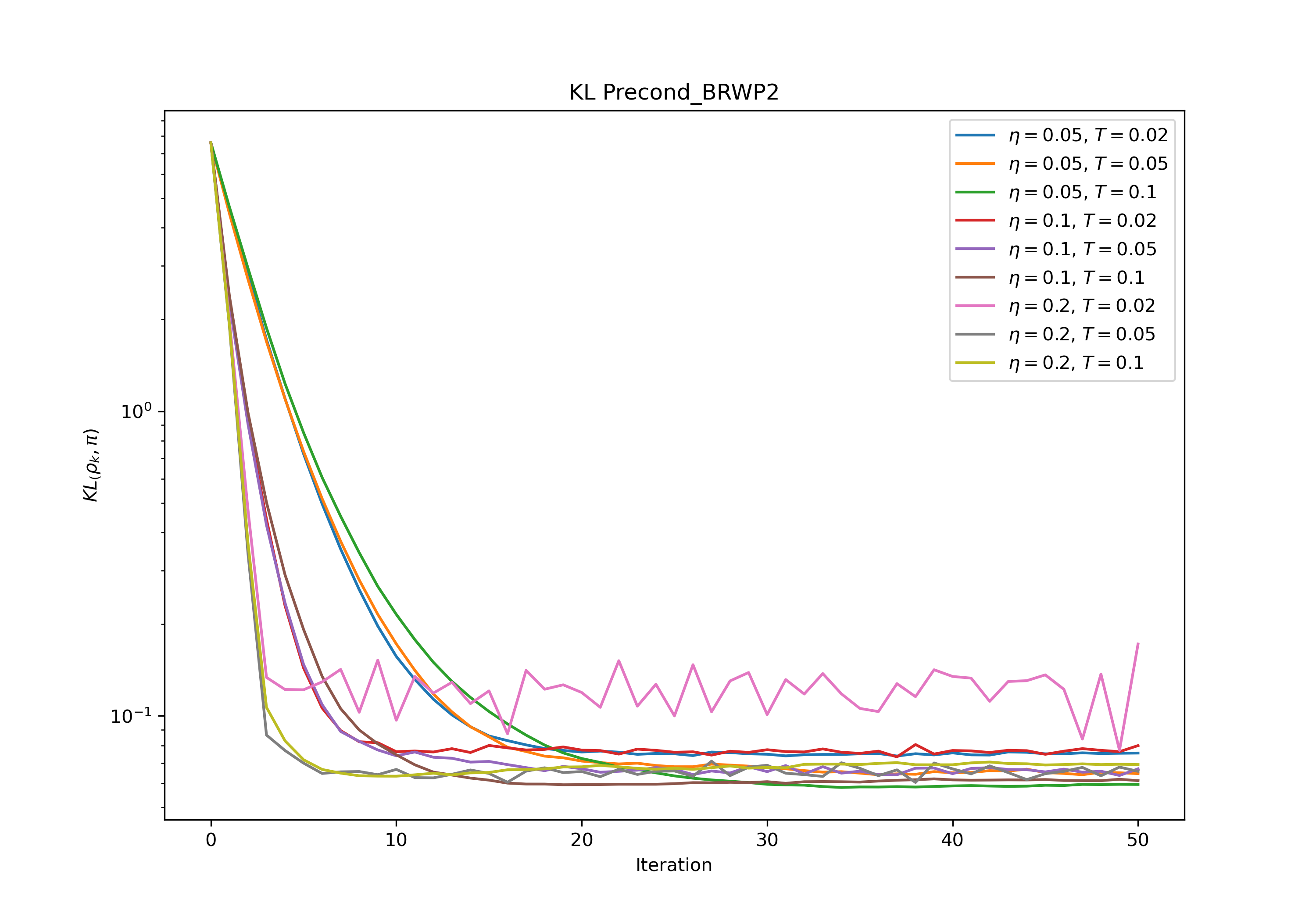}
    \caption{KL convergence curves for the bimodal experiment with large initial variance, across step-sizes $\eta \in  \{0.05,0.1,0.2\}$ and regularization $T \in \{0.02,0.05,0.1\}$. The convergence is quite stable across these parameters, with oscillations occurring for large step-size and small $T$.}
    \label{appfig:bananaAblation}
\end{figure}

\begin{figure}[ht]
    \centering
    \includegraphics[width=0.7\linewidth,trim={0 0 0 2.1cm},clip]{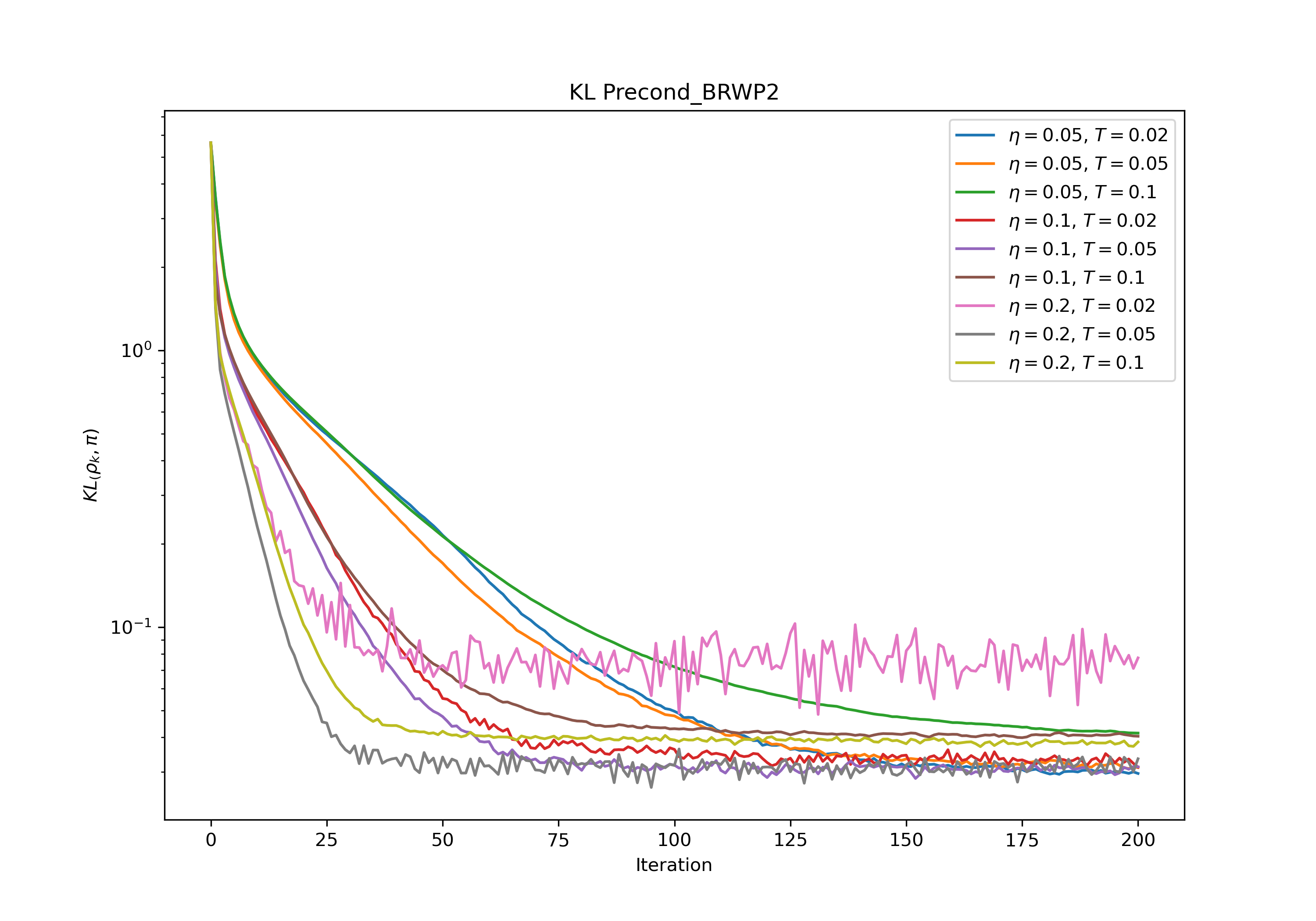}
    \caption{KL convergence curves for the annulus experiment, across step-sizes $\eta \in  \{0.05,0.1,0.2\}$ and regularization $T \in \{0.02,0.05,0.1\}$. For larger step-size, increasing $T$ makes the convergence more stable but more biased.}
    \label{appfig:annulusAblation}
\end{figure}

\section{Sensitivity to hyperparameters for imaging experiment}
\Cref{app:ablPBRWP,app:ablBRWP} plot the deconvolution reconstructions using PBRWP and BRWP respectively over a variety of choices of $T$ and step-size. Both methods are run with 40 particles up to 5000 iterations. We observe that both methods are not particularly sensitive to the choice of $T$ except for an intermediate choice of $T=10^{-2}$, for which we observe more variance. This can be explained by variance reduction phenomena: by the regularization of the Wasserstein proximal for large $T$, and by the finite number of particles for small $T$. For the smallest step-size $\eta=10^{-5}$, 5000 iterations are insufficient to reach convergence and therefore the performance is worse.

\begin{table}[htp]
\centering
\caption{PSNR (in dB) of the mean of 40 particles computed using \textbf{PBRWP}, taken with respect the MAP estimate for the deconvolution problem. Taken at iteration 5000, at which stepsizes 1e-3 and 5e-4 have converged.}\label{app:ablPBRWP}
\begin{tabular}{l|ccccc}
\toprule
\diagbox{$T$}{Step-size} & 1e-3 & 5e-4 & 2e-4 & 1e-4 & 1e-5 \\
\midrule
1e-4 & 30.28& 34.22 & 39.53& 43.34& 31.30\\
1e-3 & 30.28& 34.24 & 39.45 &42.76 &31.29\\
1e-2 &30.30 & 33.95 & 37.88 & 39.37&30.97\\
0.5 & 30.29& 34.22 & 39.53 & 43.37&31.30\\
10 & 30.28& 34.22 & 39.52 & 43.37&31.30\\
\bottomrule
\end{tabular}
\end{table}

\begin{table}[htp]
\centering
\caption{PSNR (in dB) of the mean of 40 particles computed using \textbf{BRWP}, taken with respect the MAP estimate for the deconvolution problem. Taken at iteration 5000, at which stepsizes 1e-3, 5e-4 and 2e-4 have converged.}\label{app:ablBRWP}
\begin{tabular}{l|ccccc}
\toprule
\diagbox{$T$}{Step-size} & 1e-3 & 5e-4 & 2e-4 & 1e-4 & 1e-5 \\
\midrule
1e-4 & 34.22& 38.20& 43.36& 42.75& 28.59\\
1e-3 &34.22 & 38.20 & 43.08& 42.47&28.59\\
1e-2 &34.18 &37.69 & 40.82& 39.99&28.53 \\
0.5 &34.22 &38.20 &43.36 &42.72 &28.58\\
10 & 34.22&38.20 &43.37 & 42.76&28.59\\
\bottomrule
\end{tabular}
\end{table}

\section{Computational cost}
We report the computational time per iteration of updating $N$ particles for a low-dimensional example and the deconvolution example in \Cref{tab:timingBimodal,tab:timingDeconv} respectively. 
    
We see that the per-iteration time of the non-interacting methods scale roughly linearly with the number of particles, while the interacting methods scale roughly quadratically. In the bimodal experiment \Cref{tab:timingBimodal}, (P)BRWP uses the Monte Carlo integration to approximate the normalizing constant. This leads to the slower computation for fewer particles but this overhead disappears for more particles. In the high dimensional deconvolution experiment using 40 particles, the interacting particle systems are slower than ULA and MALA but significantly faster than MYULA. This is because the proximal step on the total variation term is computationally expensive.

\begin{table}[ht]
    \centering
    \begin{tabular}{c|cccccc}
         \toprule
         \# Particles& ULA & MALA & MLA & SVGD & BRWP & PBRWP \\ \midrule
           50 & 21700 & 5410 & 7920 & 5261 & 412 & 407 \\
           200 & 14800 & 3800 & 6350 & 358 & 189 & 205 \\
          2000  & 5320 & 1490 & 2920 & 3.94 & 3.95 & 3.98 \\ \bottomrule
    \end{tabular}
    \caption{Iterations per second for the sampling algorithms for the 2D bimodal distribution, run with different numbers of particles. BRWP and PBRWP are slower at the low iteration count due to the Monte Carlo integration to approximate the normalizing constant, which was done with 25 samples.}
    \label{tab:timingBimodal}
\end{table}

\begin{table}[ht]
    \centering
    \begin{tabular}{cccccc}
         \toprule
         ULA & MYULA & SVGD & BRWP & MLA & PBRWP \\ \midrule
         89.0 & 1.14 & 14.6 & 21.7 & 53.0 & 11.1 \\ \bottomrule
    \end{tabular}
    \caption{Iterations per second for the sampling algorithms for the deconvolution experiment, using the Laplace approximation. BRWP and PBRWP are slower at the low iteration count due to the Monte Carlo integration to approximate the normalizing constant, which was done with 25 samples. MYULA is very slow due to needing to approximate the proximal operator of total variation at each step. SVGD is slower than BRWP since the kernel is applied to the gradients instead of just the particle positions.}
\label{tab:timingDeconv}
\end{table}

\section{Choice of variable preconditioner for BNNs}\label{appsec:AdamPC}
To compute the diagonal variable preconditioner, we use the second moment estimates of Adam \cite{kingma2014adam}, which are used to precondition a moving average of the gradient. We discard the first moment estimate since we are not optimizing using Adam directly. \Cref{appalg:AdamPrecond} details how pointwise gradient estimates (corresponding to each particle) give rise to the variable preconditioning matrices $M^{(k)}$. By applying this algorithm to each particle $\rvx_j$, which are updated according to PBRWP, we get their corresponding preconditioning matrices $M^{(k)}_j$.

\SetKwComment{Comment}{/* }{ */}

\begin{algorithm2e}
\caption{Adam-based Preconditioner}\label{appalg:AdamPrecond}
\KwData{Objective function $f$, exponential decay rates $\beta_2 = 0.999$, point sequence $(x^{(l)})_{l\ge1}$, epsilon $\epsilon=0.001$.}
\KwResult{Preconditioners $M^{(k)}$, where $M^{(k)} = M^{(k)}(x^{(1)},...,x^{(k)})$.}
$v_0 \gets 0$\tcp*{initialize second moment vector}
\For{$k=1,...,K$}
{
$g_k \gets \nabla f(x^{(k)})$ \tcp*{compute gradient}
$v_k \gets \beta_2 v_{k-1} + (1-\beta_2) g_k^2$ \tcp*{update second moment estimate}
$\hat v_k \gets v_k/(1-\beta_2^k)$  \tcp*{bias correction}
$M^{(k)} = \diag(1/(\sqrt{\hat v_k}+\epsilon))$ \tcp*{construct preconditioning matrix}
}
\Return $(M^{(k)})_{k =1}^K$.
\end{algorithm2e}
\end{document}

%% file: ex_shared.tex
\usepackage{lipsum}
\usepackage{amsfonts}
\usepackage{epstopdf}
\usepackage{algorithmic}

\usepackage[shortlabels]{enumitem}
\usepackage{microtype}
\usepackage[caption=false]{subfig}
\usepackage{booktabs} 
\usepackage{hyperref}
\usepackage{float}
\usepackage{url}
\usepackage{amsmath}
\usepackage{amssymb}
\usepackage{mathtools}
\usepackage{physics} 
\usepackage{xcolor}
\usepackage{bbm}
\usepackage{soul}
\usepackage{multirow}
\usepackage{adjustbox}
\usepackage[algo2e,ruled,linesnumbered]{algorithm2e}
\usepackage[capitalize]{cleveref}
\usepackage{cases} 
\usepackage{tabularray}
\usepackage{diagbox}

\ifpdf
  \DeclareGraphicsExtensions{.eps,.pdf,.png,.jpg}
\else
  \DeclareGraphicsExtensions{.eps}
\fi

\usepackage{enumitem}
\setlist[enumerate]{leftmargin=.5in}
\setlist[itemize]{leftmargin=.5in}

\newsiamremark{remark}{Remark}
\newsiamremark{hypothesis}{Hypothesis}
\crefname{hypothesis}{Hypothesis}{Hypotheses}
\newsiamthm{claim}{Claim}
\newsiamremark{fact}{Fact}
\crefname{fact}{Fact}{Facts}

\headers{Preconditioned Regularized Wasserstein Proximal Sampling Methods}{H. Y. Tan, S. Osher, and W. Li}

\title{Preconditioned Regularized Wasserstein Proximal Sampling Methods}

\author{Hong Ye Tan\thanks{Department of Mathematics, University of California, Los Angeles, 90095
  (\email{hyt35,sjo@math.ucla.edu}).}
\and Stanley Osher\footnotemark[2]
\and Wuchen Li\thanks{Department of Mathematics, University of South Carolina, Columbia, SC 29208 (\email{wuchen@mailbox.sc.edu})}}

\usepackage{amsopn}

\input{math_commands}

\crefname{algocf}{alg.}{algs.}
\Crefname{algocf}{Algorithm}{Algorithms}

%% file: math_commands.tex
\usepackage{amsmath,amsfonts,bm}

\def\eqref#1{equation~\ref{#1}}

\def\1{\bm{1}}

\def\rvx{{\mathbf{x}}}

\def\rvz{{\mathbf{z}}}

\DeclareMathAlphabet{\mathsfit}{\encodingdefault}{\sfdefault}{m}{sl}
\SetMathAlphabet{\mathsfit}{bold}{\encodingdefault}{\sfdefault}{bx}{n}
\newcommand{\tens}[1]{\bm{\mathsfit{#1}}}

\def\tM{{\tens{M}}}

\def\tX{{\tens{X}}}

\def\gC{{\mathcal{C}}}

\def\gF{{\mathcal{F}}}

\def\gI{{\mathcal{I}}}

\def\gL{{\mathcal{L}}}

\def\gN{{\mathcal{N}}}

\def\gP{{\mathcal{P}}}

\def\gW{{\mathcal{W}}}

\def\gZ{{\mathcal{Z}}}

\newcommand{\R}{\mathbb{R}}

\newcommand{\softmax}{\mathrm{softmax}}

\newcommand{\KL}{\mathrm{D}_{\mathrm{KL}}}

\newcommand{\Cov}{\mathrm{Cov}}

\DeclareMathOperator*{\argmin}{arg\,min}

\DeclareMathOperator{\prox}{prox}
\DeclareMathOperator{\wprox}{WProx}

\DeclareMathOperator{\diag}{diag}

\DeclareMathOperator{\TV}{TV}

%% file: refs.bib
@article{li2023kernel,
  title={A kernel formula for regularized {W}asserstein proximal operators},
  author={Li, Wuchen and Liu, Siting and Osher, Stanley},
  journal={Research in the Mathematical Sciences},
  volume={10},
  number={4},
  pages={43},
  year={2023},
  publisher={Springer}
}

@article{jiang2021mirror,
  title={Mirror {L}angevin {M}onte {C}arlo: the case under isoperimetry},
  author={Jiang, Qijia},
  journal={Advances in Neural Information Processing Systems},
  volume={34},
  pages={715--725},
  year={2021}
}

@article{tan2024noise,
  title={Noise-free sampling algorithms via regularized {W}asserstein proximals},
  author={Tan, Hong Ye and Osher, Stanley and Li, Wuchen},
  journal={Research in the Mathematical Sciences},
  volume={11},
  number={4},
  pages={65},
  year={2024},
  publisher={Springer}
}

@article{tan2023data,
  title={Data-driven mirror descent with input-convex neural networks},
  author={Tan, Hong Ye and Mukherjee, Subhadip and Tang, Junqi and Sch{\"o}nlieb, Carola-Bibiane},
  journal={SIAM Journal on Mathematics of Data Science},
  volume={5},
  number={2},
  pages={558--587},
  year={2023},
  publisher={SIAM}
}

@article{petersen2008matrix,
  title={The matrix cookbook},
  author={Petersen, Kaare Brandt and Pedersen, Michael Syskind and others},
  journal={Technical University of Denmark},
  volume={7},
  number={15},
  pages={510},
  year={2008}
}

@article{han2025splitting,
  title={Splitting Regularized {W}asserstein Proximal Algorithms for Nonsmooth Sampling Problems},
  author={Han, Fuqun and Osher, Stanley and Li, Wuchen},
  journal={arXiv preprint arXiv:2502.16773},
  year={2025}
}

@article{wang2022accelerated,
  title={Accelerated information gradient flow},
  author={Wang, Yifei and Li, Wuchen},
  journal={Journal of Scientific Computing},
  volume={90},
  pages={1--47},
  year={2022},
  publisher={Springer}
}

@book{santambrogio2015optimal,
  title={Optimal transport for applied mathematicians},
  author={Santambrogio, Filippo},
  volume={87},
  year={2015},
  publisher={Springer}
}

@article{afonso2010fast,
  title={Fast image recovery using variable splitting and constrained optimization},
  author={Afonso, Manya V and Bioucas-Dias, Jos{\'e} M and Figueiredo, M{\'a}rio AT},
  journal={IEEE transactions on image processing},
  volume={19},
  number={9},
  pages={2345--2356},
  year={2010},
  publisher={IEEE}
}

@article{han2025tensor,
  title={Tensor train based sampling algorithms for approximating regularized {W}asserstein proximal operators},
  author={Han, Fuqun and Osher, Stanley and Li, Wuchen},
  journal={SIAM/ASA Journal on Uncertainty Quantification},
  volume={13},
  number={2},
  pages={775--804},
  year={2025},
  publisher={SIAM}
}

@article{han2024convergence,
  title={Convergence of noise-free sampling algorithms with regularized {W}asserstein proximals},
  author={Han, Fuqun and Osher, Stanley and Li, Wuchen},
  journal={arXiv preprint arXiv:2409.01567},
  year={2024}
}

@article{vaswani2017attention,
  title={Attention is all you need},
  author={Vaswani, Ashish and Shazeer, Noam and Parmar, Niki and Uszkoreit, Jakob and Jones, Llion and Gomez, Aidan N and Kaiser, {\L}ukasz and Polosukhin, Illia},
  journal={Advances in neural information processing systems},
  volume={30},
  year={2017}
}

@book{bleistein1975asymptotic,
  title={Asymptotic expansions of integrals},
  author={Bleistein, Norman and Handelsman, Richard A},
  year={1975},
  publisher={Ardent Media}
}

@article{tibshirani2025laplace,
  title={Laplace Meets {M}oreau: Smooth Approximation to Infimal Convolutions Using {L}aplace's Method},
  author={Tibshirani, Ryan J and Fung, Samy Wu and Heaton, Howard and Osher, Stanley},
  journal={Journal of Machine Learning Research},
  volume={26},
  number={72},
  pages={1--36},
  year={2025}
}

@inproceedings{sander2022sinkformers,
  title={Sinkformers: Transformers with doubly stochastic attention},
  author={Sander, Michael E and Ablin, Pierre and Blondel, Mathieu and Peyr{\'e}, Gabriel},
  booktitle={International Conference on Artificial Intelligence and Statistics},
  pages={3515--3530},
  year={2022},
  organization={PMLR}
}

@article{habring2024subgradient,
  title={Subgradient {L}angevin methods for sampling from nonsmooth potentials},
  author={Habring, Andreas and Holler, Martin and Pock, Thomas},
  journal={SIAM Journal on Mathematics of Data Science},
  volume={6},
  number={4},
  pages={897--925},
  year={2024},
  publisher={SIAM}
}

@article{kingma2014adam,
  title={Adam: A method for stochastic optimization},
  author={Kingma, Diederik P and Ba, Jimmy},
  journal={arXiv preprint arXiv:1412.6980},
  year={2014}
}

@article{durmus2018efficient,
  title={Efficient {B}ayesian computation by proximal {M}arkov chain {M}onte {C}arlo: when {L}angevin meets {M}oreau},
  author={Durmus, Alain and Moulines, Eric and Pereyra, Marcelo},
  journal={SIAM Journal on Imaging Sciences},
  volume={11},
  number={1},
  pages={473--506},
  year={2018},
  publisher={SIAM}
}

@article{tachella2025deepinverse,
  title={DeepInverse: A Python package for solving imaging inverse problems with deep learning},
  author={Tachella, Juli{\'a}n and Terris, Matthieu and Hurault, Samuel and Wang, Andrew and Chen, Dongdong and Nguyen, Minh-Hai and Song, Maxime and Davies, Thomas and Davy, Leo and Dong, Jonathan and others},
  journal={arXiv preprint arXiv:2505.20160},
  year={2025}
}

@article{condat2013primal,
  title={A primal--dual splitting method for convex optimization involving {L}ipschitzian, proximable and linear composite terms},
  author={Condat, Laurent},
  journal={Journal of optimization theory and applications},
  volume={158},
  number={2},
  pages={460--479},
  year={2013},
  publisher={Springer}
}

@article{roberts1996exponential,
  title={Exponential Convergence of {L}angevin Distributions and Their Discrete Approximations},
  author={Roberts, Gareth O and Tweedie, Richard L},
  journal={Bernoulli},
  pages={341--363},
  year={1996},
  publisher={JSTOR}
}

@incollection{risken1989fokker,
  title={Fokker-{P}lanck equation},
  author={Risken, Hannes},
  booktitle={The {F}okker-{P}lanck equation: methods of solution and applications},
  pages={63--95},
  year={1989},
  publisher={Springer}
}

@article{kubo1963stochastic,
  title={Stochastic {L}iouville equations},
  author={Kubo, Ryogo},
  journal={Journal of Mathematical Physics},
  volume={4},
  number={2},
  pages={174--183},
  year={1963},
  publisher={American Institute of Physics}
}

@inproceedings{nijkamp2022mcmc,
  title={{MCMC} should mix: learning energy-based model with neural transport latent space {MCMC}.},
  author={Erik Nijkamp and Ruiqi Gao and Pavel Sountsov and Srinivas Vasudevan and Bo Pang and Song-Chun Zhu and Ying Nian Wu},
  booktitle={International Conference on Learning Representations (ICLR 2022).},
  year={2022}
}

@article{srivastava2017veegan,
  title={Veegan: Reducing mode collapse in {GAN}s using implicit variational learning},
  author={Srivastava, Akash and Valkov, Lazar and Russell, Chris and Gutmann, Michael U and Sutton, Charles},
  journal={Advances in neural information processing systems},
  volume={30},
  year={2017}
}

@article{li2023reducing,
  title={Reducing Mode Collapse With {M}onge--{K}antorovich Optimal Transport for Generative Adversarial Networks},
  author={Li, Wei and Liu, Wei and Chen, Jinlin and Wu, Libing and Flynn, Patrick D and Ding, Wei and Chen, Ping},
  journal={IEEE Transactions on Cybernetics},
  year={2023},
  publisher={IEEE}
}

@article{benamou2000computational,
  title={A computational fluid mechanics solution to the {M}onge-{K}antorovich mass transfer problem},
  author={Benamou, Jean-David and Brenier, Yann},
  journal={Numerische Mathematik},
  volume={84},
  number={3},
  pages={375--393},
  year={2000},
  publisher={Springer-Verlag Berlin/Heidelberg}
}

@article{laumont2022bayesian,
  title={Bayesian imaging using plug \& play priors: when {L}angevin meets {T}weedie},
  author={Laumont, R{\'e}mi and Bortoli, Valentin De and Almansa, Andr{\'e}s and Delon, Julie and Durmus, Alain and Pereyra, Marcelo},
  journal={SIAM Journal on Imaging Sciences},
  volume={15},
  number={2},
  pages={701--737},
  year={2022},
  publisher={SIAM}
}

@article{bond2021deep,
  title={Deep generative modelling: A comparative review of {VAE}s, {GAN}s, normalizing flows, energy-based and autoregressive models},
  author={Bond-Taylor, Sam and Leach, Adam and Long, Yang and Willcocks, Chris G},
  journal={IEEE transactions on pattern analysis and machine intelligence},
  year={2021},
  publisher={IEEE}
}

@article{carrillo2019blob,
  title={A blob method for diffusion},
  author={Carrillo, Jos{\'e} Antonio and Craig, Katy and Patacchini, Francesco S},
  journal={Calculus of Variations and Partial Differential Equations},
  volume={58},
  pages={1--53},
  year={2019},
  publisher={Springer}
}

@book{wand1994kernel,
  title={Kernel smoothing},
  author={Wand, Matt P and Jones, M Chris},
  year={1994},
  publisher={CRC press}
}

@article{van2003adaptive,
  title={Adaptive kernel density estimation},
  author={Van Kerm, Philippe},
  journal={The Stata Journal},
  volume={3},
  number={2},
  pages={148--156},
  year={2003},
  publisher={SAGE Publications Sage CA: Los Angeles, CA}
}

@article{botev2010kernel,
  title={Kernel density estimation via diffusion},
  author={Botev, Zdravko I and Grotowski, Joseph F and Kroese, Dirk P},
  journal={Annals of Statistics},
  volume={38},
  number={5},
  pages={2916--2957},
  year={2010},
  publisher={Institute of Mathematical Statistics}
}

@article{chen2018neural,
  title={Neural ordinary differential equations},
  author={Chen, Ricky TQ and Rubanova, Yulia and Bettencourt, Jesse and Duvenaud, David K},
  journal={Advances in neural information processing systems},
  volume={31},
  year={2018}
}

@book{gramacki2018nonparametric,
  title={Nonparametric kernel density estimation and its computational aspects},
  author={Gramacki, Artur},
  volume={37},
  year={2018},
  publisher={Springer}
}

@inproceedings{hershey2007approximating,
  title={Approximating the {K}ullback {L}eibler divergence between {G}aussian mixture models},
  author={Hershey, John R and Olsen, Peder A},
  booktitle={2007 IEEE International Conference on Acoustics, Speech and Signal Processing-ICASSP'07},
  volume={4},
  pages={IV--317},
  year={2007},
  organization={IEEE}
}

@article{rudin1992nonlinear,
  title={Nonlinear total variation based noise removal algorithms},
  author={Rudin, Leonid I and Osher, Stanley and Fatemi, Emad},
  journal={Physica D: nonlinear phenomena},
  volume={60},
  number={1-4},
  pages={259--268},
  year={1992},
  publisher={Elsevier}
}

@article{hsieh2018mirrored,
  title={Mirrored {L}angevin dynamics},
  author={Hsieh, Ya-Ping and Kavis, Ali and Rolland, Paul and Cevher, Volkan},
  journal={Advances in Neural Information Processing Systems},
  volume={31},
  year={2018}
}

@inproceedings{
  song2021scorebased,
  title={Score-Based Generative Modeling through Stochastic Differential Equations},
  author={Yang Song and Jascha Sohl-Dickstein and Diederik P Kingma and Abhishek Kumar and Stefano Ermon and Ben Poole},
  booktitle={International Conference on Learning Representations},
  year={2021}
}

@article{castin2025unified,
  title={A unified perspective on the dynamics of deep transformers},
  author={Castin, Val{\'e}rie and Ablin, Pierre and Carrillo, Jos{\'e} Antonio and Peyr{\'e}, Gabriel},
  journal={arXiv preprint arXiv:2501.18322},
  year={2025}
}

@article{tan2024unsupervised,
  title={Unsupervised Training of Convex Regularizers using Maximum Likelihood Estimation},
  author={Tan, Hong Ye and Cai, Ziruo and Pereyra, Marcelo and Mukherjee, Subhadip and Tang, Junqi and Sch{\"o}nlieb, Carola-Bibiane},
  journal={Transactions on Machine Learning Research},
    year={2024}
}

@book{gelman1995bayesian,
  title={Bayesian data analysis},
  author={Gelman, Andrew and Carlin, John B and Stern, Hal S and Rubin, Donald B},
  year={1995},
  publisher={Chapman and Hall/CRC}
}

@article{stuart2010inverse,
  title={Inverse problems: a {B}ayesian perspective},
  author={Stuart, Andrew M},
  journal={Acta numerica},
  volume={19},
  pages={451--559},
  year={2010},
  publisher={Cambridge University Press}
}

@article{geshkovski2025mathematical,
  title={A mathematical perspective on transformers},
  author={Geshkovski, Borjan and Letrouit, Cyril and Polyanskiy, Yury and Rigollet, Philippe},
  journal={Bulletin of the American Mathematical Society},
  volume={62},
  number={3},
  pages={427--479},
  year={2025}
}

@book{krauth2006statistical,
  title={Statistical mechanics: algorithms and computations},
  author={Krauth, Werner},
  volume={13},
  year={2006},
  publisher={OUP Oxford}
}

@article{mackay1995bayesian,
  title={Bayesian neural networks and density networks},
  author={MacKay, David JC},
  journal={Nuclear Instruments and Methods in Physics Research Section A: Accelerators, Spectrometers, Detectors and Associated Equipment},
  volume={354},
  number={1},
  pages={73--80},
  year={1995},
  publisher={Elsevier}
}

@inproceedings{habring2025diffusion,
  title={Diffusion at absolute zero: {L}angevin sampling using successive {M}oreau envelopes},
  author={Habring, Andreas and Falk, Alexander and Pock, Thomas},
  booktitle={2025 IEEE Statistical Signal Processing Workshop (SSP)},
  pages={61--65},
  year={2025},
  organization={IEEE}
}

@article{rossky1978brownian,
  title={Brownian dynamics as smart {M}onte {C}arlo simulation},
  author={Rossky, Peter J and Doll, Jimmie D and Friedman, Harold L},
  journal={The Journal of Chemical Physics},
  volume={69},
  number={10},
  pages={4628--4633},
  year={1978},
  publisher={American Institute of Physics}
}

@article{durmus2019high,
  title={High-dimensional {B}ayesian inference via the unadjusted {L}angevin algorithm},
  author={Durmus, Alain and Moulines, {\'E}ric},
  journal={Bernoulli},
  volume={25},
  number={4A},
  pages={2854--2882},
  year={2019},
  publisher={JSTOR}
}

@article{jordan1998variational,
  title={The variational formulation of the {F}okker--{P}lanck equation},
  author={Jordan, Richard and Kinderlehrer, David and Otto, Felix},
  journal={SIAM journal on mathematical analysis},
  volume={29},
  number={1},
  pages={1--17},
  year={1998},
  publisher={SIAM}
}

@article{hwang2024fadam,
  title={FAdam: Adam is a natural gradient optimizer using diagonal empirical {F}isher information},
  author={Hwang, Dongseong},
  journal={arXiv preprint arXiv:2405.12807},
  year={2024}
}

@article{kunstner2019limitations,
  title={Limitations of the empirical {F}isher approximation for natural gradient descent},
  author={Kunstner, Frederik and Hennig, Philipp and Balles, Lukas},
  journal={Advances in neural information processing systems},
  volume={32},
  year={2019}
}

@article{liu2016stein,
  title={Stein variational gradient descent: A general purpose {B}ayesian inference algorithm},
  author={Liu, Qiang and Wang, Dilin},
  journal={Advances in neural information processing systems},
  volume={29},
  year={2016}
}

@article{salim2020wasserstein,
  title={The {W}asserstein proximal gradient algorithm},
  author={Salim, Adil and Korba, Anna and Luise, Giulia},
  journal={Advances in Neural Information Processing Systems},
  volume={33},
  pages={12356--12366},
  year={2020}
}

@article{beck2003mirror,
  title={Mirror descent and nonlinear projected subgradient methods for convex optimization},
  author={Beck, Amir and Teboulle, Marc},
  journal={Operations Research Letters},
  volume={31},
  number={3},
  pages={167--175},
  year={2003},
  publisher={Elsevier}
}

@article{girolami2011riemann,
  title={Riemann manifold {L}angevin and {H}amiltonian {M}onte {C}arlo methods},
  author={Girolami, Mark and Calderhead, Ben},
  journal={Journal of the Royal Statistical Society Series B: Statistical Methodology},
  volume={73},
  number={2},
  pages={123--214},
  year={2011},
  publisher={Oxford University Press}
}

@inproceedings{li2016preconditioned,
  title={Preconditioned stochastic gradient {L}angevin dynamics for deep neural networks},
  author={Li, Chunyuan and Chen, Changyou and Carlson, David and Carin, Lawrence},
  booktitle={Proceedings of the AAAI conference on artificial intelligence},
  volume={30},
  number={1},
  year={2016}
}

@article{fessler1999conjugate,
  title={Conjugate-gradient preconditioning methods for shift-variant PET image reconstruction},
  author={Fessler, Jeffrey A and Booth, Scott D},
  journal={IEEE transactions on image processing},
  volume={8},
  number={5},
  pages={688--699},
  year={1999},
  publisher={IEEE}
}

@article{ben2001ordered,
  title={The ordered subsets mirror descent optimization method with applications to tomography},
  author={Ben-Tal, Aharon and Margalit, Tamar and Nemirovski, Arkadi},
  journal={SIAM Journal on Optimization},
  volume={12},
  number={1},
  pages={79--108},
  year={2001},
  publisher={SIAM}
}

@article{bonet2024mirror,
  title={Mirror and preconditioned gradient descent in {W}asserstein space},
  author={Bonet, Cl{\'e}ment and Uscidda, Th{\'e}o and David, Adam and Aubin-Frankowski, Pierre-Cyril and Korba, Anna},
  journal={Advances in Neural Information Processing Systems},
  volume={37},
  pages={25311--25374},
  year={2024}
}

@inproceedings{fan2022variational,
  title={Variational {W}asserstein gradient flow},
  author={Fan, Jiaojiao and Zhang, Qinsheng and Taghvaei, Amirhossein and Chen, Yongxin},
  booktitle={International Conference on Machine Learning},
  pages={6185--6215},
  year={2022},
  organization={PMLR}
}

@article{yao2024wasserstein,
  title={Wasserstein proximal coordinate gradient algorithms},
  author={Yao, Rentian and Chen, Xiaohui and Yang, Yun},
  journal={Journal of Machine Learning Research},
  volume={25},
  number={269},
  pages={1--66},
  year={2024}
}
